\documentclass[11pt]{article}

\usepackage[utf8]{inputenc} 
\usepackage[T1]{fontenc}    
\usepackage{hyperref}       
\usepackage{url}            
\usepackage{booktabs}       
\usepackage{amsfonts}       
\usepackage{nicefrac}       
\usepackage{microtype}      
\usepackage{xcolor}         

\usepackage{caption}
\usepackage{subcaption}
\usepackage{enumitem}
\usepackage{hyperref}
\usepackage[nocompress]{cite}
\usepackage{graphicx}
\usepackage{mathtools}
\usepackage{amssymb}

\usepackage{local_macros}

\usepackage{geometry}
\title{Finite Sample Identification of Wide Shallow Neural Networks with Biases}
\author{ Massimo Fornasier$^1$, Timo Klock$^{2}$,  Marco Mondelli$^{3,}$\thanks{Marco Mondelli was partially supported by the 2019 Lopez-Loreta prize.}\,, Michael Rauchensteiner$^{4}$}
\date{%
    $^1$Department of Mathematics, Technical University of Munich,\\ Bolzmannstra\ss e 3, 85748, Garching, Germany,\\ Email: \url{massimo.fornasier@ma.tum.de}\\%
    $^2$Deeptech Consulting, Oslo, Norway \\Email: \url{timo@deeptechconsulting.no}\\%
    $^3$Institute of Science and Technology Austria (ISTA),\\ Am Campus 1, 3400 Klosterneuburg\\Email: \url{marco.mondelli@ist.ac.at}\\%
    $^4$Department of Mathematics, Technical University of Munich,\\Bolzmannstra\ss e 3, 85748, Garching, Germany, \\Email: \url{michael.rauchensteiner@ma.tum.de}\\[2ex]%
    \today
}

\begin{document}
\maketitle
\begin{abstract}
  Artificial neural networks are functions depending on a finite number of parameters typically encoded as weights and biases.
  The identification of the parameters of the network from finite samples of input-output pairs is often referred to as the \emph{teacher-student model}, and this model has represented a popular framework for understanding training and generalization. 
  Even if the problem is NP-complete in the worst case, a rapidly growing literature -- after adding suitable distributional assumptions --  has established finite sample identification of two-layer networks with a number of neurons $m=\mathcal O(D)$, $D$ being the input dimension.
  For the range $D<m<D^2$ the problem becomes harder, and truly little is known for networks parametrized by biases as well. This paper fills the gap by providing constructive methods and theoretical guarantees of finite sample identification for such wider shallow networks with biases. Our approach is based on a two-step pipeline: first, we recover the direction of the weights, by exploiting second order information; next, we identify the signs by suitable algebraic evaluations, and we recover the biases by empirical risk minimization via gradient descent. Numerical results demonstrate the effectiveness of our approach. 
\end{abstract}

\section{Introduction}\label{sec:introduction}

Training a neural network is an NP-complete \cite{JUDD1988177,BLUM1992117} and non-convex optimization problem which exhibits spurious and disconnected local minima \cite{Auer96, SafranShamir2018,Yun2019}. However, highly over-parameterized networks are routinely trained to zero loss and generalize well over unseen data \cite{zhang2016understanding}. In an effort to understand these puzzling phenomena, a line of work has focused on the implicit bias of gradient descent methods  \cite{arora2019convergence,NIPS2019_8960,bah2019,moroshko2020implicit,neyshabur2015,
soudry2018implicit,woodworth2020kernel}. Another popular framework to characterize training and generalization is the so-called
{\it teacher-student model} \cite{brutzkus2017globally,tian2017analytical,sedghi2014provable,li2017convergence,
du2018convolutional,du2018gradient,soltanolkotabi2017learning,soltanolkotabi2018theoretical,zhang2019learning,fu2020guaranteed,fornasier2018identification,fornasier2019robust,Fornasier2012,anandkumar15,Lin20, MondMont18,zhong2017recovery}. Here, the training data of a so-called {\it student network} are assumed to be realizable by an unknown {\it teacher network}, which interpolates them. 
This model is justified by the wide literature -- both classical and more recent -- on memorization capacity 
\cite{cover1965geometrical,pinkus_1999,1189626,Andrea2020,bubeck2020network,DBLP:conf/nips/YunSJ19,doi:10.1137/20M1314884,bombari2022memorization}, which shows that generic data can be realized by slightly over-parametrized networks. Furthermore, it has also been proved that, in certain settings, small generalization errors necessarily require identification of the parameters \cite{MondMont18}.
This leads to the fundamental question of understanding when the minimization of the empirical error simultaneously promotes the identification of the teacher parameters and consequently 
the perfect generalization beyond training data.

Existing results mostly focus on the identification of the {\it weights} of shallow (i.e., two-layer) networks with a number of neurons $m$ scaling linearly in the input dimension $D$ (see the related work discussed below). There is also evidence of the average-case hardness of the regime $D^{3/2}<m<D^2$, as weight identification can be reduced to tensor decomposition \cite{MondMont18}. Let us also highlight that, even
if the role of biases is often neglected, 
most of the known universal approximation results would not hold without biases
\footnote{If the activation is odd, then it evaluates to $0$ in $0$ and one can only represent functions which are $0$ in $0$.}.

\paragraph{Main contributions.} In this paper, we give theoretical guarantees on the recovery of both {\it weights} and {\it biases} from finite samples in the regime $D <m<D^2$, under rather mild assumptions on the smoothness of the activation function, incoherence of the weights, and boundedness of the biases. More specifically, the teacher network is given by
\begin{align}
	\label{eq:snn_general}
	f: \R^D \to \R, \quad {f}(x) := \sum_{j=1}^{m}  g(\langle w_j, x\rangle +  \tau_j),
\end{align}
where $w_1,\ldots,w_m$ are unit-norm weights and $\tau_1,\ldots,\tau_m$ are bounded biases. 
We propose a two-step parameter recovery pipeline that decouples the learning of the weights from the recovery of the remaining network parameters. In the first step, we use
second order information to recover the weights $w_1,\ldots,w_m$ up to signs. 
The method (cf. Section \ref{sec:weight_identification})
comes with provable guarantees of recovery up to $m\log^2(m) = O(D^2)$ weights,
provided that \emph{(i)} the weights are sufficiently incoherent, and \emph{(ii)} second order derivatives of $f$ carry \emph{enough information}.
Our approach is based on the observation that $\nabla^2 f(x) =\sum_{j=1}^{m}  g^{(2)}(\langle w_j, x\rangle +  \tau_j) w_j \otimes w_j  \in  \CW = \operatorname{span}\{w_1\otimes w_1, \dots, w_m \otimes w_m\}$ and, hence,  multiple samples of independent Hessians allow to compute an approximating subspace $\wCW \approx \CW$. 
The construction of such a subspace is based exclusively on second order information and differs from the tensor approach in \cite{anandkumar15}, which uses higher order tensor decomposition.
The identification of the weights is then performed by projected gradient ascent, the so-called {\it subspace power method} \cite{fiedlerStableRecoveryEntangled2021,kileelSubspacePowerMethod2021,kileelLandscapeAnalysisImproved2021a}, seeking for solutions of
\begin{align}\label{eq:spm_objective_program0}
\max_{u \in \bbS^{D-1}} \norm{P_{\wCW} (u\otimes u)}^2_F \approx \max_{u \in \bbS^{D-1}}  \norm{P_{\CW} (u\otimes u)}^2_F.
\end{align}

In the second step (cf. Section \ref{sec:rem}), we show how to identify the {\it signs} by suitable algebraic evaluations and the  {\it biases} by empirical risk minimization via gradient descent. For this part, we give a suitable initialization of the algorithm and provide convergence guarantees to the ground-truth biases.
The convergence proof is based on a linearization argument inspired by the {\it neural tangent kernel} (NTK) approach. However, delicate technical adaptations are needed, in order to \emph{(i)} ensure that biases do not compromise the well-posedness of the linearized system, and \emph{(ii)} control the effect on the gradient descent iteration of the errors accumulated in the weight identification step.
The theoretical findings of this paper can be summarized in the following informal statement.
\begin{theorem}[Informal]
	Let $f$ be the shallow network \eqref{eq:snn_general} with $D$ inputs and $m$ neurons such that $m \log^2 m  = O(D^2)$. Then, for sufficiently large $D$, there exists a constructive algorithm recovering all weights and shifts of the network with high probability from $O(D m^2 \log^2 m)$ network queries. 
\end{theorem}
A few comments on the complexity are in order. For $m=\mathcal{O}(D)$, the proposed pipeline has polynomial complexity in $D,m$. For $D<m<D^2$, while the pipeline is still guaranteed to converge globally, our findings clarify precisely how the hardness of the problem consists in distinguishing  local maximizers of  \eqref{eq:spm_objective_program0}. This fine geometrical description is novel, and it could pave the way towards a more refined understanding of the hardness of the network identification problem. In fact, our numerical experiments (cf. Section \ref{sec:numerics_main_nips}) consistently show that the network recovery remains surprisingly successful with low complexity up to the information-theoretic upper bound\footnote{This information-theoretic upper bound holds for all methods employing order-2 tensors. In fact, for $m\approx D^2/2$, $\operatorname{span}\{w_1\otimes w_1, \dots, w_m \otimes w_m\}$ coincides with the space of all symmetric matrices, making it impossible to distinguish $w_j\otimes w_j$ from any other rank-1 matrix.} $m \approx D^2/2$. 

\paragraph{Related work.} A line of work spanning three decades has considered network identification under the assumption of being able to access exactly \emph{all} possible input-output pairs (shallow networks in \cite{SUSSMANN1992589,Albertini93uniquenessof}, fully-connected deep networks in \cite{Fefferman1994ReconstructingAN} and, most recently, deep networks without clone nodes and piecewise $C^1$ activations with bounded-variation derivative in \cite{DBLP:journals/corr/abs-2006-11727}). However, as a neural network remains fully determined by a {\it finite} number of parameters, it is not at all expected to  generically require an infinite amount of training samples. This has motivated the rapidly growing literature on the teacher-student model. A popular setup is to minimize the population risk by assuming a Gaussian distribution of the weights: a two-layer ReLU network with a single neuron is considered in \cite{tian2017analytical}, a single convolutional filter in \cite{du2018convolutional,brutzkus2017globally}, multiple convolutional filters with no overlap in \cite{du2018gradient}, and residual networks in \cite{li2017convergence}. Gradient descent methods have also been widely studied: \cite{soltanolkotabi2017learning} considers a single ReLU unit; \cite{soltanolkotabi2018theoretical} provide a global convergence result for shallow networks with quadratic activations and a local convergence result for more general activations; and gradient descent is combined with an initialization based on tensor decomposition in \cite{zhong2017recovery,zhang2019learning,fu2020guaranteed}. A local convergence analysis for student networks containing at least as many neurons as the teacher network is provided by \cite{zhou2021local}. Let us highlight that these results neglect the role of biases, and the convergence guarantees are either local or, when global, require a number of neurons $m=\mathcal O(D)$. Inspired by papers dating back to the 1990s \cite{BUHMANN1999103,PE98,CHUI1992131}, the works \cite{sedghi2014provable,fornasier2018identification,Fornasier2012,anandkumar15,Lin20, MondMont18,zhong2017recovery} have explored the connection between differentiation of shallow networks and symmetric tensor decompositions. Once the weights have been identified, the computation of {\it biases} has also been considered by direct estimation \cite{fornasier2018identification} for $m\leq D$, or by Fourier methods \cite{anandkumar15} for mildy overparametrized network. These results, however, do not offer rigorous guarantees for the regime $D<m<D^2$. Finally, a line of work from statistical physics, which started with \cite{saad1995line,saad1995exact,saad1996learning,biehl1995learning,riegler1995line}
and recently culminated with \cite{goldt2019dynamics}, has characterized the dynamics of one-pass stochastic gradient descent via a set of ordinary differential equations, thus providing insights on the generalization behavior.

\paragraph{Technical tools and innovations.} \emph{\underline{Weight identification:}} Instead of considering higher order tensor decompositions (and, hence, higher order differentiation of the network) as done by \cite{anandkumar15}, here we follow the strategy of  \cite{fornasier2018identification}, which exploits the information coming from the Hessians. However, while \cite{fornasier2018identification} require the weights to be linearly independent (and, thus, $m\leq D$), we tackle the challenging overcomplete case $m>D$. Furthermore, for the identification of the weights we use \eqref{eq:spm_objective_program0}, namely a robust non-linear program over vectors, which is significantly less computationally expensive than the  minimum rank selection of \cite{fornasier2018identification}. Our analysis improves upon \cite{fiedlerStableRecoveryEntangled2021} by allowing to go beyond a linear scaling between $m$ and $D$, and it takes advantage of the new insights provided by \cite{kileelLandscapeAnalysisImproved2021a} on the subspace power method.

\emph{\underline{Shift identification:}} Differently from \cite{fornasier2018identification, anandkumar15}, we set up an empirical risk minimization problem, and we solve it via gradient descent. Our proof of convergence is based on certain kernel matrices, which are reminiscent of those appearing in the {\it neural tangent kernel (NTK)} theory \cite{jacot2018neural}. The NTK perspective has been used to prove global convergence of gradient descent for shallow 
\cite{DuEtal2018_ICLR, OymakMahdi2019, SongYang2020,wu2019global,Andrea2020,song2021subquadratic} and deep neural networks
\cite{AllenZhuEtal2018,DuEtal2019,zou2020gradient,ZouGu2019,QuynhMarco2020,nguyen2021proof,bombari2022memorization}. The technical innovations of our paper with respect to this line of work are as follows. First, we exchange the role between input variable $x$ and weights: we consider the Jacobian of the network with respect to {\it its input $x$}, and not to its parameters. This allows us to keep fixed the size of the network and to analyze the NTK spectrum for large input samples. Second, we extend the NTK theory to handle networks with biases. Finally, as the accuracy of the linearization argument depends on the errors accumulated in the weight identification step, we carry out a delicate perturbation analysis.

\paragraph{Notation.} Given two vectors $u$ and $v$, let $u\otimes v$ be their Kronecker product and $u \odot v$ their element-wise product. Given a vector $u$, let $\norm{u}_2$ be its $\ell_2$ norm and $\operatorname{diag}(v)$ the diagonal matrix with $v$ on diagonal. Given a matrix $A$, let $\norm{A}$ be its operator norm, $\norm{A}_F$ its Frobenius norm, and $\norm{A}_{F\to F}=\sup_{\| X \|_F = 1} \norm{A X}_F$. Let $\operatorname{Sym}(\mathbb R^{d\times d})$ be the space of symmetric matrices in $\R^{d\times d}$, $\CC^{n}(\R)$ the space of functions in $\R$ with $n$ continuous derivatives, $\textrm{Uni}(\bbS^{D-1})$ the uniform distribution on the $D$-dimensional sphere $\bbS^{D-1}$, and $\Id_p$ the identity matrix in $\R^{p\times p}$. Given a function $g$, let $g^{(n)}$ be its $n$-th derivative.
Given a vector $v$ and a permutation $\pi$, let $v_\pi$ be the vector obtained by permuting the entries of $v$ according to $\pi$.

\section{Network model and main result}\label{sec:network_model}
We consider the parameter recovery of a planted shallow neural network $f:\R^D \to \R$ of the form \eqref{eq:snn_general}.
We assume that the weights are drawn uniformly at random from the sphere, i.e.,  $w_1,\ldots,w_m \sim_{\rm i.i.d.} \textrm{Uni}(\bbS^{D-1})$, and the shifts are contained in a given interval, i.e., $\tau_1, \ldots, \tau_m \in [-\tau_{\infty},+\tau_{\infty}]$. We also make the following assumptions on the activation function $g$ and on the Hessians of $f$.

\begin{enumerate}[label=\textbf{(M\arabic*)}]
\item\label{enum:activation}  $g \in \CC^{3}(\R)$ and 
\begin{align}\label{eq:def_kappa}
	\kappa &:= \max_{n\in [3]} \norm{g^{(n)}}_\infty < \infty.
	\end{align}
Furthermore, $g^{(2)}$ is strictly monotonic on $(-\tau_{\infty},+{\tau}_{\infty})$, $g^{(1)}$ is strictly positive or negative on $(-\tau_{\infty},+{\tau}_{\infty})$ and 
there exists $s \in \{-1, +1\}$ such that for all $\tau \in [-\tau_{\infty},+{\tau}_{\infty}]$ we have 
\begin{align*}
	s= \operatorname{sgn}\left(\int_{\R} g^{(1)}(t + \tau) \exp(- t^2 / 2 ) dt \right).
\end{align*}

\item\label{enum:nonpolynomoial}  $g^{(1)}$ is not a polynomial of degree 3 or less {and  
$\int_{\R} g(t)^2 \exp(-t^2/2)dt < \infty$.} 
\item\label{enum:learnability} The Hessians of $f$ have sufficient information for weight recovery, i.e.,
\begin{align}\label{eq:learnability_condition}
\lambda_{m}\left(\E_{X \sim \CN(0, \Id)}[\opvec(\nabla^2 f(X))^{\otimes 2}]\right) \geq \alpha > 0.
\end{align}
\end{enumerate}
The size of the interval $[-\tau_{\infty}, +\tau_{\infty}]$ 
does not depend on $m$ or $D$, but only on $g$ via \ref{enum:activation}. This assumption is satisfied by common activations, such as $g(x)=\tanh(x)$ for ${\tau}_{\infty}\approx  0.6$ and the sigmoid $g(x)=1/(1+\exp(-x))$ for ${\tau}_{\infty} \approx 1.5$.
Condition \ref{enum:learnability} guarantees that combining Hessians of $f$ at sufficiently many generic inputs provides enough information to recover all individual weights. A potential way to show that \eqref{eq:learnability_condition} holds is as follows. First, note that $\nabla^2 f(x) = \sum^m_{k=1} g^{(2)}(w_k^\top x + \tau_k) w_k \otimes w_k \in \operatorname{span}\{w_1\otimes w_1, \dots, w_m \otimes w_m \}$. Hence, 
by exploiting the incoherence of $w_1, \dots, w_m \sim \operatorname{Uni}(\bbS^{D-1})$, one could relate the smallest eigenvalue in \eqref{eq:learnability_condition} to that of the matrix with entries $(\E_{X \sim \CN(0, \Id)} [g^{(2)}(\langle w_k, X \rangle + \tau_k )  g^{(2)}(\langle w_\ell, X \rangle + \tau_\ell ) ])_{k, \ell}.$ This last quantity may then be bounded using the tools developed in Section \ref{sec:gradient_descent}. Making these passages rigorous is beyond the scope of this work, and we leave it as an open question. We also highlight that assumption \ref{enum:learnability} is common in the related literature \cite{fornasier2019robust, fornasier2018identification, fiedlerStableRecoveryEntangled2021, Anandkumar2014GuaranteedNT}.

We also assume the ability to evaluate the network $f$ and to approximate its derivatives.

\begin{enumerate}[label=\textbf{(G\arabic*)}]
		\item\label{enum:active_sampling} We can query the teacher network ${f}$ and the activation $g$ at any point without noise, and the number of neurons $m$ is known.
		\item\label{enum:numerical_diff} We assume access to a numerical differentiation method, denoted by $\Delta^n[\cdot]$, that computes the derivatives for $n=1,2,3$ up to an accuracy $\epsilon>0$. To be more precise, we require that the derivatives of $g$ with respect to a vector input $x\in \R^D$ fulfill
		\begin{align}
		\norm{\nabla^n g(w^\top x) - \Delta^n[g(w^\top x)]}_F  
		\leq C_{\Delta} \norm{w^{\otimes n}}_F \epsilon,
		\end{align}
		where $C_\Delta$ is a universal constant only depending on the activation through $\kappa$ (see \eqref{eq:def_kappa}). Furthermore, for any $b, t_0 \in \R$ the derivatives of $t \mapsto g(b \, t)$ can be approximated as
		 \begin{align}\label{eq:numerical_diff_prop_sapprox}
					\snorm{ \left.  \frac{d^n}{d t^n} g(b \cdot t)\right|_{t=t_0}  -  \Delta^{n}[g(b \cdot)](t_0) } \leq C_{\Delta}b^{n+2} \epsilon.
				\end{align}
		We also assume that the numerical differentiation method is linear, i.e.,
		\begin{align}
		\Delta^n [a \cdot g + h] = a \cdot \Delta^n [g] + \Delta^n[h],
		\end{align}
		for any functions $g,h$ and scalar $a\in \R$.
		 Finally, the numerical differentiation algorithm requires a number of queries equal to the dimension of the approximated derivative, i.e.,
		 $\mathcal{O}(1)$ for partial derivatives and $\mathcal{O}(D^n)$ for $n$-th order derivative tensors.
\end{enumerate}
We note that all the properties in \ref{enum:numerical_diff} are fulfilled by a standard central finite difference scheme. 

Our proposed algorithm for the recovery of the parameters of the planted model \eqref{eq:snn_general} is based on a two-step procedure. In the first step, we learn the weight vectors (up to a sign) from the space spanned by Hessian approximations of $f$ (cf. Section \ref{sec:weight_identification}). Recovering the weights provides access to vectors $\hw_k$, which satisfy
$ s_k \hw_k \approx w_k$ for some signs $s_1,\ldots, s_m \in \{-1,1\}$. In the second step, we identify the signs $s = ( s_1,\ldots,s_m)$
and shifts $\tau = (\tau_1,\dots,\tau_m)$ (cf. Section \ref{sec:rem}). We begin by finding $s$ and an initialization of the shifts $\hat \tau \approx \tau$ by a linearization through higher order (numerical) differentiation along the previously computed weight approximations. The shift approximation $\hat\tau$ is then refined by empirical risk minimization. More precisely, we consider the parametrization
\begin{align}\label{eq:studnet}
	\hat f(x, \hat \tau) := \sum_{k=1}^{m} g(s_k \langle \hw_k, x\rangle + \htau_k),
\end{align} 
which is fit against the planted model $f(x)$ defined in  \eqref{eq:snn_general} by minimizing the least squares objective
\begin{align}
	\label{eq:def:loss}
	J(\hat \tau) = \frac{1}{2N_{\text{train}}}\sum^{N_{\text{train}}}_{i=1} \Big(f(x_i) - \hat{f}(x_i, \htau) \Big)^2 
\end{align}
via gradient descent, where $x_1,\ldots,x_{N_{\text{train}}}\sim_{\rm i.i.d.} \CN(0,\Id_D)$.
Provided that the activation function satisfies 
\ref{enum:activation}-\ref{enum:nonpolynomoial}
, we show that gradient descent is guaranteed to converge locally to the
ground truth shifts up to an error depending only on the accuracy of the initial weight estimates $\hat w_k \approx \pm w_k$. The combination of these two steps leads to Algorithm \ref{alg:pipeline} and to our main result, stated below. Its proof is deferred to Section \ref{app:pfmain} of the supplementary materials, and it  follows as a combination of Theorem \ref{thm:weight_recovery}, Proposition \ref{prop:initialization}, and Theorem \ref{thm:local_result} (discussed in the rest of the paper).

\begin{algorithm}[t]
\KwIn{Teacher neural network $f$ defined in \eqref{eq:snn_general} with known number of neurons $m$, numerical differentiation method $\Delta^n [\cdot]$ with accuracy $\epsilon$, number of Hessian locations $N_h$ and gradient descent samples $N_{\text{train}}$, number of steps for refinement via gradient descent $N_{\text{GD}}$.}
Compute weights $\wW = [\hw_1| \dots | \hw_m]$ by PCA of Hessians followed by iterations of the subspace power method (cf. Algorithm \ref{alg:recover_weights} in the supplementary materials, 
and discussion in Section \ref{sec:weight_identification});\\
Find signs $\hat s$ and initial shifts $\hat{\tau}\in \R^m$ by linearization through higher order differentiation along approximated weight vectors (cf. Algorithm \ref{alg:initialization}  in the supplementary materials, 
and discussion in Section \ref{sec:initialization});\\
Set $\wW \leftarrow \wW\operatorname{diag}(\hat s)$ and construct a student network $\hat f$ as in \eqref{eq:studnet} with 
parameters $\wW, \htau$;\\
Draw samples $x_1,\dots, x_{N_{\text{train}}} \sim \CN(0, \Id_D)$ and refine the shifts of $\hat f$ by minimizing $J(\htau)$ (cf. \eqref{eq:def:loss})
via gradient descent for $N_{\text{GD}}$ steps (cf. Section \ref{sec:gradient_descent}). Denote by $\htau^{[N_{\text{GD}}]}$ the final iterate.\\
\KwOut{Weights $\wW$ and final shifts $\htau^{[N_{\text{GD}}]}$ of $\hat f$.}
\caption{\textbf{Network reconstruction}}
\label{alg:pipeline}
  \end{algorithm}
\begin{theorem}[Main result on network reconstruction]\label{thm:main_theorem}
Consider the teacher network $f$ defined in \eqref{eq:snn_general}, where $w_1, \dots, w_m  \sim \textrm{Uni}(\bbS^{D-1})$ and ${\tau}_1,\ldots,{\tau}_m \in [-\tau_{\infty},\tau_{\infty}]$. Assume $g$ satisfies \ref{enum:activation} - \ref{enum:nonpolynomoial}
and $f$ satisfies the learnability condition \ref{enum:learnability} for some $\alpha>0$. Assume we run
Algorithm \ref{alg:pipeline} with $N_h > t(m + m^2 \log(m) / D)$ for some $t \geq 1$ and $N_{\text{train}}  > m\sqrt{D}$.
Then, there exists $D_0 \in \N$ and a constant $C > 0$ only depending on $g$ and ${\tau}_{\infty}$ such that the following holds with probability at least
$1- m^{-1} - 2 D^2 \exp\left(-\min\{\alpha,1\} t /C  \right) - Cm^2 \exp(- \sqrt{D}/C)$:
If $m \geq D\geq D_0$, $C m \log^2 m \leq D^2$, and the numerical differentiation accuracy $\epsilon$ satisfies
	\begin{align}\label{eq:eps_bound_main}
  \epsilon \leq \frac{D^{1/2} \min\{1,\alpha^{1/2}\} }{C m^{9/2} \log(m)^{3/2}},
	\end{align}
then Algorithm \ref{alg:pipeline} returns weights and shifts $(\wW = [\hw_1|\dots | \hw_m],\htau^{[N_{GD}]})$ that fulfill
\begin{align}\label{eq:errw}
\max_{k \in [m]} \|\hw_{\pi(k)} - {w}_k\|_2 &\leq C (m/ \alpha)^{1/4} \epsilon^{1/2}, \end{align}
\begin{align}\label{eq:errs}
\|\htau^{[N_{GD}]}_{\pi} - {\tau}\|_2 &\leq  C \left( \frac{m^{7/4} D^{1/4} \epsilon^{1/2}}{ \alpha^{1/4} N_{\text{train}}^{1/2}}  +  \frac{\xi^{N_{GD}}}{m^{1/2}}+\Delta_{W,1}\right), 
\end{align}
for some permutation $\pi$ and some constant $\xi \in [0,1)$ where 
\begin{align}\label{eq:def_delta_w1}
	\Delta_{W,1} &:= 	\frac{ m^{1 / 2} \log(m)^{3/4}}{D^{1 / 4}}  \cdot \left( 
		\| \wW - W\|_F  + \frac{\Delta_{W, O}^{1/2}}{D^{1 / 2}} + \left\| \sum^m_{k=1} w_k - \hw_k \right\|_2 \right) ,\\
		\Delta_{W, O} &:= \sum^m_{k \neq k'} \left|\inner{w_k - \hw_{k}, w_{k'} - \hw_{k'}}\right|. \label{eq:def_delta_wo}
\end{align}
\end{theorem}
By choosing an appropriate numerical accuracy $\epsilon$, \eqref{eq:eps_bound_main} is satisfied and the error on the weights in \eqref{eq:errw} can be made arbitrarily small. The error on the shifts in \eqref{eq:errs} depends on three terms. The first term scales with $\sqrt{\epsilon / N_{\text{train}}}$, hence it is controlled by taking a large number of training samples. 
The second term vanishes exponentially with the number of gradient steps $N_{GD}$. Thus, for large enough $N_{\text{train}}$ and $N_{GD}$, the dominant factor is $\Delta_{W,1}$. This last term decreases with the weight approximation error, i.e., if $\wW = W$, then $\Delta_{W,1}=0$. In fact, $\Delta_{W,1}$ scales with $\epsilon^{1/2}$, hence it can be reduced by improving the numerical accuracy.

It is natural to compare the residual error term $\Delta_{W,1}$ \emph{after gradient descent} with the error on the shifts \emph{before gradient descent}, i.e., at initialization as given by Proposition \ref{prop:initialization} (cf. \eqref{eq:hattau}). If we assume randomness on the weight errors (with variance matching the upper bound in \eqref{eq:errw}), i.e., $\hw_{\pi(k)} - {w}_k \sim_{i.i.d.} \mathcal{N}(0, (m/\alpha)^{(1/2)} \epsilon / D \cdot \Id_D )$ then, up to poly-logarithmic factors, $\Delta_{W,1}$ scales as
	\begin{equation}\label{eq:impapp}
	    \frac{\epsilon^{1/2}}{\alpha^{1/4}}\left(\frac{m^{5/4}}{D^{1/4}} + \frac{m^{7/4}}{D}\right).
	\end{equation}
This last quantity is provably smaller than the error \eqref{eq:hattau} at initialization, see the discussion after Proposition \ref{prop:initialization}. In the worst case, when all weight errors are aligned, $\Delta_{W,1}$ is dominated by $\left\| \sum^m_{k=1} w_k - \hw_k \right\|_2 = \mathcal{O}(m^{5/4} \alpha^{-1/4} \epsilon^{1/2})$, which would not lead to a provable improvement over \eqref{eq:hattau}. However, in Section \ref{sec:numerics_main_nips}, we numerically observe that this type of error accumulation does not occur: the term $\left\| \sum^m_{k=1} w_k - \hw_k \right\|_2$ is negligible and $\Delta_{W,1}$ is significantly smaller than \eqref{eq:hattau}, see Figure \ref{fig:scaling} and the related discussion.

\section{Identification of the weights}
\label{sec:weight_identification}
\begin{definition}[RIP]
	\label{def:RIP}
	Let $W \in \R^{D\times m}$, $1\leq p \leq m$ be an integer, and $\delta \in (0,1)$.
	We say that $W$ is \textit{$(p, \delta)$-RIP} if every $D\times p$ submatrix $W_p$ of $W$ satisfies
	$\|W_p^\top W_p - \Id_p\|_{2}\leq \delta$.
\end{definition}

\begin{definition}[Properties of isotropic random weights]
\label{def:assumptions_overcomplete}
Let $W := [w_1|\ldots|w_m]$ and $(G_n)_{k\ell} := \langle w_k, w_\ell\rangle^{n}$. We define the following incoherence properties:
\begin{enumerate}[leftmargin=1.5cm,label=\textbf{(A\arabic*)}]
	\item\label{enum:RIP} There exists $c_1 > 0$, depending only on $\delta$, such that $W$ is $(\lceil c_1 D/\log(m) \rceil,\delta)$-RIP.
	\item\label{enum:correlation} There exists  $c_2 > 0$, independent of $m,D$, so that $\max_{i\neq j}\langle w_i, w_j\rangle^2 \leq c_2\log(m)/D$.
	\item \label{enum:GInverse} There exists $c_3 > 0$, independent of $m,D$, so that $\norm{G_n^{-1}} \leq c_3$,  for all $n\geq 2$.
\end{enumerate}
\end{definition}
\noindent
If the number of weights $m$ is $o(D^2)$, 
weights drawn from the uniform spherical distribution fulfill  \ref{enum:RIP} - \ref{enum:GInverse} with high probability. This follows from a result due to \cite{kileelLandscapeAnalysisImproved2021a} (cf. Proposition \ref{prop:incoherence_holds_for_uniform} in the supplementary materials). We are going to use the properties of Definition \ref{def:assumptions_overcomplete} throughout our analysis.

The weight recovery consists of two steps. First, we leverage the fact that approximated Hessians of the network expose the weights according to
\begin{align*}
\Delta^2 f(x) \approx \nabla^2 f(x) = \sum^m_{k=1} g^{(2)}(\langle w_k, x\rangle + {\tau}_k)  w_k \otimes  w_k,
\end{align*}
such that independent sampling of Hessian locations eventually spans (approximately) the space
\begin{align}
\wCW \approx \CW := \operatorname{span}\left\lbrace  w_1 \otimes  w_1 , \dots,   w_m \otimes  w_m\right\rbrace,
\end{align}
with $\wCW,\CW \subset \operatorname{Sym}(\R^{D\times D})$.
This holds w.h.p. for Hessian locations $x_1, \dots, x_{N_h}$ drawn as  standard Gaussians as a consequence of \ref{enum:learnability}, provided $N_h$ is sufficiently large. The resulting approximation error $\norm{P_{\CW} - P_{\wCW}}_{F\rightarrow F}$ can be controlled by the accuracy of the numerical differentiation $\epsilon$, see Lemma \ref{lem:subspace_pertubation} in the supplementary materials.

Next, the weights are uniquely identified (up to a sign) as the $2m$ local maximizers of the program \eqref{eq:spm_objective_program0},
which belong to a certain level set $\{ u \in \bbS^{D-1} | \norm{P_{\wCW}(u\otimes u) }_F^2 \geq \beta \}$ of the underlying objective. This follows as a special case from the theory within \cite{kileelSubspacePowerMethod2021,fiedlerStableRecoveryEntangled2021, kileelLandscapeAnalysisImproved2021a}. More specifically, \cite{kileelSubspacePowerMethod2021} study the problem in the unperturbed case, \cite{fiedlerStableRecoveryEntangled2021} extend the subspace power method to the perturbed objective but their analysis is limited to $m<2D$, and finally \cite{kileelLandscapeAnalysisImproved2021a}  go for $2$-tensor decompositions up to $m=o(D^2)$ for the perturbed objective.
Then, the local maximizers of \eqref{eq:spm_objective_program0} are computed via a projected gradient ascent algorithm that iterates 
\begin{align}\label{eq:pgd_iteration}
  u_{j+1} =  P_{\bbS^{D-1}}(  u_j+ 2\gamma P_{\wCW}((   u_j)^{\otimes 2})  u_j),
\end{align}
where $\gamma$ is the step-size and $P_{\bbS^{D-1}}/P_{\wCW}$ denote the projections on $\bbS^{D-1}/\wCW$. The iteration \eqref{eq:pgd_iteration} starts from a random initialization $u_0 \in \bbS^{D-1}$, and it was introduced by \cite{kileelSubspacePowerMethod2021} as a subspace power method (SPM). 
By iterating \eqref{eq:pgd_iteration} until convergence repeatedly from independent starting points, one can collect all $m$ local maximizers of \eqref{eq:spm_objective_program0} and thereby learn (approximately) all weights up to sign. 
Assuming the retrieval of every local maximizer is equally likely, the average number of repetitions needed to recover all local maximizers follows from the analysis of the coupon collection problem and grows like $\Theta(m \log m)$ (see also \cite{Fornasier2012}).
The theorem below provides a bound on the uniform approximation error for the weights. Its proof, as well as the description of Algorithm \ref{alg:recover_weights} summarizing the overall procedure of weight identification, is deferred to Section \ref{sec:proofs_weight_recovery} of the supplementary materials.
\begin{theorem}[Weight recovery]\label{thm:weight_recovery}
Consider the teacher network $f$ defined in \eqref{eq:snn_general}, where $w_1, \dots, w_m  \sim \textrm{Uni}(\bbS^{D-1})$ and ${\tau}_1,\ldots,{\tau}_m \in [-\tau_{\infty},\tau_{\infty}]$. Assume $g$ satisfies \ref{enum:activation} - \ref{enum:nonpolynomoial}
and $f$ satisfies the learnability condition \ref{enum:learnability} for some $\alpha>0$.
Then, there exists $D_0 \in \N$ and a constant $C>0$ depending only on $g,{\tau}_{\infty}$, such that, for all $D\geq D_0$ and $C m \log^2 m \leq D^2$, the following holds with probability at least 
$1- m^{-1} - D^2 \exp\left(-\min\{\alpha,1\} t / C \right) - C\exp(- \sqrt{m}/C)$:
(i) The weights $w_1,\ldots,w_m$ fulfill \ref{enum:RIP} - \ref{enum:GInverse}, and
(ii) if we run Algorithm \ref{alg:recover_weights} with numerical differentiation accuracy $\epsilon \leq \frac{\sqrt{\alpha}}
{C \sqrt{m}}$ and using  $N_h > t(m + m^2 \log(m) /D)$ Hessian locations for some $t \geq 1$,
we obtain a set of approximated weights $\CU \subset \bbS^{D-1}$ such that, for all $\hw \in \CU$, there exists a $k \in [m]$ and a sign $s\in \left\lbrace -1, +1\right\rbrace$ for which
\begin{align}
\norm{w_k - s \hw_k}_2 \leq C (m / \alpha )^{1/4} \epsilon^{1/2}.
\end{align}
\end{theorem}

\section{Identification of the signs and shifts}\label{sec:rem}

By leveraging the fact that
differentiation exposes the weights of the network as components of the tensor
$\nabla^n f(x) = \sum^m_{k=1} g^{(n)}(x^\top w_k+ {\tau}_k)w_k^{\otimes n}$ for $n=2$, Theorem \ref{thm:weight_recovery} gives that 
$\hw_k\approx {s}_k {w}_k$ for some signs ${s}_k \in \{-1,+1\}$. In this section, we show how to recover 
the remaining parameters (shifts and signs) for a given set of ground truth weights $\{w_1, \dots, w_m\}\subset \bbS^{D-1}$ which are sufficiently incoherent and approximated by $\{\hw_1, \dots, \hw_m\}\subset \bbS^{D-1}$ up to a sign. This recovery can be broken down into two steps. First, we find the correct signs and good initial shifts (cf. Section \ref{sec:initialization}); once the parameters are known, a student network can be initialized from these starting values. Second, the shifts of the student network are refined by empirical risk minimization via gradient descent (cf. Section \ref{sec:gradient_descent}). 
\subsection{Parameter initialization}\label{sec:initialization}

Our initialization strategy is centered around the recovery of the quantities $\mathcal{C}_2=(\mathcal{C}_{2, 1}, \ldots, \mathcal{C}_{2, m})$ and $\mathcal{C}_3=(\mathcal{C}_{3, 1}, \ldots, \mathcal{C}_{3, m})$, where 
\begin{align}
\mathcal{C}_{n,k} :=   s_k^n  g^{(n)}({\tau}_k), \quad \text{ for } k \in [m],\quad n\in \{2, 3\}.
\end{align}
If $g$ satisfies \ref{enum:activation}, then $g^{(3)}$ does not change sign on the interval $(-\tau_{\infty}, \tau_{\infty})$ due to the monotonicity of $g^{(2)}$. Hence, we can infer the sign $s_k$ from $\mathcal{C}_{3, k}$. Furthermore, as $g^{(2)}$ is monotone on $[-\tau_\infty, \tau_\infty]$, it admits an inverse, which allows for the recovery of $\tau_k$ from $\mathcal{C}_{2,k}$.  To learn $\mathcal{C}_2, \mathcal{C}_3$, we rely on numerical approximations of the quantities $\langle \nabla^n f(x), \hw_k^{\otimes n} \rangle$, namely, the directional derivatives of the network $f$ along the approximated weights. We consider the following linear system representation of the directional derivatives. Computing the derivative for $x=0$ reveals 
\begin{align*}
	\langle \nabla^n f(0), \hw_\ell^{\otimes n} \rangle = \sum^m_{k = 1} s_k^n g^{(n)}(\tau_k ) \inner{ s_k w_k, \hw_\ell}^n.
\end{align*}
Denote by $\tilde{G}_n \in \R^{m\times m}$ the matrix with entries $(\tilde{G}_n)_{\ell, k} = {\langle  \hw_\ell, s_k w_k \rangle}^n$. 
Then, we have
\begin{align}\label{eq:Tn}
\tilde{G}_n \cdot \mathcal{C}_n 
= \begin{bmatrix}
 	\langle \nabla^n f(0), \hw_1^{\otimes n} \rangle 	\\
 \vdots\\
 	\langle \nabla^n f(0), \hw_m^{\otimes n} \rangle
\end{bmatrix} := T_n.
\end{align}
In \eqref{eq:Tn}, $T_n$ is a vector containing all directional derivatives of $f$ evaluated at $0$ along the recovered weights $\hw_1, \dots, \hw_m$. These directional derivatives can be approximated from only $\mathcal{O}(n)$ evaluations of the network by numerical differentiation (cf. \ref{enum:numerical_diff}), which allows us to compute $\tilde{T}_n  \approx T_n$. Provided the weight approximations are sufficiently accurate and incoherent in the sense of Definition \ref{def:assumptions_overcomplete}, the matrix $\tilde{G}_n$ is invertible and can  be estimated by $(\widehat{G}_n)_{\ell, k} := \langle \hw_\ell, \hw_k \rangle^n$.   
Therefore, we obtain $\mathcal{C}_n\approx \tilde{G}_n^{-1} \tilde{T}_n  \approx \widehat{G}_n^{-1} \tilde{T}_n$. This strategy is summarized in Algorithm \ref{alg:initialization} detailed in supplementary materials, and the robustness analysis of Proposition \ref{prop:initialization} makes all the approximations rigorous. 
This procedure could be carried out for any order of directional derivatives, allowing us to benefit from the higher incoherence of $\langle w_k^{\otimes n}, w_\ell^{\otimes n}\rangle = \langle w_k, w_\ell\rangle^n$. 
However, for the sake of simplicity and to be more aligned with our network model, we combine only the second and third order directional derivatives.
\begin{proposition}[Parameter initialization]\label{prop:initialization}  Consider the teacher network $f$ defined in \eqref{eq:snn_general}, where the weights $\{w_k \in \bbS^{D-1}, \, k\in [m] \}$ satisfy \ref{enum:correlation} - \ref{enum:GInverse} with constants $c_2,c_3$ and the activation $g$ satisfies \ref{enum:activation}. Then, there exist constants $C>0$ only depending on $g, c_2,c_3, {\tau}_{\infty}$ and $D_0\in \N$, such that, for $ m \geq D\geq D_0, m \log^2 m \leq D^2$, the following holds. Given $\hw_1, \dots, \hw_m \in \bbS^{D-1}$ such that 
	\begin{align}\label{eq:dmax}
	\delta_{\max} &:= \max_{k\in[m]} \min_{s\in \{-1,1\}} \norm{w_k -  s\hw_k}_2 \leq  \frac{D^{1/2}}{ C m \sqrt{  \log m}  },
	\end{align}
	Algorithm \ref{alg:initialization} returns a set of shifts $\hat{\tau}$ such that
		\begin{align}\label{eq:hattau}
			 \norm{\hat{\tau} - {\tau}}_2 \leq  C\sqrt{m}\epsilon + C m^{3/2} \left( \frac{\log m}{D}\right)^{3/4} \delta_{\max} 	,
		\end{align}
	where $\epsilon>0$ is the accuracy of the numerical differentiation method. Furthermore, once the RHS of \eqref{eq:hattau} is smaller than $1$ and $\epsilon \leq (C m )^{-1}$, the signs returned by Algorithm \ref{alg:initialization} are identical to the ground truth signs. 
	\end{proposition}
The proof is postponed to Section \ref{sec:proofs_initialization} of the supplementary materials. By Theorem \ref{thm:weight_recovery}, we have that $\delta_{\max}$ scales as $(m/\alpha)^{1/4}\epsilon^{1/2}$. Thus, by taking a suitably small $\epsilon$, \eqref{eq:dmax} is satisfied and, after omitting poly-logarithmic factors, the dominant term in \eqref{eq:hattau} scales as
\begin{equation}\label{eq:impapp2}
 \frac{\epsilon^{1/2}}{\alpha^{1/4}}\frac{m^{7/4}}{D^{3/4}}.   
\end{equation}
By comparing \eqref{eq:impapp} and \eqref{eq:impapp2} and recalling that $m$ scales at least linearly in $D$, it is clear that gradient descent improves upon its initialization, under a random model for the weight errors. This improvement is also evident if we evaluate $\Delta_{W, 1}$ on the \emph{actual} weights errors coming from the proposed algorithmic pipeline (cf. Figure \ref{fig:scaling}).

\subsection{Local convergence of gradient descent}\label{sec:gradient_descent}

So far, we have obtained weight approximations $\wW \approx W$ and shift approximations $\htau \approx  \tau$ of the shallow teacher network $f$ defined in \eqref{eq:snn_general}. These parameters $(\wW, \htau)$ allow us to define the neural network $\hat f$ in \eqref{eq:studnet} and, depending on the accuracy of the previous steps, we would expect already a strong similarity between realizations of $\hat f$ and $f$. In this section, we explore to what degree the approximation $\hat f$ can further be improved by tuning the shifts $\htau$ in a teacher-student setting. Assume $x_1,\dots, x_{N_\text{train}}$ generic inputs and access to ${N_\text{train}}$ input-output pairs $(x_i, y_i)_{i\in[N_\text{train}]} = (x_i, f(x_i))_{i\in[N_\text{train}]}$ of the network $f$. Based on the initial network configuration of $\hat f$, we seek to learn the shifts $\tau$ attributed to ${f}$ by minimizing the least-squares objective \eqref{eq:def:loss}
via the gradient descent scheme
\begin{align}\label{eq:gd_iteration}
\htau^{(n+1)} = \htau^{(n)} - \gamma \nabla J(\htau^{(n)}).
\end{align}
Here, $\gamma>0$ represent the step-size of the gradient updates. For the case $\wW =  W$, we show that w.h.p. the gradient descent iteration \eqref{eq:gd_iteration} produces a sequence $(\tau^{(n)})_{n\in \N}$ that converges linearly to $\tau$ provided that $\|\htau^{(0)} - {\tau}\|_2 = \mathcal{O}(m^{-1/2})$. In the perturbed case where $\wW \approx  W$, we provide an analysis that estimates the error of the shifts w.r.t. \emph{(i)} the Frobenius error $\|\wW -  W\|_F$, \emph{(ii)} the alignment between the individual weight errors $\Delta_{W, O}$ (cf. \eqref{eq:def_delta_wo}), and \emph{(iii)} $\|\sum^m_{k=1} w_k - \hw_k \|_2$.
More precisely, for sufficiently many training samples $N_\text{train}$, the gradient descent iteration will settle within distance $\Delta_{W,1}$ of the optimal solution.
\begin{theorem}[Local convergence]\label{thm:local_result} 
Consider the teacher network $f$ defined in \eqref{eq:snn_general}, with shifts ${\tau}_1,\dots, {\tau}_m \in [-\tau_{\infty}, +\tau_{\infty}]$ and weights $w_1, \dots, w_m  \sim \textrm{Uni}(\bbS^{D-1})$ that are incoherent according to Definition \ref{def:assumptions_overcomplete}. Assume $g$ satisfies \ref{enum:activation}-\ref{enum:nonpolynomoial}, and consider the least-squares objective $J$ in \eqref{eq:def:loss} constructed with ${N_\text{train}}\geq m$  network evaluations $y_1,\dots, y_{N_\text{train}}$ of ${f}$, where $y_i = {f}(X_i)$ and  $X_1,\dots, X_{N_\text{train}} \sim \CN(0,\Id_D)$. Let $\hat f$ be parameterized by $\wW$ and $\htau$, as in \eqref{eq:studnet}. 
Then, there exists a constant $C>0$ depending only on $g,\tau_{\infty}$ and $D_0>0$ such that the following holds with probability at least $1- m \exp(-{N_\text{train}} / Cm)-2m^2\exp\left( -t / C\right)$ for $t>0$: Assume $C m \log^2 m \leq D^2, m\geq D\geq D_0$ and
	\begin{align}\label{eq:assumption_thm_local}
			\norm{\tau - {\htau}}_2 + \Delta_W  \leq \frac{1}{C \sqrt{m}},
	\end{align}
	where $\Delta_W = \Delta_{W,1} + \left(\frac{m^3 \delta_{\max}^2 t}{{N_{\text{train}}}}\right)^{1/2}$ and $\Delta_{W,1}$ is given by \eqref{eq:def_delta_w1}.
	Then, the gradient descent iteration \eqref{eq:gd_iteration} with sufficiently small step-size $\gamma>0$ started from $\htau^{(0)} = \htau$ satisfies 
	\begin{align}
	\|\htau^{(n)} - {\tau}\|_2 \leq 2 \xi^n \|\htau^{(0)} -  {\tau}\|_2+C(1-\xi^n )\Delta_W.
\end{align}
\end{theorem}
Note that \eqref{eq:assumption_thm_local} can always be satisfied within our framework as all factors depend on $\epsilon$ which can be chosen freely. The proof of Theorem is in Section \ref{sec:proofs_local} of the supplementary materials.
The idea is to express $J$ as a quadratic form $J(\htau) = (\htau-\tau)^\top A(\htau)(\htau-\tau)$, where $A(\htau)$ denotes the Jacobian obtained by taking derivatives w.r.t. the input features. Then, we linearize around the true solution by replacing $A(\htau)$ with $A(\tau)$. Analyzing the idealized objective $(\htau-\tau)^\top A(\tau)(\htau-\tau)$ requires to guarantee the well-posedness of $A(\tau)$, which we prove by using techniques from the NTK literature adapted to our setting (Section \ref{subsec:well_posedness_idealized}). The error due to the replacement of $A(\htau)$ with $A(\tau)$ depends on the error in the weight approximation, and we control it via a delicate argument exploiting Hermite expansions and the incoherence of the weights (Section \ref{subsec:controlling_perturbation}).

\section{Numerical results}\label{sec:numerics_main_nips}
We corroborate our theoretical results by testing the pipeline of Algorithm \ref{alg:pipeline}, in order to identify parameters of shallow networks of the type ${f}(x) = \sum^m_{k=1} \tanh({w}_k^\top x + {\tau}_k )$. As assumed in the theory, the weights are given by $w_1,\dots, w_m \sim_{\text{i.i.d.}}\textrm{Uni}(\bbS^{D-1})$,
the shifts are sampled according to $\tau_1,\dots, \tau_m \sim_{\text{i.i.d.}} \operatorname{Uni}(-0.5, 0.5)$ and the activation satisfies \ref{enum:activation}-\ref{enum:nonpolynomoial}. The number of neurons $m$ depends on the input dimension $D$ according to the rule $m =  \lceil \frac{2}{5} D^\beta \rceil$, where the  exponent  $1/2 \leq \beta \leq 2$ is referred to  as the \emph{order of neurons}.
The accuracy is evaluated via the following metrics: \emph{(i)} the uniform error (of the approximating network), computed as $E_\infty = m^{-1}\max_{i } \vert f(x_i) - \hat{f}(x_i)\vert$ on a set of $10^6$ unseen Gaussian inputs, \emph{(ii)} the worst weight approximation, i.e., $\max_{k \in [m]} \min_{s \in \{-1,1\}} \|{w}_k - s \hw_k\|_2$, and \emph{(iii)} the error of the shift approximation, i.e., $ m^{-1/2} \|{\tau}-\htau\|_2$. The scaling $m^{-1}$ of $E_{\infty}$ normalizes for the fact that the range of $f(x)$ grows with $m$ according to our network model \eqref{eq:snn_general}. All experiments were performed using one NVIDIA Tesla\textsuperscript{\textregistered} P100 16GB/GPU in a NVIDIA DGX-1.

\paragraph{Baseline.} As a baseline for our pipeline, we first try to identify the network parameters in a standard teacher-student setup. The teacher network is fit by empirical risk minimization via SGD applied to a student network of identical architecture. Using 8 minutes of training time with Tensorflow and the hardware as stated above ($N_{\text{train}}= \frac{5}{2} m \cdot D^2$ teacher evaluations, mini-batch size of 64 and learning rate $0.005$), we obtain the uniform error depicted in the top row on the left in Figure  \ref{fig:nips_main}. These results are averaged over four repetitions. The experiment shows that SGD manages to identify the network parameters and achieve a low uniform error as long as the number of neurons $m$ is small, in particular much smaller than a quadratic scaling such as $m =\lceil \frac{2}{5}D^2\rceil$. Furthermore, the results worsen for growing dimension $D$ despite higher incoherence of the network weights, possibly due to the fixed training time and learning rate. In an attempt to improve these results, we additionally run SGD for 50 minutes and several different learning rates, fixing the case $D = 50$. The results, shown in the top row on the right in Figure \ref{fig:nips_main}, indicate an improvement of SGD for certain hyperparameter combinations, yet we were not able to find a suitable tuning for $D = 50, m= 1000$. For this experiment, we choose $\tau_1,\dots, \tau_m \sim_{\text{i.i.d.}} \mathcal{N}(0, 0.05)$, thus the ground-truth shifts are closer to the initialization (set at $0$) than if they are uniform in $[-0.5, 0.5]$, which should facilitate the task of the SGD algorithm. 

{
\begin{figure}
 		\begin{center}
   		\includegraphics[width=0.65\linewidth]{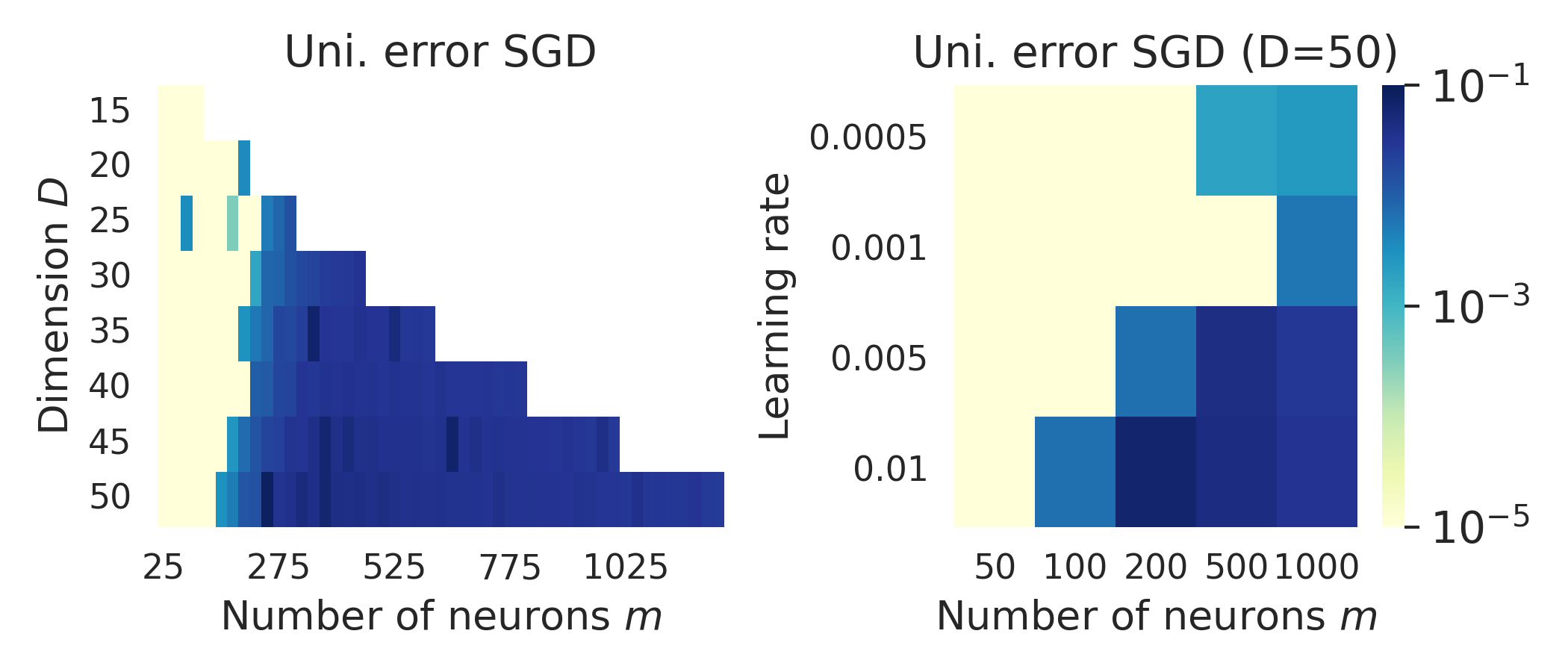}\\
  		\includegraphics[width=0.65\linewidth]{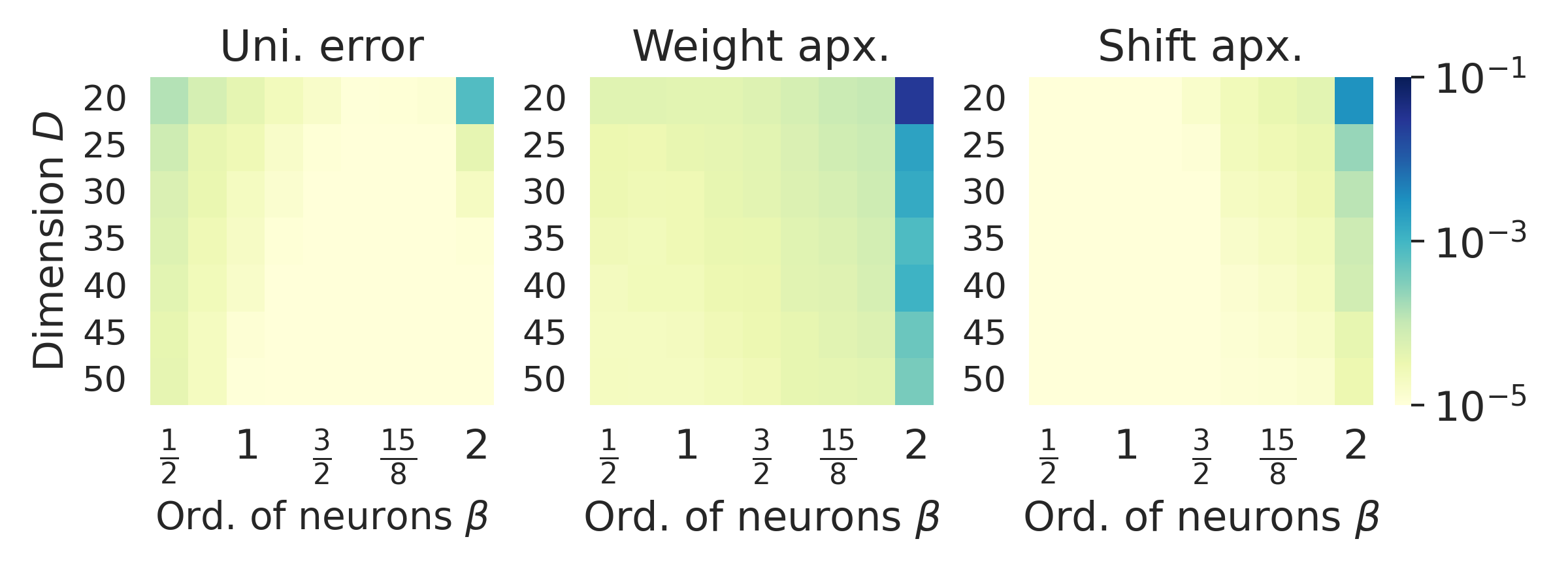}%
   	\end{center}
   	\vspace{-20pt}
   	\caption{
   	Performance of parameter identification of shallow networks with $\tanh$ activations, $m$ neurons and input size $D$ via SGD for shifts $\tau_k \sim \mathcal{N}(0, 0.05)$ (top row), our pipeline for shifts $\tau_k \sim \operatorname{Uni}(-0.5, 0.5)$ (bottom row).  
   	}
  \label{fig:nips_main}
  \vspace{-10pt}
   \end{figure}
   }

\paragraph{Recovery pipeline.} We now discuss the results of our recovery pipeline in Algorithm \ref{alg:pipeline} to identify shallow networks with $\tanh$ activation. For the weight recovery, we use $N_h = \lceil \log(D) m  \rceil$ Hessian approximations, which are computed via central finite differences with step-size $\epsilon_{\text{FD}}=0.01$ and are anchored at evaluations $x_1,\dots, x_{N_h} \sim_{\text{i.i.d.}} \CN(0, \Id_D)$. Then, we run $R=5m\log(m)$ SPM iterations \eqref{eq:pgd_iteration} in parallel for $10^3$ steps with step-size $\gamma =2$. The initial shifts computed by the parameter initialization are finalized via (stochastic) gradient descent as described in Section \ref{sec:rem}. We use $N_{\text{train}} = m \cdot D^2$ samples, learning rate $\gamma = 10^{-3}$ and batch size $64$. The training input points are drawn from a standard Gaussian distribution. The refinement of the shifts (by gradient descent) is timed out after $180$ seconds, or once we reach a training error below $10^{-8}$. 

The results of our pipeline in the bottom row of Figure \ref{fig:nips_main} demonstrate successful recovery of all weights and shifts consistently over $10$ repetitions, and for all combinations of $m,D$. For $\beta=2$ (or $m=\lceil\frac{2}{5} D^2\rceil$) the performance of the weight recovery is worse for small $D$. The causes of this effect may be two-fold: the weights do not yet behave statistically as in the average case scenario for larger $D$; and the gap between $m=\frac{2}{5} D^2$ and the theoretical limit for weight recovery, $m = D(D-1)/2$, decreases in $D$. Moreover, we emphasize that the time spent for the weight recovery (which includes the time necessary to approximate all Hessian matrices by numerical differentiation) is in the order of seconds, reaching a maximum of $112s$ for $D=50, \beta=2$. The overall runtime of the pipeline is below 5 minutes over all individual runs. %

{
\begin{figure}
 		\begin{center}
   		\includegraphics[width=0.65\linewidth]{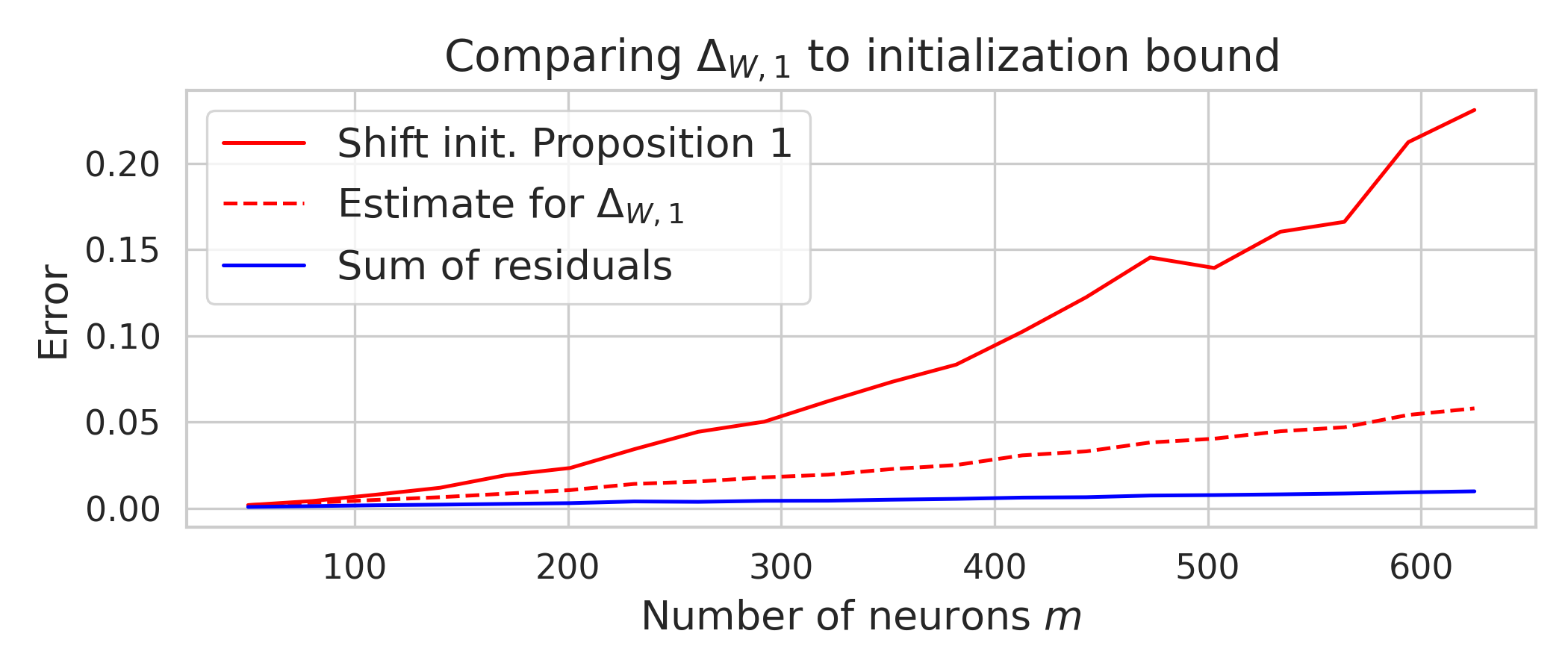}
   	\end{center}
   	\vspace{-20pt}
   	\caption{
   	Comparison between the guaranteed accuracy of the shift initialization (red), the term $\Delta_{W,1}$ (dashed red) and the sum of residual errors $\| \sum^m_{k=1} w_k - \hw_k \|_2$ (blue) for weights approximated by our pipeline and $D = 50$. 
   	}
  \label{fig:scaling}
  \vspace{-10pt}
   \end{figure}
   }
\paragraph{Improvement of the shifts by GD.} In Figure \ref{fig:scaling}, we compare the error bound \eqref{eq:hattau} on the initial shifts with $\Delta_{W,1}$ (cf. \eqref{eq:def_delta_w1}), where for simplicity the constants $C$ are taken to be $1$ in both statements. The results are averaged over $10$ realizations. The plot shows that \emph{(i)} the sum of the residuals $\| \sum^m_{k=1} w_k - \hw_k \|_2$ in blue has only a negligible contribution to $\Delta_{W,1}$, and hence \emph{(ii)} by settling within distance $\Delta_{W,1}$ of the true shifts, GD will improve over the initialization.

\section{Concluding remarks}

In this paper, we give an algorithm with provable guarantees for the  finite sample identification of shallow networks \emph{with biases}, where the number of neurons $m$ is roughly $\mathcal O(D^2)$. By doing so, we improve upon previous work \cite{fornasier2018identification,anandkumar15}, which neglects the role of biases and is limited to $m=\mathcal O(D)$. Let us finally mention that \cite{fornasier2019robust,fiedlerStableRecoveryEntangled2021} have provided partial results on finite sample identification of deep networks. Thus, giving a complete pipeline for the deep case, which keeps into full consideration also the important role of biases, is an interesting future direction. 

\appendix

\section{Proofs: Weight recovery}\label{sec:proofs_weight_recovery}

Algorithm \ref{alg:recover_weights} summarizes the first step of the reconstruction pipeline which is the weight recovery. For more details on the exact procedure we refer to Section \ref{sec:weight_identification}. This section is concerned with the proof of Theorem \ref{thm:weight_recovery}, which provides a uniform bound on the approximation error associated with the weight recovery. Additionally, we characterize the incoherence of the resulting approximated weights in terms of the numerical accuracy. A large part of the proofs in this section will operate under the assumptions that vectors, which are drawn uniformly from a high-dimensional sphere, are well separated. To make this more concrete, we rely on a result due to \cite{kileelLandscapeAnalysisImproved2021a} which allows the application of the deterministic incoherence properties \ref{enum:RIP} -\ref{enum:GInverse} stated in Definition \ref{def:assumptions_overcomplete} to the ground truth weights $w_1, \dots, w_m \in \bbS^{D-1}$ which are modeled by a uniform spherical distribution according to our network model (cf. Section \ref{sec:network_model}). 
\begin{proposition}[{\cite[Prop. 13]{kileelLandscapeAnalysisImproved2021a}}]\label{prop:incoherence_holds_for_uniform}
Let $w_1, \dots, w_m$ be drawn independently from  $\textrm{Uni}(\bbS^{D-1})$. If $m = o(D^2)$, then, for any arbitrary constant $\delta \in (0,1)$, there exist constants $C>0$ and $D_0 \in \N$ depending only on $\delta$ such that for all $D\geq D_0$, and with probability at least \begin{align}
1 - m^{-1} - 2 \exp(-C \delta^2 D) - C\left(\frac{e \cdot D}{\sqrt{m}}\right)^{-C \sqrt{m}}
\end{align}
conditions \ref{enum:RIP} - \ref{enum:GInverse} hold with constants $c_2,c_3 < C$.
\end{proposition}
\begin{algorithm}
	\KwIn{Shallow neural network $f$, number of neurons $m$, number of Hessian locations $N_h$, stepsize $\gamma > 0$, $\beta$ threshold for rejection of spurious local maximizers}
	Draw independent samples $x_1, \dots, x_N \sim \CN(0, \Id)$. \\
	Construct the matrix
	\begin{align*} \widehat{M} :=
		\begin{bmatrix}
			\opvec(\Delta^2 {f}(x_1))& \dots  & \opvec(\Delta^2 {f}(x_N))
		\end{bmatrix} \in \R^{D^2 \times N_h}.
	\end{align*}\\
	Denote by $\mathcal{P}_{\wCW}$ the orthogonal projection onto the $m$-th left singular subspace of $\widehat{M}$. \\
	Define $P_{\wCW}$ as the orth. proj. in matrix space corresponding to $\mathcal{P}_{\wCW}$\\
	Set $\CU  \leftarrow  \emptyset$\\
	\While{$|\CU|< m$}{
		Sample $ u_0 \sim\textrm{Unif}(\bbS^{D-1})$\\
		Iterate projected gradient ascent
			$$  u \leftarrow  P_{\bbS^{D-1}}(  u+ 2\gamma P_{\wCW}((   u)^{\otimes 2})  u)$$
		until convergence, and denote the vector of the final iteration by $\hat{u}$.\\
		\If {$\norm{P_{\wCW}(\hat{u}^{\otimes 2})}_F^2 > \beta$} {
			\If{$\hat{u} \not\in {\CU}$ and $-\hat{u} \not\in {\CU}$}{
				$\CU\leftarrow \CU \cup \{\hat{u}\}$\\
			}
		}
	}
	\KwOut{$\CU$}
	\caption{\textbf{Weight recovery}}
	\label{alg:recover_weights}
\end{algorithm}

\paragraph{Proof Sketch of Theorem \ref{thm:weight_recovery}}
The proof of Theorem \ref{thm:weight_recovery} relies on two individual auxiliary statements. 
First, a subspace approximation bound covered in Lemma \ref{lem:subspace_pertubation}, that controls the error $\norm{P_{\CW} - P_{\wCW}}_{F\rightarrow F}$.
Recall that $\wCW$ is constructed to approximate the matrix space 
\begin{align*}
    \CW = \operatorname{span}\{ w_1 \otimes w_1, \dots w_m \otimes w_m \} \subset \operatorname{Sym}(\R^{D \times D})
\end{align*}
from which individual weights can be identified as the rank-1 spanning elements. It is noteworthy that the proof of Lemma \ref{lem:subspace_pertubation} as well as Algorithm \ref{alg:recover_weights} make use of the following convention: We associated every matrix subspace with a classical vector space induced by vectorization. The vectorization of a matrix is denoted by the operator $\operatorname{vec}(\cdot)$ whose output applied to a matrix $X \in \R^{a\times b}$ is the vector in $\R^{a\cdot b}$ containing the columns of $X$ stacked on top of other, i.e.  
\begin{align*}
    \operatorname{vec}\left(\begin{bmatrix}
        \vline & & \vline \\
        x_1 & \dots & x_b\\
        \vline & & \vline \\ 
    \end{bmatrix} \right)  = \begin{pmatrix} x_1 \\ \vdots \\ x_b \end{pmatrix}.
\end{align*}
This allows us to associate a space like $\CW \subset \operatorname{Sym}(\R^{D\times D})$ with the space 
\begin{align*}
    \operatorname{span}\{\opvec(w_1\otimes w_1), \dots, \opvec(w_m\otimes w_m)\} \R^{D^2}.
\end{align*}

\begin{lemma}\label{lem:subspace_pertubation} Consider the teacher network $f$ defined in \eqref{eq:snn_general}. Assume the activation $g$ satisfies \ref{enum:activation} - \ref{enum:nonpolynomoial}
and $f$ satisfies the learnability condition \ref{enum:learnability} for some $\alpha>0$.
Furthermore assume that the network weights $w_1, \dots, w_m \in \bbS^{D-1}$ fulfill \textbf{(A2)} of Definition \ref{def:assumptions_overcomplete} with constant $c_2$.
Let $P_\CW$ be the orthogonal approximation onto $\CW = \operatorname{span}\{ w_1 \otimes w_1, \dots w_m \otimes w_m \}$ and let $P_{\wCW}$
be constructed as described in Algorithm \ref{alg:recover_weights}.
Then there exists a constant $C>0$ depending only on $g$ and $c_2$, such that for  numerical diff. accuracy $\epsilon < \frac{\sqrt{\alpha}}{C\sqrt{m}}$ and $N_h > t(m + m^2 \log(m) / D)$ for some $t\geq 1$ we have
\begin{align}
\norm{P_{\CW} - P_{\wCW}}_{F\rightarrow F} \leq C \sqrt{m / \alpha } \cdot \epsilon,
\end{align}
with probability at least $1-  D^2 \exp\left(-\frac{t \alpha}{C}  \right)$.
\end{lemma}
\begin{proof}
Consider $X_1, \dots, X_{N_h}$ independent copies of a standard Gaussian, i.e. $X_i \sim  \CN(0, \Id_D)$. Denote by $\mathcal{P}_\CW \in \R^{D^2 \times D^2}$ the orthogonal projection matrix onto $\operatorname{span}\left\{\opvec({w}_k\otimes {w}_k)\, | \,  i =1,\ldots,m\right\}$ and
by $M$ the matrix with columns given by the exact vectorized Hessians at the inputs $X_1,\ldots,X_{N_h}$, i.e.
\begin{align}
M := \begin{bmatrix}
\opvec(\nabla^2 f(X_1)) & \dots &\opvec(\nabla^2 f(X_{N_h}))
\end{bmatrix} \in \R^{D^2\times N_h}.
\end{align}
We associate the matrix subspaces $\CW$ and $\widehat \CW$ with their corresponding $D^2$-dimensional
vector subspaces described by the orthogonal projection matrices $\mathcal{P}_{\CW}, \mathcal{P}_{\wCW}$, respectively. Note that
\begin{align*}
\norm{P_{\CW} - P_{\wCW}}_{F\rightarrow F} &= \sup_{\norm{M}_F=1}\norm{P_{\CW}(M) - P_{\wCW}(M)}_{F} = \norm{\mathcal{P}_{\CW} - \mathcal{P}_{\wCW}}
\end{align*}
with $\|\cdot\|$ describing the ordinary spectral normal in $\R^{D^2}$. Hence, to prove the result,
we can rely on the well-known Wedin bound, see for instance \cite{fornasier2018identification, fornasier2019robust, fiedlerStableRecoveryEntangled2021, kileelLandscapeAnalysisImproved2021a}, giving
\begin{align}\label{eq:wedin_reparam_subspace_error}
\norm{P_{\CW} - P_{\wCW}}_{F\rightarrow F} = \norm{\mathcal{P}_{\CW} - \mathcal{P}_{\wCW}} \leq \frac{\|M-\widehat{M}\|_F}{\sigma_m(\widehat{M})},
\end{align}
for as long as $\sigma_m(\widehat{M}) > 0$.
We continue to provide separate bounds for the numerator and denominator of \eqref{eq:wedin_reparam_subspace_error}.
For the numerator we obtain
\begin{align*}
\|M-\widehat{M}\|_F &\leq \sqrt{N_h } \max_{i\in [N_h ]} \norm{\nabla^2 {f}(X_i) - \Delta^2 {f} (X_i)}_F\\
&\leq \sqrt{N_h m} \max_{\substack{i\in [N_h ] \\ k \in [m]}} \norm{\nabla^2 g(w_k^\top X_i + {\tau}_k) - \Delta^2 g({w}_k^\top X_i + {\tau}_k)}_F \\
&\leq \max_{k\in [m]} C_\Delta \sqrt{N_h m} \norm{{w}_k\otimes {w}_k }_F \epsilon = C_\Delta \sqrt{N_h m} \epsilon,
\end{align*}
where we used the linearity of $\Delta^2, \nabla^2$ in the second step and our assumptions on the numerical differentiation method \ref{enum:numerical_diff} in the last line which gives rise to the constant $C_\Delta$ that only depends on $g$. For the denominator in \eqref{eq:wedin_reparam_subspace_error} we use Weyl's inequality \cite{weylAsymptotischeVerteilungsgesetzEigenwerte1912} which leads to the lower bound 
\begin{align}
\sigma_m(\widehat{M}) \geq \sigma_m(M) - \|M-\widehat{M}\| \geq \sigma_m(M) - \|M-\widehat{M}\|_F \geq  \sigma_m(M) - C_\Delta \sqrt{N_h  m} \epsilon.
\end{align}
Lastly, we need to control $\sigma_{m}(M)$ by a concentration argument in combination with the learnability assumption \ref{enum:learnability} of Section  \ref{sec:network_model}. We first express $\sigma_m(M)$ as sum of independent matrices:
\begin{align}
\sigma_m(M)^2 = \sigma_m(MM^\top) = \sigma_m \left( \sum^{N_h }_{i=1} \opvec(\nabla^2 f (X_i)) \otimes \opvec(\nabla^2 f(X_i)) \right).
\end{align}
Denote $A_i =\opvec(\nabla^2 {f}(X_i)) \otimes \opvec(\nabla^2 {f}(X_i))$. By \ref{enum:learnability} we know that
\begin{align*}
\sigma_m \left( \sum^{N_h }_{i=1}\mathbb{E} A_i \right) =N_h\alpha > 0.
\end{align*}
We will make use of the matrix Chernoff (see \cite{troppUserFriendlyTailBounds2012} Corollary 5.2 and the following remark) which states that
\begin{align}\label{eq:reparam_matrix_chernof}
\mathbb{P}\left( \sigma_m\left( \sum^{N_h }_{i=1} A_i \right) \leq (1-s) \sigma_m\left( \sum^{N_h }_{i=1}\mathbb{E} A_i \right)   \right) \leq D^2 \exp\left(-(1-s)^2 \sigma_m\left( \sum^{N_h }_{i=1}\mathbb{E} A_i \right) / 2K \right)
\end{align}
for $s \in [0,1]$ and $K = \max_{i\in [N_h ]} \norm{A_i}_2$. The norm of $A_i$ can be bound uniformly over all $x \in \R^D$ by
\begin{align*}
\norm{\opvec(\nabla^2 f(X_i)) \otimes \opvec(\nabla^2 {f}(X_i))}_2 &\leq \sup_{x\in \R^D} \norm{\opvec(\nabla^2 {f}(x))}^2_2 = \sup_{x\in \R^D} \norm{\nabla^2 {f}(x)}^2_F \\
&= \sup_{x\in \R^D}\norm{\sum^m_{k=1} g^{(2)}({w}_k^\top x + {\tau}_k)  {w}_k\otimes {w}_k}_F^2 \\
&= \sup_{x\in \R^D} \sum^m_{k,\ell = 1} g^{(2)}({w}_k^\top x + {\tau}_k) g^{(2)}({w}_\ell^\top x + {\tau}_\ell) \langle {w}_k, {w}_\ell \rangle^2 \\
&\leq \kappa^2  \sum^m_{k,\ell = 1} \langle {w}_k, {w}_\ell \rangle^2 \leq \kappa^2 ( m + c_2 m(m-1)\log m / D).
\end{align*}
The last inequality follows by the incoherence assumption \ref{enum:correlation} from the initial statement.
Combining this with \eqref{eq:reparam_matrix_chernof} for $s=1/2$ together with the bound on the spectrum of the expectation yields
\begin{align}
\mathbb{P}\left( \sigma_m\left( \sum^{N_h}_{i=1} A_i \right) \geq \frac{1}{2} N_h\alpha   \right) \geq1-  D^2 \exp\left(-\frac{N_h D \alpha}{8 \kappa^2(Dm+c_2 m^2 \log m )}  \right).
\end{align}
Conditioning on this event, and assuming $\epsilon < \sqrt{\alpha / 8C^2_\Delta m}$ the initial subspace bound now holds as
\begin{align}\label{eq:subspace_eq_last}
\norm{P_{\CW} - P_{\wCW}}_{F\rightarrow F} &\leq \frac{\norm{M-\widehat{M}}_2}{\sigma_m(\widehat{M})} \leq \frac{C_\Delta \sqrt{N_h m} \epsilon}{\sqrt{\frac{1}{2}N_h  \alpha } - C_\Delta \sqrt{N_h m} \epsilon} = \frac{C_\Delta \sqrt{m}\cdot \epsilon}{\sqrt{\frac{\alpha}{2} }- C_\Delta \sqrt{m} \cdot \epsilon}\\
&\leq \frac{\sqrt{2}C_\Delta \sqrt{m}\cdot \epsilon}{\sqrt{{\alpha} }}
\end{align}
with said probability. The final result follows by applying the bound on $\epsilon$ onto the denominator. More precisely, we need that $C > 2C_\Delta$ to fulfill \eqref{eq:subspace_eq_last} and $C > 8\kappa^2 \max\{1, c_2\}$ which implies 
\begin{align*}
1-  D^2 \exp\left(-\frac{N_h D \alpha}{8 \kappa^2(Dm+c_2 m^2 \log m )}  \right) \leq 1-  D^2 \exp\left(-t \alpha / C \right),
\end{align*}
due to our assumption that $N_h > t(m + m^2 \log(m) / D)$. 
\end{proof}
\noindent The second part of Algorithm \ref{alg:recover_weights} performs projected gradient ascent to find the local maximizers of  
\begin{align}\label{eq:program_in_weight_section_proof}
u \mapsto \norm{P_{\wCW}(u\otimes u)}^2_F, u \in \bbS^{D-1}.
\end{align}
The landscape for this functional, for $m = o(D^2)$,  has been recently analyzed (in particular the properties of its local maximizers) in \cite{kileelLandscapeAnalysisImproved2021a} for the general problem of symmetric tensor decomposition. We now provide one of their main statements adopted to the matrix scenario.
\begin{theorem}[{\cite[Theorem 16]{kileelLandscapeAnalysisImproved2021a}}]\label{thm:kilel_et_all_over_main}
Let $m, D \in \N$ such that $m \log^2(m) \leq D^2$. Assume ${w}_1, \dots, {w}_m$ satisfy \ref{enum:RIP} - \ref{enum:GInverse} of Definition \ref{def:assumptions_overcomplete} for some $\delta, c_1, c_2, c_3 > 0$. Then there exists $\delta_0$, depending only on $c_2, c_3$ and $D_0, \Delta_0, C$ which depend additionally on $c_1$, such that if $\delta < \delta_0, D>D_0$ and $\norm{P_{\CW} - P_{\wCW}}_{F\rightarrow F} \leq \Delta_0$, the program \eqref{eq:program_in_weight_section_proof} has exactly $2m$ second-order critical points in the superlevel set
\begin{align}\label{eq:superlevel_set_thm16}
\left\lbrace x \in \bbS^{D-1} \middle| \,  \norm{P_{\wCW}(x\otimes x)}^2_F \geq C  m \log^2(m) / D^2 + 5 \norm{P_{\CW} - P_{\wCW}}_{F\rightarrow F} \right\rbrace.
\end{align}
Each of these critical points is a strict local maximizer for $\operatorname{argmax}_{u \in \bbS^{D-1}} \norm{P_{\wCW}(u\otimes u)}^2_F$. Furthermore, for each such point $x^*$, there exists a unique $k \in [m]$ such that
\begin{align}
\min_{s \in \{-1,1\}}\norm{x^* - s w_k}_2 \leq \sqrt{\norm{P_{\CW} - P_{\wCW}}_{F\rightarrow F}}.
\end{align}
\end{theorem}
\noindent This establishes that the local maximizers of \eqref{eq:program_in_weight_section_proof} that belong to the superlevel set \eqref{eq:superlevel_set_thm16} will be close to one of the weights $w_1, \dots, w_m$ up to sign.
The projected gradient ascent iteration in Algorithm \ref{alg:recover_weights} converges monotonically to one of the constrained stationary points of \eqref{eq:program_in_weight_section_proof} as shown in \cite[Theorem 5.1]{kileelSubspacePowerMethod2021}. We are now ready to prove the main result on the weight recovery which relies on the lemma above, Proposition \ref{prop:incoherence_holds_for_uniform}, and the machinery developed in  \cite{kileelLandscapeAnalysisImproved2021a} represented by Theorem \ref{thm:kilel_et_all_over_main}.
\begin{proof}[Proof of Theorem \ref{thm:weight_recovery}]
The weights ${w}_1, \dots, {w}_m$ of $f$ are drawn uniformly from the unit sphere. By Proposition \ref{prop:incoherence_holds_for_uniform}, and for any $\delta_0\in (0,1)$, there exists $D_1\in \N, C_1>0$ depending only on $\delta_0$ such that for all $D\geq D_1$ this set of weights fulfills conditions \ref{enum:RIP} - \ref{enum:GInverse} of Definition \ref{def:assumptions_overcomplete} with constants $c_2,c_3 < C_1$  and with probability at least 
$$1 - m^{-1} - 2 \exp(-C_1 \delta_0^2 D) - C_1 \left(\frac{e \cdot D}{\sqrt{m}}\right)^{-C_1 \sqrt{m}}.$$
 We condition on this event and denote it by $E_1$ for the remaining part of the proof. Now, due to the incoherence of the weights and according to our initial assumption which includes $N_h > t(m + m^2 \log(m) /D)$, the conditions of Lemma \ref{lem:subspace_pertubation} are met which provides an error bound for the subspace which is constructed in the first part of Algorithm \ref{alg:recover_weights}, such that
\begin{align}\label{eq:subspaceerror_weightproof}
\norm{P_{\CW} - P_{\wCW}}_{F\rightarrow F} \leq  C_2 \sqrt{m / \alpha}\cdot \epsilon,
\end{align}
with probability at least $1-  D^2 \exp\left(-t \alpha / C_2  \right)$ for a constant $C_2$ only depending on $g$. Denote the event that this subspace bound holds by $E_2$ and assume it occurs, which only depends on the number of Hessians $N_h$ in relationship to $D,m$. Note that $\delta_0$ can be freely chosen in $(0,1)$. By Theorem  \ref{thm:kilel_et_all_over_main} there exist constants $D_2, \Delta_0, C_3$, such that for $D\geq D_2$ and $ \norm{P_{\CW} - P_{\wCW}}_{F\rightarrow F}  \leq \Delta_0$, the local maximizers of the program $\operatorname{argmax}_{u \in \bbS^{D-1}} \norm{P_{\wCW}(u\otimes u)}^2_F$ fulfill
\begin{align}\label{eq:proof_weight_recovery_uniform}
\min_{s \in \{-1,1\}}\norm{x^* - s {w}_k}_2 \leq \sqrt{\norm{P_{\CW} - P_{\wCW}}_{F\rightarrow F}} \leq \sqrt{ C_2 \sqrt{m / \alpha} \cdot \epsilon},
\end{align}
as long as they belong to the level set
\begin{align*}
\left\lbrace x \in \bbS^{D-1} \middle| \,  \norm{P_{\wCW}(x\otimes x)}^2_F \geq C_3  m \log^2(m) / D^2 + 5   C_2 \sqrt{m / \alpha} \cdot \epsilon \right\rbrace.
\end{align*}  
By iterating projected gradient ascent until convergence, every vector $\hat{u}$ will be one of the these local maximizers. Also note that by construction all vectors returned by Algorithm \ref{alg:recover_weights} must have unit norm, hence $\CU \subset \bbS^{D-1}$. We need to make sure that level set is not empty, which is guaranteed for  $C_3 m \log^2 (m)/ D^2 \leq \frac{1}{4}$ and $\epsilon \leq \frac{\alpha^{1/2}}{20 C_2 \sqrt{m}}$ which leads to the threshold
\begin{align}
C_3  m \log^2(m) / D^2 +  5   C_2 \sqrt{m / \alpha} \cdot \epsilon  \leq \frac{1}{4} + \frac{1}{4}=\frac{1}{2}.
\end{align}
Therefore, only considering local maximizers that fulfill $\norm{P_{\wCW}(x\otimes x)}^2_F \geq 1/2$ will guarantee that all local maximizers are of the kind which satisfies \eqref{eq:proof_weight_recovery_uniform}. Before we conclude, there are still some points that need to be addressed. To achieve the bound \eqref{eq:proof_weight_recovery_uniform} we had to assume that
 $\norm{P_{\CW} - P_{\wCW}}_{F\rightarrow F}  \leq \Delta_0$. This is true due to \eqref{eq:subspaceerror_weightproof} given the accuracy satisfies  $\epsilon \leq\frac{ \Delta_0 \alpha^{1/2}}{C_3m^{1/2}}$ which is clearly realizable by our initial assumptions on $\epsilon$, since $\Delta_0$ is independent of $m,D$. Hence, by further unifying also the constants $C_1,C_2,C_3, D_1, D_2$, we showed that there exists constants $C>0,D_0\in \N$ such that for $D\geq D_0$ and $Cm \log^2 m\leq D^2$  all vectors $u \in \CU$ returned by Algorithm  \ref{alg:recover_weights} ran with num. accuracy $\epsilon\leq \frac{\sqrt{\alpha}}{C \sqrt{m}}$  will fulfill the uniform error bound, 
\begin{align}
\min_{s \in \{-1,1\}}\norm{x^* - s \bar{w}_k}_2 \leq C  (m / \alpha)^{1/4} \epsilon^{1/2},
\end{align}
and this result holds with the combined probability
\begin{align}
1-  D^2 \exp\left(-t \alpha / C \right) -  m^{-1} - 2 \exp(- D/ C) - C\left(\frac{e \cdot D}{\sqrt{m}}\right)^{- \sqrt{m} / C}.
\end{align}

\end{proof}

The following short result does show a useful property of the spectrum of higher order Grammians which will prove useful for the upcoming part about parameter initialization. 

\begin{lemma}[Higher order Hadamard products]\label{lemma:grammian_order_nondec_minev}
In the setting of Definition \ref{def:assumptions_overcomplete} we have
$\lambda_{\min}(G_{2+\widetilde n}) \geq \lambda_{\min}(G_2)$ and thus in particular
$\norm{G_{2+\widetilde n}^{-1}}_2 \leq c_3$ for all $\widetilde n \in \N_{\geq 0}$ as well.
\end{lemma}
\begin{proof}
For each $n \in \N$ the matrix $G_n$ is a Grammian of the tensors $\{w_1^{\otimes n},\ldots,w_m^{\otimes n}\}$
and as such it is a positive semidefinite matrix. Since $\lambda_{\min}(A\odot B) \geq  \min_{i}a_{ii} \lambda_{\min}(B)$
for any pair of positive semidefinite matrices $A,B$, see \cite[Theorem 3]{bapat1985majorization}, we thus have
\begin{align*}
\lambda_{\min}(G_{2+\widetilde n}) = \lambda_{\min}(G_{2} \odot G_{\widetilde n}) \geq \lambda_{\min}(G_2)\min_{i}\langle w_i,w_i\rangle^{\widetilde n} =  \lambda_{\min}(G_2).
\end{align*}
\end{proof}

Let us conclude this section with an important auxiliary result. As said before we generally operate in a setting where the ground truth weight are sufficiently incoherent and fulfill \ref{enum:RIP} - \ref{enum:GInverse} of Definition \ref{def:assumptions_overcomplete}. It is clear that these properties translate to accurate approximations of the ground truth weights. The following result makes this explicit alongside with a few other minor technical results which will be used throughout the remaining proofs. 

\begin{lemma}[Incoherence of approximated Weights]\label{lemma:incoherence_apx_weigths}
Assume the ground truth weights $\{w_k \in \bbS^{D-1} | k\in [m] \}$ fulfill \ref{enum:RIP} - \ref{enum:GInverse} of Definition \ref{def:assumptions_overcomplete} with constants $c_2,c_3$ and that $D \leq m$. Then there exists a constant $C>0$ only depending on $c_2,c_3$ such that for approximations $\{\hw_k \in \bbS^{D-1} | k\in [m] \}$ which satisfy the error bound
\begin{align}
	\max_{k\in[m]}\min_{s \in \{-1,1\}} \norm{s \hw_k - {w}_k}_2 = \delta_{\max} \leq \frac{1}{C}  \frac{D^{1/2}}{m \sqrt{ \log m}  } \label{eq:assumption_for_inc_apxweights}
\end{align}
condition \ref{enum:correlation} - \ref{enum:GInverse} of Definition \ref{def:assumptions_overcomplete} holds with constants $2c_2 ,2c_3$. Furthermore, denote 
$\tilde{G}_n \in \R^{m\times m}$ the matrix with entries    $\tilde{G}_{n, \ell k } = \langle \hw_\ell, {s}_k{w}_k \rangle^n$, where $s_k$ are the ground truth signs.
Then there exists $D_0$ such that for $m\geq D\geq D_0$, $n=2,3$ the following holds true:  
\begin{enumerate}
    \item[(i)] For all $k\neq \ell$ we have $\langle \hw_k, s_\ell w_\ell \rangle^2 \leq 2c_2\log(m)/D$
    \item[(ii)]  $\tilde{G}_n$ is invertible and $\|\tilde{G}_n^{-1}\| \leq 3 c_2$  
    \item[(iii)]  Denote by $\tilde{G}_n \in \R^{m\times m}$ the matrix with entries    $\tilde{G}_{n, \ell k } = \langle \hw_\ell, {s}_k{w}_k \rangle^n$, then 
	
\begin{align}
	\norm{\tilde{G}_n - \widehat{G}_n} \leq C m \left( \frac{\log m}{D}\right)^{(2n - 1)/4} \delta_{\max}.
\end{align}
\end{enumerate}
\end{lemma}
\begin{proof}
W.l.o.g. we can assume that $C$ is choosen such that 
\begin{align}
	\max_{k\in[m]}\min_{s \in \{-1,1\}} \norm{s \hw_k - w_k}_2 = \delta_{\max} \leq  \min\left\lbrace \frac{1}{8} \left( \frac{c_2 \log m}{D}\right)^{1/2}, \frac{D^{1/2}}{8c_3 m \sqrt{2 c_2  \log m}  }\right\rbrace 
\end{align} holds.
We start by showing \ref{enum:correlation} for the approximated weights. Pick any $k,\ell \in [m], k\neq \ell$. A first observation is that we can disregard the sign that appears in \eqref{eq:assumption_for_inc_apxweights} since $\inner{\hw_k, \hw_\ell}^2 = \inner{-\hw_k, \hw_\ell}^2$ . So w.l.o.g. assume that both signs are correct and therefore $\norm{\hw_k - w_k}_2 \leq \delta_{\max}$ and $\norm{\hw_\ell - {w}_\ell}_2 \leq \delta_{\max}$. Then
\begin{align}
\inner{\hw_k, \hw_\ell}^2 &\leq  \left(  \snorm{\inner{{w}_k, {w}_\ell}} + \snorm{\inner{\hw_k - {w}_k, {w}_\ell}} +\snorm{\inner{ {w}_k, \hw_\ell  - w_\ell}} +\snorm{\inner{\hw_k - w_k,\hw_\ell -  {w}_\ell}} \right)^2 \label{eq:proof_apx_weight_incoherence_componentwise} \\
&\leq \left( \snorm{\inner{{w}_k, {w}_\ell}} + 2\delta_{\max} + \delta_{\max}^2  \right)^2 \leq \snorm{\inner{w_k, w_\ell}}^2 + 6\delta_{\max} \snorm{\inner{w_k, w_\ell}} + 9\delta_{\max}^2 \notag  \\
&\leq \frac{c_2 \log m}{D} + 6 \left( \frac{c_2 \log m}{D}\right)^{1/2}\delta_{\max} + 9 \delta_{\max}^2 \notag  \\ &\leq  \frac{c_2 \log m}{D} + \frac{48+9}{64} \frac{c_2 \log m}{D}  \leq  \frac{2 c_2 \log m}{D} \notag ,
\end{align}
which proves that \ref{enum:correlation} is fulfilled by the approximated weights for a constant $2c_2$. 
Moving on to \ref{enum:GInverse}, we need to bound the minimal eigenvalue of $\widehat G_n = (\wW^\top \wW)^{\odot n}$ from below. Assuming $\widehat G_2$ is invertible, we know by Lemma \ref{lemma:grammian_order_nondec_minev} that 
\begin{align*}
    \| \widehat G_n^{-1}\| \leq \| \widehat G_2^{-1} \| \quad \text{ for all } \quad n\geq 2.
\end{align*}
Thus, it is sufficient to show that \ref{enum:GInverse} holds for the approximated weigths for $n=2$. Denote $ G_2 = (W^\top W)^{\odot 2}$. Clearly $G_2, \widehat G_2$ are symmetric, and since \ref{enum:GInverse} holds for the ground truth weights we know that the minimal eigenvalue of $G_2$ can be bounded by a constant $\snorm{\sigma_m({G}_2)} \geq c_3^{-1}$. Hence, by Weyl's inequality we have
\begin{align}\label{eq:starting_with_weyl_inc}
\snorm{\sigma_m(\widehat G_2)} \geq c_3^{-1} - \norm{\widehat G_2 - {G}_2}.
\end{align}
Our goal is to find an upper bound the spectral norm on the right hand side,. Note that the diagonal of both matrices is identical due to the fact that all columns of $\wW$ and ${W}$ have unit norm, so we focus on the off diagonal exclusively. Via Gershgorin's circle theorem we attain
\begin{align*}
 \norm{\widehat G_n - {G}_n} &\leq \max_{k\in[m]} \sum^m_{\substack{\ell=1 \\ \ell \neq k}}\snorm{ \inner{\hw_k, \hw_\ell}^2 - \inner{w_k, w_\ell}^2 } \\&= \max_{k\in[m]} \sum^m_{\substack{\ell=1 \\ \ell \neq k}}\snorm{ \inner{s_k \hw_k, s_\ell \hw_\ell}^2 - \inner{ w_k, w_\ell}^2 } \\
 &\leq \max_{k\in[m]} \sum^m_{\substack{\ell=1 \\ \ell \neq k}}\snorm{ \inner{s_k\hw_k, s_\ell \hw_\ell} + \inner{w_k, w_\ell}}\snorm{ \inner{s_k \hw_k, s_\ell \hw_\ell} - \inner{w_k, w_\ell} } \\
 &\leq  2 \left(\frac{2 c_2 \log m }{D}\right)^{1/2}\max_{k\in[m]} \sum^m_{\substack{\ell=1 \\ \ell \neq k}}\snorm{ \inner{s_k \hw_k - w_k, s_\ell \hw_\ell} - \inner{w_k, s_\ell \hw_\ell - w_\ell} } \\
 &\leq 4 \left(\frac{2 c_2 \log m }{D}\right)^{1/2} m \cdot \delta_{\max} \leq \frac{1}{2c_3},
 \end{align*}
 where we used the fact that \ref{enum:correlation} holds for the ground truth weights and approximated weights in the penultimate inequality followed by the uniform bound in \eqref{eq:assumption_for_inc_apxweights} at the end. We conclude with Weyl's inequality which yields 
\begin{align}
\snorm{\sigma_m(\widehat G_2^{-1})} \leq  \snorm{\sigma_1(\widehat G_2)}^{-1} \leq 2c_3 .
\end{align}
Hence, the approximated weights fulfill \ref{enum:GInverse} with constant $2c_3$ for $n=2$ which extends to $n\geq 2$ by Lemma \ref{lemma:grammian_order_nondec_minev}. 
Let us now proof $(i)-(iii)$. The first statement follows directly from our proof of \ref{enum:correlation} for the approximated weights, since for any $k\neq \ell$ we have 
\begin{align*}
\langle \hw_k, s_\ell {w}_\ell \rangle^2 \leq \left( \snorm{\langle w_\ell, w_k \rangle} + \snorm{ \langle \hw_\ell - {w}_\ell, 
{w}_k \rangle }\right)^2  \leq  \left( \snorm{\langle {w}_\ell, {w}_k \rangle} + \delta_{\max} \right)^2 \leq \frac{2c_2 \log m}{D},
\end{align*}
follows by the chain of inequalities started in \eqref{eq:proof_apx_weight_incoherence_componentwise}. To show $(iii)$ we first split the difference $\tilde{G}_n - \widehat{G}_n = D_n + O_n$ into a diagonal part $D_n$ and an off-diagonal part $O_n$. We have $\| \tilde{G}_n - \widehat{G}_n \| \leq  \| D_n \| + \| O_n \|$ and start by controlling $\| O_n \|$ via Gershgorin's circle theorem: 
\begin{align*}
\| O_n \| &\leq \max_{\ell \in[m]} \sum^m_{\substack{k=1 \\k \neq \ell}}\snorm{ \inner{\hw_k, \hw_\ell}^n - \inner{\hw_k, s_\ell w_\ell}^n }\\
&\leq \max_{\ell \in[m]} \sum^m_{\substack{k=1 \\k \neq \ell}} \snorm{\inner{\hw_k, \hw_\ell} - \inner{\hw_k, s_\ell w_\ell}}  \snorm{ \sum^{n}_{i=1} \inner{\hw_k, \hw_\ell}^{n-i} \inner{\hw_k, s_\ell w_\ell}^{i-1}  } \\
&\leq  n \left(\frac{2 c_2 \log m }{D}\right)^{(n-1)/2} \max_{\ell \in[m]} \sum^m_{\substack{k=1 \\k \neq \ell}} \snorm{\inner{\hw_k, \hw_\ell - s_\ell w_\ell}}.
\end{align*}
From here we can slightly improve over Cauchy-Schwarz, and instead use that
\begin{align*}
\sum^m_{\substack{k=1\\k\neq \ell}} \snorm{\inner{\hw_k, \hw_\ell - s_\ell w_\ell}} \leq \sqrt{m-1} \sqrt{\sum^m_{\substack{k=1\\k\neq \ell}} \inner{\hw_k, \hw_\ell - s_\ell w_\ell}^2 } \leq \sqrt{m} \norm{\wW} \delta_{\max}. 
\end{align*}
Using $\|\wW\|  =  \|\wW^\top \wW\|^{1/2} \leq \left(1+ m\left(\frac{2c_2 \log m}{D}\right)^{1/2}\right)^{1/2}$ we arrive at the following bound for the off-diagonal terms: 
\begin{align*}
    \| O_n \| &\leq n \left(\frac{2 c_2 \log m }{D}\right)^{(n-1)/2} \sqrt{m}  \left(1+ m\left(\frac{2c_2 \log m}{D}\right)^{1/2}\right)^{1/2}\delta_{\max}\\
	&\leq C n m \left(\frac{ \log m }{D}\right)^{(n-1)/2}   \left(\frac{ \log m}{D}\right)^{1/4}\delta_{\max} \leq C n m \left( \frac{\log m}{D}\right)^{(2n - 1)/4} \delta_{\max},
\end{align*}
where $C>0$ is an absolute constant only depending on $c_2$ and $m\geq D$ was used in the second inequality.
For the diagonal part we receive 
\begin{align*}
    \|D_n \| = \left\vert 1 - \max_{\ell \in[m]} \snorm{ \inner{\hw_\ell , {w}_\ell  }}^n \right\vert \leq  \snorm{1 - (1-\delta_{\max}^2 / 2 )^n}.
\end{align*}
Hence, we attain overall 
\begin{align*}
	\norm{\tilde{G}_n - \widehat{G}_n} &\leq   \snorm{1 - (1-\delta_{\max}^2 / 2 )^n}  +   C n m \left( \frac{\log m}{D}\right)^{(2n - 1)/4} \delta_{\max}.
	\end{align*}
	For $n=2,3$ and some constant $C_1 > 0$ depending only on $c_2$ this can be further simplified using the bound on $\delta_{\max}$ as
	\begin{align*}
	\norm{\tilde{G}_n - \widehat{G}_n} 
	&\leq \delta_{\max}^2 + C n m \left( \frac{\log m}{D}\right)^{(2n - 1)/4} \delta_{\max}\\ 
	&\leq C_1 m \left( \frac{\log m}{D}\right)^{(2n - 1)/4} \delta_{\max},
\end{align*}
which confirms (iii). To prove (ii) we need to show that $\norm{\tilde{G}_n - \widehat{G}_n}  \leq c_4 / 2$ from which the rest follows as before by Weyl's inequality. We can reuse (iii) in combination with \ref{eq:assumption_for_inc_apxweights} obtaining 
\begin{align*}
\norm{\tilde{G}_n - \widehat{G}_n}  \leq  C_1 m \left( \frac{\log m}{D}\right)^{(2n - 1)/4} \delta_{\max}\leq C_2 \left(\frac{\log m}{D}\right)^{1/4}
\end{align*} 
for some constant $C_2$. Hence (ii) is true for $D\geq D_0$ sufficiently large. 
\end{proof}

\section{Proofs: Parameter initialization}\label{sec:proofs_initialization}

\begin{algorithm}[t]
	\KwIn{Approximated weights $\wW$, numerical differentiation schema $\Delta^n[\cdot]$ with accuracy $\epsilon>0$, interval on which $g^{(2)}$ is monotonic $[-\tau_\infty, +\tau_\infty]$}
	Set $\widehat G_2  \leftarrow ( \wW^\top \wW )^{\odot 2}$,$\widehat G_{3}  \leftarrow ( \wW^\top \wW )^{\odot 3}$ \\
	\For{$k=1,\ldots, m$}{
		Compute directional derivative approximations $\tilde T_{2,k} = \Delta^2[f(\cdot \hw_k )](0), \tilde T_{3,k} = \Delta^{3}[f(\cdot \hw_k )](0)$
	}
	Set $\tilde{\mathcal{C}}_2 \leftarrow \widehat G_2^{-1} \tilde T_2$, $\tilde {\mathcal{C}}_{3} \leftarrow \widehat G_{3}^{-1} \tilde T_{3}$\\
	\For{$k=1, \dots, m$}{
		\begin{align}
			\tilde{\tau}_k &\leftarrow{\begin{cases}
					{(g^{(2)})}^{-1}(\tilde{\mathcal{C}}_{2,k}), &\text{if }  {(g^{(2)})}^{-1} \text{ is defined for } \tilde{\mathcal{C}}_{2,k},\\
					\operatorname{argmin}_{t \in [-{\tau}_\infty, \tau_{\infty}]} \snorm{ g^{(2)}(t) - \tilde{\mathcal{C}}_{2,k} } & \text{else},
			\end{cases}}\\
			\tilde{s}_k &\leftarrow \operatorname{sign}\Big(\tilde{\mathcal{C}}_{3, k}\cdot g^{(3)}(0) \Big),\label{def:init_signs} \end{align}
	}
	\KwOut{$\tilde{\tau}, \tilde{s}$}
	\caption{\textbf{Parameter Initialization}}
	\label{alg:initialization}
\end{algorithm}


In this section we proof Proposition \ref{prop:initialization}, which asses the quality of shifts computed by Algorithm \ref{alg:initialization}. These initial shifts will later be used as an initialization for gradient descent (cf. Section \ref{sec:proofs_local}). 
\paragraph{Proof Sketch of Proposition \ref{prop:initialization}}
As discussed in Section \ref{sec:initialization}, goal of Algorithm \ref{alg:initialization} is to recover the vectors
\begin{align*}
	\mathcal{C}_2 = g^{(2)}(\tau), \quad \text{and} \quad \mathcal{C}_3 = s \odot g^{(3)}(\tau).
\end{align*}
This recovery is only possible up to a approximations $\tilde{\mathcal{C}}_2, \tilde{\mathcal{C}}_3$ due to perturbations accumulated in the weight recovery and errors caused by the numerical approximation of derivatives. The proof begins with an auxiliary statement, namely Lemma \ref{lemma:approximation_S_error}, that develops an upper bound on $\|\mathcal{C}_n -  \tilde{\mathcal{C}}_n \|_2$  ($n=2,3$) assuming that the weight recovery achieved a certain level of accuracy. The proof of Proposition \ref{prop:initialization} will then utilize the properties of the activation function (\ref{enum:activation}-\ref{enum:nonpolynomoial}) to show that the shifts $\tau$ can be approximated by using the components of $\tilde{\mathcal{C}}_2 \approx g^{(2)}(\tau)$, whereas the signs of the original weights are revealed by $\tilde{\mathcal{C}}_3 \approx s \odot g^{(3)}(\tau)$. 

\begin{lemma}\label{lemma:approximation_S_error} Denote by $\mathcal{\tilde{C}}_n$ the coefficient vectors computed by Algorithm \ref{alg:initialization} for an input network ${f}$ 
with ground truth weights $\{w_k \in \bbS^{D-1} | k\in [m] \}$ which fulfill \ref{enum:correlation} - \ref{enum:GInverse} of Definition \ref{def:assumptions_overcomplete} with constants $c_2,c_3$ and activation $g$ that fulfills \ref{enum:activation}. Then there exist constants $C>0$ only depending on $g, c_2,c_3$ and $D_0\in \N$, such that for $m \geq D\geq D_0, m \log^2 m \leq D^2$, $n=2,3$ and provided approximations $\{\hw_k \in \bbS^{D-1} | k\in [m] \}$ to the ground truth weights such that 
\begin{align}\label{eq:bound_delta_max_arbb}
	\max_{k\in[m]}\min_{s \in \{-1,1\}} \norm{s \hw_k - w_k}_2 = \delta_{\max} \leq \frac{1}{C} \frac{D^{1/2}}{m \sqrt{ \log m}  },
\end{align}
	we have
	\begin{align}\label{eq:approximation_S_error_main}
		\norm{\tilde{\mathcal{C}}_n -  {s}^n \odot g^{(n)}({\tau})}_2 \leq C\sqrt{m}\epsilon + C m^{3/2} \left( \frac{\log m}{D}\right)^{(2n - 1)/4} \delta_{\max},
	\end{align}
	where $s$ is the vector storing the true signs that are implied by \eqref{eq:bound_delta_max_arbb}.
\end{lemma}
\begin{proof}[Proof of Lemma \ref{lemma:approximation_S_error}]
	Denote as in Algorithm \ref{alg:initialization} $\tilde T_{n,k} = \Delta^n[f(\cdot \hw_k )](0)$ and $T_{n,k} = \langle \nabla^n {f}(0), \hw_k^{\otimes n}\rangle$. By their definition and the linearity of $\nabla^n, \Delta^n$ we have
	\begin{align}
		\norm{T_n - \tilde{T}_n}_\infty &= \sup_{k\in [m]} \snorm{\langle \nabla^n f(0), \hw_k^{\otimes n}\rangle -  \Delta^n_\epsilon [f(\hw_k \cdot)](0) }\\
		&\leq \sup_{k\in [m]} 
			\sum^m_{\ell=1} \snorm{  \left. \frac{\partial^n}{\partial t^n}  g(\langle \hw_k, w_\ell \rangle t + {\tau}_\ell)\right|_{t=0} - \Delta^n [g(\langle \hw_k, w_\ell \rangle \cdot + {\tau}_\ell)](0)}\\
		&\leq C_\Delta \epsilon   \sup_{k\in [m]} 
					\sum^m_{\ell=1}   \snorm{\inner{\hw_k, {w}_\ell }}^{n+2}  \leq C_\Delta \epsilon \left(1 + m \left(\frac{ 2 c_2 \log m }{D}\right)^{\frac{n+2}{2}} \right),
	\end{align}
	where we used the second point of \ref{enum:numerical_diff} in the last line followed by the incoherence of the apx. weights  \ref{enum:correlation} established in Lemma \ref{lemma:incoherence_apx_weigths}. Making use of $D^2 \geq m \log^2 m$, this simplifies to 
	\begin{align*}
	\norm{T_n - \tilde{T}_n}_\infty \leq C_1\cdot \epsilon
	\end{align*}
	with constant $C_1 = (1+ 4c_2^2)C_\Delta$ for $n=2,3$.  Coming back to our initial objective, we can express $s^n \odot g^{(n)}({\tau})$ as the product ${s}^n \odot g^{(n)}({\tau}) = T_n \tilde{G}_n$ where $\tilde{G}_n$ describes the matrix with entries given by $(\tilde{G}_n)_{k\ell} = \langle \hw_k, {s}_\ell {w}_\ell\rangle^n$. Note that Algorithm \ref{alg:initialization} constructs $\tilde{\mathcal{C}}_n  = \widehat G_n^{-1}\tilde{T}_n$, where $(\widehat{G}_n)_{k\ell} = \langle \hw_k,  {\hw}_\ell\rangle^n$. We can reduce our main statement \eqref{eq:approximation_S_error_main} into separate bounds
	\begin{align}
		\norm{\tilde{\mathcal{C}}_n -  {s}^n \odot g^{(n)}({\tau})}_2 &= \norm{\widehat G_n^{-1}\tilde{T}_n - \tilde{G}_n^{-1}T_n}_2 \\
	    &\leq \norm{\widehat G_n^{-1}(T_n - \tilde{T}_n)}_2 + \norm{(\widehat G_n^{-1} -\tilde{G}_n^{-1}) T_n }_2\\
		&\leq \sqrt{m} \norm{\widehat G_n^{-1}} \norm{T_n - \tilde{T}_n}_\infty + \norm{(\widehat G_n^{-1} -\tilde{G}_n^{-1}) T_n }_2 
		\\&\leq C_1 \sqrt{m} \cdot \epsilon + \norm{(\widehat G_n^{-1} -\tilde{G}_n^{-1}) T_n }_2 
	\end{align}
	To bound $\norm{(\widehat G_n^{-1} - \tilde{G}^{-1}_n) T_n }_2$ we first decompose according to
	\begin{align}
		\norm{(\widehat G_n^{-1} - \tilde{G}^{-1}_n) T_n }_2 = \norm{\widehat G_n^{-1}(\widehat G_n - \tilde{G}_n)\tilde{G}^{-1}_n T_n}_2 = \norm{\widehat G_n^{-1}(\widehat G_n - \tilde{G}_n) ({s}^n \odot g^{(n)}({\tau}))}_2.
	\end{align}
By invoking  Definition \ref{enum:GInverse} again we continue with
\begin{align}
	\norm{\widehat G_n^{-1}(\widehat G_n - \tilde{G}_n) ({s}^n \odot g^{(n)}({\tau}))}_2 \leq 2c_3 \norm{\widehat G_n - \tilde{G}_n} \norm{{s}^n \odot g^{(n)}({\tau})}_2 \leq 2c_3 \kappa \sqrt{m} \norm{\widehat G_n - \tilde{G}_n},
\end{align}
where we used $\norm{g^{(n)}}_\infty \leq \kappa$. The statement then follows by using inequality (iii) of Lemma \ref{lemma:incoherence_apx_weigths} onto $\norm{\widehat G_n - \tilde{G}_n}$ and unifying the involved constants.
\end{proof}

We are now ready to prove the main result for the parameter initialization.
\begin{proof}[Proof of Proposition \ref{prop:initialization}]
First note that due the assumptions made, we can freely apply the results of Lemma \ref{lemma:incoherence_apx_weigths} and Lemma \ref{lemma:approximation_S_error}. As a consequence the approximated weights considered in the statement of Proposition \ref{prop:initialization} fulfill \ref{enum:correlation}-\ref{enum:GInverse} of Definition \ref{def:assumptions_overcomplete} with constants derived from the ground truth weights as described in Lemma \ref{lemma:incoherence_apx_weigths}. We continue with the remark that \ref{enum:activation} guarantees the existence of the inverse function ${g^{(2)}}^{-1}$ on $[-{\tau}_{\infty}, {\tau}_{\infty}]$ and here we can disregard the signs such that 
	\begin{align}
		{g^{(2)}}^{-1} ({s}^2 \odot g^{(2)}({\tau})) = {g^{(2)}}^{-1} \left( 1 \odot g^{(2)}(\tau) \right) = {\tau}
	\end{align}
	While ${s}^2 \odot g^{(2)}({\tau})$ is not directly available, $\tilde{\mathcal{C}}_2$ serves as an approximation $\tilde{\mathcal{C}}_2 \approx {s}^2 \odot g^{(2)}({\tau})$. Fix any $k\in [m]$, and assume that 
	
	\begin{align}\label{eq:interval_init_proof} 
	\tilde{\mathcal{C}}_{2,k} \in \left[\min_{t\in [-\tau_{\infty}, +\tau_{\infty}]}g^{(2)}(t), \max_{t\in [-\tau_{\infty},+\tau_{\infty}]}g^{(2)}(t)\right],
	\end{align}
	then by the mean value theorem
	\begin{align*}
		{\htau}_k &=   {g^{(2)}}^{-1}(\tilde{\mathcal{C}}_{2,k}) = {g^{(2)}}^{-1}\left(g^{(2)}(\tau_k) + \tilde{\mathcal{C}}_{2,k} - g^{(2)}(\tau_k)\right) \\&= {g^{(2)}}^{-1}\left(g^{(2)}(\tau_k) \right)+
		\left(\tilde{\mathcal{C}}_{2,k} - g^{(2)}(\tau_k)\right) ({g^{(2)}}^{-1})'(\xi_k)\\
		&= {\tau}_k + \left(\tilde{\mathcal{C}}_{2,k} - g^{(2)}(\tau_k)\right) \frac{1}{g^{(3)}\left({g^{(2)}}^{-1}(\xi_k)\right)}
	\end{align*}
	for some $\xi_k\in \left[\min_{t\in [-\tau_{\infty}, +\tau_{\infty}]}g^{(2)}(t), \max_{t\in [-\tau_{\infty},+\tau_{\infty}]}g^{(2)}(t)\right]$. Since $g^{(2)}$ is strictly monotonic on $[-{\tau}_{\infty}, {\tau}_{\infty}]$ and differentiable we have
	\begin{align*}
	\theta := \max_{t\in [-{\tau}_{\infty}, {\tau}_{\infty}]}\snorm{g^{(3)}(t)} > 0.
	\end{align*}
	 Hence, we can bound $\snorm{\frac{1}{g^{(3)}\left({g^{(2)}}^{-1}(\xi_k)\right)}} \leq \theta^{-1}$ from the outgoing assumption \ref{enum:activation}.  Applying Lemma \ref{lemma:approximation_S_error} to bound $\norm{g^{(2)}(\tau) - \tilde{\mathcal{C}}_2}_2$ therefore yields 
	\begin{align}\label{eq:bound_tau}
		\norm{\hat{\tau} - {\tau}}_2 \leq
		\theta^{-1} \left(C\sqrt{m}\epsilon + C m^{3/2} \left( \frac{\log m}{D}\right)^{3/4} \delta_{\max} \right)
	\end{align}
	Now assume there is a $k\in [m]$ such that $\tilde{\mathcal{C}}_{2,k}$ does not satisfy \eqref{eq:interval_init_proof}. By the monotonicity we also know that the maximal and minimal value of $g^{(2)}$ are found exactly on $\pm \tau_\infty$. If $\tilde{\mathcal{C}}_{2,k}$ does not lie in the image of $g^{(2)}$ on $[-\tau_{\infty}, +\tau_{\infty}]$ it has to exceed one of those. We can assume w.l.o.g. that $\tilde{\mathcal{C}}_{2,k} > \max_{t\in [-\tau_{\infty}, +\tau_{\infty}]}g^{(2)}(t) = g^{(2)}(\tau_\infty)$. Then
	\begin{align*}
		\snorm{g^{(2)}(\tau_\infty)- g^{(2)}({\tau}_k)} < \snorm{\tilde{\mathcal{C}}_{2,k} - g^{(2)}({\tau}_k)},
	\end{align*}
	which shows that $g^{(2)}(\tau_\infty)$ is simply a better estimate of $g^{(2)}({\tau}_k)$ than $\tilde{\mathcal{C}}_{2,k}$, and ${g^{(2)}}^{-1}$ is also defined for $g^{(2)}(\tau_\infty)$. Hence, the same error bound as above holds for all $k\in [m]$.
	The expression in \eqref{def:init_signs} yields the correct sign if $\operatorname{sign}(\tilde{\mathcal{C}}_{3,k}) = \operatorname{sign}({s}_k^{(3)} )\cdot \operatorname{sign}( g^{(3)}({\tau}_k))=\operatorname{sign}({s}_k )\cdot \operatorname{sign}( g^{(3)}({\tau}_k))$.
	This is the case if
	\begin{align}\label{eq:needed_for_signs}
		\snorm{{s}_k^{(3)} \cdot  g^{(3)}({\tau}_k)} > \snorm{s_k^{(3)} \cdot  g^{(3)}({\tau}_k)- \tilde{\mathcal{C}}_{3,k}}.
	\end{align}
	By our outgoing assumption $\snorm{s_k^{(3)} \cdot  g^{(3)}({\tau}_k)}\geq \theta$ and together with Lemma \ref{lemma:approximation_S_error} applied to the right hand side of the inequality above we get that the signs are correct as long as 
	\begin{align}\label{eq:prop_interediate_bound_signs}
	 \theta >  \left(C\sqrt{m}\epsilon + C m^{3/2} \left( \frac{\log m}{D}\right)^{5/4} \delta_{\max}\right).
	\end{align}
	Assume now that the RHS of  \eqref{eq:hattau} is smaller than $1$ and $\epsilon \leq (C m)^{-1}$, this implies in particular
	$$ C m^{3/2} \left( \frac{\log m}{D}\right)^{3/4} \delta_{\max} < 1.$$	
	We can estimate the right hand side of \eqref{eq:prop_interediate_bound_signs} from above by
	\begin{align*}
		C\sqrt{m}\epsilon + C m^{3/2} \left( \frac{\log m}{D}\right)^{5/4} \delta_{\max} \leq \frac{1}{m^{1/2}} + \left( \frac{\log m}{D} \right)^{2/4},
	\end{align*}
	which clearly is smaller than any constant for $D$ large enough, and therefore the signs will be correct for $D_0$ chosen accordingly since \eqref{eq:needed_for_signs} is fulfilled. 
	\end{proof}

\section{Proof of Theorem \ref{thm:local_result}}\label{sec:proofs_local}
\label{sec:proof_of_theorem_local}

Let us shortly recall the setting of Theorem \ref{thm:local_result}. We consider the identification of the parameters ${W}, {\tau}$ attributed to a shallow neural network $f(\cdot, {W}, {\tau})$ which falls into the class of networks described in Section \ref{sec:network_model}. By means of Algorithms \ref{alg:recover_weights} -\ref{alg:initialization}, we can find weight approximations $\wW \approx  W$ and shift approximations $\htau \approx  \tau$ of $f$. The parameters $(\wW, \htau)$ give rise to a neural network $\hat f(\cdot, \wW, \htau)$ which is architecturally identical to ${f}$, and, depending on the accuracy of the previous algorithmic steps, we would already expect some agreement in terms of $\hat f \approx  f$. Given network evaluations $ y_1 = f(x_1),\dots, y_{N_{\text{train}}} = f(x_{N_{\text{train}}})$ of ${f}$, we consider further improvement of the approximated shifts $\htau$ by empirical risk minimization of the objective 
\begin{align}\label{eq:gd_loss_proof}
J(\htau) = \frac{1}{2{N_{\text{train}}}} \sum^{{N_{\text{train}}}}_{i=1} \left( \hat f(x_i, \htau)  - y_i\right)^2,
\end{align}
via gradient descent given by 
\begin{align}\label{eq:gd_iteration_proof}
\htau^{(n+1)} = \htau^{(n)} - \gamma \nabla J(\htau^{(n)}).
\end{align}
In Theorem \ref{thm:local_result}, we prove a local convergence result with the guarantee that, for sufficiently large $N_{\text{train}}$, $\|\htau^{(n)} - {\tau}\|_2$ is roughly 
\begin{align*}
	\|\htau^{(n)} - {\tau}\|_2 \lesssim  \frac{ m^{1 / 2} \log(m)^{3/4}}{D^{1 / 4}}  \left( 
		\| \wW - W\|_F  + \frac{\Delta_{W, O}^{1/2}}{D^{1 / 2}} + \left\| \sum^m_{k=1} w_k - \hw_k \right\|_2
	\right). 
\end{align*}

where $$\Delta_{W, O} = \sum^m_{k \neq k'} \left|\inner{\hw_{k} - w_k, \hw_{k'} - w_{k'}}\right|.$$
\paragraph{Proof sketch.} For the proof, we rely on an idealized loss given by a quadratic functional in $\htau$:
\begin{align}\label{eq:def_Jbar_proof}
J_*(\htau) =  (\htau - \tau)^\top A (\htau- \tau),
\end{align}
with 
\begin{align}\label{eq:def_A}
A := \frac{1}{2{N_{\text{train}}}}\sum^{N_{\text{train}}}_{i=1}\nabla  \hat f(x_i, {\tau})\nabla \hat f(x_i, {\tau})^\top.
\end{align}
The proof can then be broken down into two steps. First, in Lemma \ref{lem:lower_bound_expectation_technical}, it is shown that $J_*$ is strictly convex by estimating a lower bound on the minimal eigenvalue $\lambda_m(A)$ of $A$. The proof relies on techniques from the NTK literature to first control the spectrum of $\mathbb{E}_{X_1,\ldots,X_{N_{\text{train}}} \sim \CN(0,\Id_D)}[A]$ by leveraging  \ref{enum:nonpolynomoial} and the incoherence of $\hw_1, \dots, \hw_m$. In particular, Lemma \ref{lem:lower_bound_expectation_technical} implies that minimizing $J_*$ via the gradient descent iteration given by
\begin{align}\label{eq:idealized_gd_iteration}
\htau_*^{(n+1)} = \htau_*^{(n)} - \gamma \nabla J_*(\htau_*^{(n)}) = \htau_*^{(n)} - \gamma A(\htau_*^{(n)}-\tau) 
\end{align}
with step-sizes $\gamma \leq 1 / \|A \|$ 
does necessarily converge to the global minimum attained at $\htau_* = \tau$. As a second step, we control the perturbation between the iterations $\htau^{(n)}, \htau_*^{(n)}$, when starting them from an identical vector $\htau^{(0)}=\htau_*^{(0)}$. In particular, in Lemma \ref{lemma:delta_gradient_descent_iteration} it is shown that the difference $\| \htau^{(n)} - \htau_*^{(n)}\|_2$ adheres to   
\begin{align*}
\| \htau^{(n)} - \htau_*^{(n)}\|_2 \leq  \xi^n \|\htau^{(0)} - {\tau}\|_2+ \left(1- \xi^n \right)\Delta_W,
\end{align*}
provided $\htau^{(0)}$ is sufficiently close to the optimal solution $\tau$, and where $\Delta_W$ is an error term which satisfies $\Delta_W \to 0$ as $\| \wW - {W}\|_F \to 0$ and $\xi \in [0,1)$. By the triangle inequality, we then bound the distance of the original gradient descent iteration \eqref{eq:gd_iteration_proof} to $\tau$ via 
\begin{align*}
\| \htau^{(n)} - \tau \|_2 &\leq  \| \htau^{(n)} - \htau_*^{(n)}\|_2 + \|\htau_*^{(n)} - \tau \|_2 \\ 
&\leq  \xi^n \|\htau^{(0)} - {\tau}\|_2+ \left(1- \xi^n \right)\Delta_W +  (1-\gamma \lambda_{m}(A))^n\|\htau_*^{(0)} - \tau \|_2 \to \Delta_W ,
\end{align*}
for $n\to \infty$. Hence, we establish that the iteration $\htau^{(n)}$ settles in an area around the optimal shifts $\tau$ that is determined by the initial and irreparable error present in the weight approximation $\wW$ of ${W}$. 
\paragraph{Organisation of this section.}
Subsection \ref{subsec:well_posedness_idealized} is dedicated to analyze the matrix $A$ in \eqref{eq:def_A} in expectation (over $x_i$'s) and proves the well-posedness. Subsection \ref{subsec:controlling_perturbation} analyzes the perturbation between gradient descent on the idealized objective $J_{*}$ and the true objective $J$. Subsection \ref{subsec:concluding_theorem4} concludes the proof by combining the well-posedness and the perturbation analysis.
\subsection{Well-posedness of the idealized objective in expectation}
We begin this section with a short primer on Hermitian expansions, a technical tool which is commonly used in the NTK literature. Afterwards, we prove the well-posedness of $A$ in \eqref{eq:def_A} in expectation.
\label{subsec:well_posedness_idealized}
\subsubsection{A primer on Hermitian expansions}
The Hermitian polynomials form an orthonormal basis of the $L_2$ space,
weighted by the Gaussian kernel $w_G$, which we denote as $L_2(\R, w_G)$. The $r$-th Hermitian polynomial is defined as
\begin{align*}
h_r(y) := \frac{1}{\sqrt{r!}}(-1)^{r}\exp\Big(\frac{y^2}{2}\Big)\frac{d^r}{dy^r}\exp\Big(\frac{-y^2}{2}\Big).
\end{align*}
Any function $h \in L_2(\R, w_G)$ can be expanded as $h \equiv \sum_{r}\mu_r(h)h_r$
with Hermitian coefficients $\mu_r(h)$ as
\begin{align*}
\mu_r(h) := \int h(y)h_r(y)w_G(y)dy.
\end{align*}
As per Assumption \ref{enum:activation} the first three derivatives of $g$ are bounded, hence $\max_{k \in [3]} \|g_{\tau}^{(k)}\|_\infty < \infty$ for any $\tau \in \R$. It is easy to
check that this implies that these functions lie within $L_2(\R, w_G)$.
\begin{lemma}\label{lemma:bounded_in_gaussianL2}
	Assume $h$ is bounded, then $h \in L_2(\R, \omega_G)$ and 
	\begin{align*}
		\sum_{r\geq 0} \mu_r(h)^2 \leq \sqrt{2\pi} \|h\|_\infty^2
	\end{align*}
\end{lemma}
\begin{proof}
	\begin{align*}
		\int_\R h(t)^2 \exp(-t^2 / 2)dt \leq \sqrt{2\pi} \|h\|_\infty^2 < \infty. 
	\end{align*}
	The second statement follows from the fact that $L_2(\R, \omega_G)$ is a Hilbert space and the hermite polynomials form an orthonormal system within that space. 
\end{proof}

We further assume in \ref{enum:nonpolynomoial} that $g^{(1)}$ is not a polynomial of degree three or less, implying that
also $g^{(1)}_{\tau}$ is not a polynomial of degree three or less. Since $h_0,h_1,h_2,h_3$ form
a basis for the space of affine functions, this implies $g^{(1)}_{\tau} \not\in \textrm{Span}(h_0,h_1,h_2,h_3)$.
In particular, $\mu_{r}(g^{(1)}_{\tau}) \neq 0$ for some $r \geq 4$ and any $\tau \in \R$. In the following,
we denote
$$
\omega := \min_{\tau \in [-{\tau}_{\infty},{\tau}_{\infty}]}\sum_{r \geq 4} \mu_r(g^{(1)}(\cdot + \tau))^2 > 0,
$$
which depends only on the activation function $g^{(1)}$ and the shift bound $\tau_{\infty}$. Lastly,
a useful property of Hermitian expansions and the Hermitian basis is the following identity.
\begin{lemma}[{\cite[Lemma D.2]{QuynhMarco2020}}]
\label{lem:mondelli_statement}
For two unit norm vectors $x,y \in \R^D$ and every $k, \ell \geq 0$ we have
\begin{align*}
	\mathbb{E}_{X\sim \mathcal{N}(0,\Id_D)}\left[h_k(v^\top X)h_\ell(u^\top X)\right] = \delta_{k \ell} \langle u, v\rangle^k ,
\end{align*}
where $\delta_{k\ell}=1$ if $k=\ell$ and 0 otherwise.
\end{lemma}
\subsubsection{Well-posedness in expectation}
The central object of study in this section is the matrix
\begin{align}
\label{eq:definition_expectation}
E:= \mathbb{E}_{X_1,\ldots,X_{N_{\text{train}}} \sim \CN(0,\Id_D)}[A].
\end{align}
We prove its well-posedness in Lemma \ref{lem:lower_bound_expectation_technical}. The proof relies on the observation that $E$ is actually equal to a sum of positive semidefinite Grammian matrices as shown in Lemma \ref{lem:expectation_full}.
\begin{lemma}
\label{lem:expectation_full}
Assume that \ref{enum:activation} holds, and let  $E$  be defined as in \eqref{eq:definition_expectation}. Then, we have
\begin{align*}
	E = \frac{1}{2}\sum_{r=0}^{\infty}T_r T_r^\top,\quad \textrm{ where } T_r :=
	\left[
	\begin{array}{c}
		\mu_r(g^{(1)}_{\tau_1})\operatorname{vec}(\hw_1^{\otimes r}) \\
		\vdots\\
		\mu_r(g^{(1)}_{\tau_m})\operatorname{vec}(\hw_m^{\otimes r})\\
	\end{array}
	\right] \in \R^{m\times D^r}.
\end{align*}
In particular, we have $E \succcurlyeq \frac{1}{2} \sum_{r \in \mathcal{R}} T_r T_r^\top$ for any subset $\mathcal{R} \subseteq \N_{\geq 1}$, where $A \succcurlyeq B$ means $A-B$ is positive semidefinite.
\end{lemma}
\begin{proof}
The matrix $A$ can be written as
\begin{align*}
	A_{k\ell} &= \frac{1}{2{N_{\text{train}}}}\sum_{i=1}^{{N_{\text{train}}}} g^{(1)}(\hw_k^\top X_i+\tau_k) g^{(1)}(\hw_\ell^\top X_i+{\tau}_\ell)
\end{align*}
and the corresponding expectation reads
\begin{align*}
	E_{k\ell} = \frac{1}{2}\mathbb{E}_{X\sim\CN(0,\Id_D)}\left[g^{(1)}_{{\tau}_k}(\hw_k^\top X) g^{(1)}_{{\tau}_\ell}(\hw_\ell^\top X)\right].
\end{align*}
Now, note that $g^{(1)}_{\tau} = g^{(1)}(\cdot + \tau) \in L_2(\R, w_H)$ for any $\tau \in \R$ by \ref{enum:activation} and Lemma \ref{lemma:bounded_in_gaussianL2}. Hence, $g^{(1)}_{\tau}$ has
a Hermitian expansion and we can write
\begin{align*}
	E_{k\ell} = \frac{1}{2} \mathbb{E}_{X\sim\CN(0,\Id_D)}&\left[\Bigg(\sum_{r=0}^{\infty}\mu_r(g^{(1)}_{{\tau}_k}) h_r (\hw_k^\top X)\Bigg) \left(\sum_{r=0}^{\infty}\mu_r(g^{(1)}_{{\tau}_\ell})h_r(\hw_\ell^\top X)\right)\right]	.
\end{align*}
 Using now Lemma \ref{lem:mondelli_statement} to express expectations of scalar products of Hermite polynomials, we obtain
\begin{align*}
	E_{k\ell} &= \frac{1}{2}\sum_{r=0}^{\infty} \mu_r(g^{(1)}_{{\tau}_k})\mu_r(g^{(1)}_{{\tau}_\ell})\langle \hw_k, \hw_\ell\rangle^r,
\end{align*}
which can  equivalently  be written as $\frac{1}{2}\sum_{r=0}^{\infty}T_rT_r^\top$.
The second part of the statement follows from the fact that each individual matrix $T_rT_r^\top$ is a positive
semidefinite Grammian matrix.
\end{proof}
\begin{lemma}
\label{lem:lower_bound_expectation_technical}
Let $E$ be defined as in \eqref{eq:definition_expectation}
and assume that the approximated weights satisfy $\|\hw_k \|_2 = 1$ and \ref{enum:correlation}
for some universal constant $c_2$.
Furthermore, assume the activation function adheres to \ref{enum:nonpolynomoial}. Then,
we have
\begin{align}
	\label{eq:lower_bound_expectation_technical}
	&\lambda_{m}(E) \geq  \omega -C(m-1)\left(\frac{\log m}{D}\right)^{2},
\end{align}
where $\omega$ and $C$ are constants depending only on $g$ and $\tau_{\infty}$. Specifically, we have
\begin{align*}
\omega &= \frac{1}{2}\min_{\tau \in [-\bar{\tau}_{\infty},\bar{\tau}_{\infty}]}\sum_{r \geq 4} \left(\mu_r(g^{(1)}(\cdot + \tau))\right)^2,\\
C &=\frac{1}{2} c_2^2 \max_{\tau,\widetilde \tau \in [-\tau_{\infty}, \tau_{\infty}]}\sum_{r\geq 4}\snorm{\mu_{r}(g^{(1)}_{\tau})\mu_r(g^{(1)}_{\widetilde\tau})}.
\end{align*}
\end{lemma}
\begin{proof}[Proof of Lemma \ref{lem:lower_bound_expectation_technical}]
	To simplify the notation, we introduce the shorthand $\mu_{r,k} := \mu_r(g^{(1)}_{{\tau}_k})$.
	By Lemma \ref{lem:expectation_full} we have $E \succcurlyeq \frac{1}{2} \sum_{r \geq 4} T_r T_r^\top$, so we concentrate on the
	expression on the right hand side. As $\norm{{\hw}_k}_2 = 1$ for all $k \in [m]$, we first note that we can rewrite $\frac{1}{2} \sum_{r \geq 4} T_r T_r^\top$
	as $ \frac{1}{2} \sum_{r \geq 4} T_r T_r^\top =  D_{4} + O_{4}$,
	where the matrix $D_{4}$ is given by
	\begin{align*}
		D_{4} &:= \frac{1}{2}\operatorname{Diag}\Big(\sum_{r \geq 4}\mu_{r,1}^2,\ldots, \sum_{r \geq 4}\mu_{r,m}^2\Big)
	\end{align*}
	and the remainder $O_{4}$ equals $ \frac{1}{2}\sum_{r \in 4} T_r T_r^\top$ with its diagonal
	set to $0$. To show \eqref{eq:lower_bound_expectation_technical}, we compute a lower eigenvalue bound
	for $	D_{4} $ and an upper eigenvalue bound for $O_4$ independently, and then complete the argument with
	Weyl's eigenvalue perturbation bound \cite{weylAsymptotischeVerteilungsgesetzEigenwerte1912}. The smallest
	eigenvalue of $	D_{4} $ can be read from the diagonal and is given by
	\begin{align*}
	\lambda_{\min}(D_{4}) = \frac{1}{2} \min_{k \in [m]}\sum_{r \geq 4}\mu_{r,k}^2 \geq \omega > 0.
	\end{align*}
 	For the spectral norm of $O_{4}$ we use $L_{1}/L_{\infty}$-Cauchy-Schwarz inequalities and $\norm{\hw_k}_2 = 1$ for all $k \in [m]$. Specifically, for any unit norm vector $u$ we have
	\begin{align*}
		u^\top O_{4}u &= \frac{1}{2}\sum_{k=1}^{m}\sum_{\ell\neq k}u_{k}u_{\ell}\sum_{r \geq 4}\mu_{r,k}\mu_{r,\ell}\langle \hw_k, \hw_\ell\rangle^r\\
		&\leq \frac{1}{2} \sum_{k=1}^{m}\sum_{\ell \neq k}\snorm{u_{k}}\snorm{u_{\ell}}\sum_{r \geq 4}\snorm{\mu_{r,k}\mu_{r,\ell}}\snorm{\langle \hw_k, \hw_\ell\rangle}^r.
	\end{align*}
	By dragging out the maximum of the sums over Hermitian coefficients, we further bound
	\begin{align*}
		u^\top O_{4}u \leq \Big( \frac{1}{2}\max_{\tau,\widetilde \tau \in [-\tau_{\infty}, \tau_{\infty}]}\sum_{r\geq 4}\snorm{\mu_{r}(g^{(1)}_{\tau})\mu_r(g^{(1)}_{\widetilde\tau})}\Big)\sum_{k=1}^{m}\sum_{\ell\neq k}\snorm{u_{k}}\snorm{u_{\ell}}\snorm{\langle \hw_k, \hw_\ell\rangle}^{4}.
	\end{align*}
	The trailing factor is, for all unit norm $u$, bounded by the spectral norm of the matrix
	\begin{equation}
	( \widehat O_{4})_{ij} := \begin{cases}
	0,& \textrm{ if } i = j,\\
	\snorm{\langle  \hw_i, \hw_j\rangle}^{4},& \textrm{ else }.
\end{cases}
\end{equation}
Therefore we have $u^\top O_{4}u \leq C_{g,\tau_{\infty}}\|\widehat O_{4}\|$ for all unit norm $u$, and with the constant $C_{g,\tau_{\infty}}$ given as 
$$
C_{g,\tau_{\infty}} =\frac{1}{2} \max_{\tau,\widetilde \tau \in [-\tau_{\infty}, \tau_{\infty}]}\sum_{r\geq 4}\snorm{\mu_{r}(g^{(1)}_{\tau})\mu_r(g^{(1)}_{\widetilde\tau})},
$$
and only dependent on $g$ and the shift bound $\tau_{\infty}$. By Gershgorin's circle theorem we further have
\begin{align*}
	\|\widehat O_{4}\| \leq \max_{k\in[m]} \sum_{\ell\neq k}^m |(\widehat O_{4})_{k\ell}| \leq (m-1)\left(\frac{c_2 \log m}{D}\right)^{2},
\end{align*}
where we used the fact that $\hw_1,\ldots,\hw_m$ satisfy \ref{enum:correlation}. 
\end{proof}
\subsection{Controlling the perturbation from the idealized GD iteration}
\label{subsec:controlling_perturbation}
This section is concerned with the divergence between the two gradient descent iterations in \eqref{eq:gd_iteration_proof} and \eqref{eq:idealized_gd_iteration}. We start with a number of auxiliary results that control certain series involving the Hermite coefficients of the activation and its derivatives. These technical statements are needed to control the perturbation in the GD iteration that is caused by the weight approximation. The bounds enable Lemma \ref{lemma:delta_grads} which provides an upper bound for the difference between the gradients $\nabla J(\htau), \nabla J_*(\htau)$, defined in \eqref{eq:gd_loss_proof}, \eqref{eq:def_Jbar_proof}, respectively, w.r.t. the accuracy of the estimated weights and shift initializations.  
\subsubsection{Controlling perturbation from weigths}
The first part of this section is concerned with estimating a series that contains elements
\begin{align}\label{eq:def_hermite_series_coefficients}
	S_{r,\ell} &:= \sum^m_{k=1} \mu_{r}(g_{\tau_k}) \mu_{r}(g^{(1)}_{\htau_\ell})\left( \inner{\hw_k, \hw_\ell}^r - \inner{w_k, \hw_\ell}^r \right),
\end{align}
where $ \mu_{r}(g_{\tau_k}), \mu_{r}(g^{(1)}_{\htau_\ell})$ correspond to the $k$-th and $\ell$-th Hermite coefficient of the function $g_{\tau_k}(\cdot) = g(\cdot + \tau_k)$, $g_{\htau_\ell}'(\cdot) = g^{(1)}(\cdot + \htau_\ell)$, respectively. These coefficients are assumed to be uniformy bounded for all $r\geq 0$ which is a consequence of \ref{enum:nonpolynomoial} and Lemma \ref{lemma:bounded_in_gaussianL2}. 
The following results pave the way for perturbation bound w.r.t. estimated weights and we use the following shorthands to keep the expressions more compact: 
\begin{align}
	\Delta_{W, F} &= \| \wW - W\|_F \label{eq:delta_fro_shorthand},\\
	\Delta_{W, O} &= \sum^m_{k\neq k'} \left|\inner{\hw_{k} - w_k, \hw_{k'} - w_k'}\right|. \label{eq:delta_offdiag_shorthand}
\end{align}
\begin{lemma}\label{lemma:aux_for_delta_grads} Consider weights and approximated weights $(w_k)_{k\in [m]},(\hw_k)_{k\in [m]}$ of unit norm as before that both fulfill \ref{enum:correlation} and (i) in Lemma \ref{lemma:incoherence_apx_weigths}, as well shifts $(\tau_k)_{k\in [m]},(\htau_k)_{k\in [m]}$ within $[-\tau_\infty, \tau_\infty]$ for some $\tau_\infty < \infty$. Let $S_{r, \ell}$ be defined as in \eqref{eq:def_hermite_series_coefficients} and assume that $g$ fulfills the assumption \ref{enum:activation} - \ref{enum:nonpolynomoial}. Then, there exists a constant $C>0$ such that, for $m \geq D$,
	\begin{align*} 
		\sum^m_{\ell = 1}S_{r,\ell}^2 &\leq C r^2 \max_{k, \ell \in [m]} \mu_{r}(g^{(1)}_{\htau_\ell})^2\mu_r(g_{\tau_k})^2 {\left( 1+ m \left(\frac{\log m}{D}\right)^{r / 2}\right)}\\
		&\cdot \left[ \Delta_{W, F}^2 +  \left(\frac{\log m }{D}\right)^{(r-1)/2} \Delta_{W, O} \right].
	\end{align*}
	Furthermore, for any fixed $R\geq 2$ we have
	\begin{align*}
		\sum_{r = 2}^R 2^r  \sum^m_{\ell=1} S_{r,\ell}^2 \leq \frac{ Cm \log m}{D}\left(  \Delta_{W, F}^2 +  \left(\frac{\log m}{D}\right)^{1/2} \Delta_{W, O }\right),
	\end{align*}
	where the constant $C>0$ additionally depends on $R$.
\end{lemma}
\begin{proof}[Proof of Lemma \ref{lemma:aux_for_delta_grads}]
	Throughout this proof we use the convention that, for any vector, we have $v^{\otimes 0} = 1$, $1 \otimes 1 = 1$ and $v \otimes 1 = 1 \otimes v = v$, which will be relevant for the case $r=1$. 
	We start with a chain of equalities that uses elementary properties of the Frobenius inner product: 
\begin{align*}
	\sum^m_{\ell = 1}S_{r,\ell}^2 =&\sum^m_{\ell = 1}\left[\sum^m_{k=1}  \mu_{r}(g_{\tau_k}) \mu_{r}(g^{(1)}_{\htau_\ell})\left( \inner{\hw_k, \hw_\ell}^r - \inner{w_k, \hw_\ell}^r \right) \right]^2 \\
	=&  \sum^m_{\ell = 1}\left[\sum^m_{k=1}  \mu_{r}(g_{\tau_k}) \mu_{r}(g^{(1)}_{\htau_\ell})\inner{\hw_k - w_k, \hw_\ell} \left( \sum^{r}_{i=1}\inner{\hw_k, \hw_\ell}^{r-i}  \inner{w_k, \hw_\ell}^{i-1} \right) \right]^2 \\
	=&\sum^m_{\ell = 1}\left[\sum^m_{k=1}  \mu_{r}(g_{\tau_k}) \mu_{r}(g^{(1)}_{\htau_\ell})\inner{\hw_k - w_k, \hw_\ell}  \left\langle \sum^{r}_{i=1} \hw_k^{\otimes (r-i)} \otimes w_k^{\otimes (i - 1)}, \hw_\ell^{\otimes r-1}\right\rangle    \right]^2\\
	=&\sum^m_{\ell = 1}\left[\sum^m_{k=1}  \mu_{r}(g_{\tau_k}) \mu_{r}(g^{(1)}_{\htau_\ell})  \left\langle (\hw_k - w_k) \otimes \sum^{r}_{i=1} \left( \hw_k^{\otimes (r-i)} \otimes w_k^{\otimes (i - 1)} \right), \hw_\ell^{\otimes r}\right\rangle    \right]^2\\
	=&\sum^m_{\ell = 1}\left[ \mu_{r}(g^{(1)}_{\htau_\ell})  \left\langle \sum^m_{k=1}  \mu_{r}(g_{\tau_k}) (\hw_k - w_k) \otimes \sum^{r}_{i=1} \left( \hw_k^{\otimes (r-i)} \otimes w_k^{\otimes (i - 1)} \right), \hw_\ell^{\otimes r}\right\rangle    \right]^2\\
	=&\sum^m_{\ell = 1} \mu_{r}(g^{(1)}_{\htau_\ell})^2   \left\langle \sum^m_{k=1}  \mu_{r}(g_{\tau_k}) (\hw_k - w_k) \otimes \sum^{r}_{i=1} \left( \hw_k^{\otimes (r-i)} \otimes w_k^{\otimes (i - 1)} \right), \hw_\ell^{\otimes r}\right\rangle^2.
\end{align*}
At this stage, we separate the coefficients depending on $\ell$ such that 
\begin{align*}
	&\sum^m_{\ell = 1}\left[\sum^m_{k=1}  \mu_{r}(g_{\tau_k}) \mu_{r}(g^{(1)}_{\htau_\ell})\left( \inner{\hw_k, \hw_\ell}^r - \inner{w_k, \hw_\ell}^r \right) \right]^2 \\
	\leq& \max_{\ell \in [m]} \mu_{r}(g^{(1)}_{\htau_\ell})^2 \sum^m_{\ell = 1}    \left\langle \sum^m_{k=1}  \mu_{r}(g_{\tau_k}) (\hw_k - w_k) \otimes \sum^{r}_{i=1} \left(\hw_k^{\otimes (r-i)} \otimes w_k^{\otimes (i - 1)}\right), \hw_\ell^{\otimes r}\right\rangle^2.
\end{align*}
Now, note that the set of tensors $(\hw_\ell^{\otimes r})_{\ell \in [m]}$ forms a frame whose upper frame constant is bounded by the upper spectrum of the Grammian $(\widehat{G}_r)_{ij} = \inner{\hw_i, \hw_j}^r$, see also Lemma \ref{lemma:aux_grammian_frame_bound}. Due to Lemma \ref{lemma:grammians_inc_greshgorin} which relies on Gershgorin's circle theorem we know there exists an absolute constant $C>0$ such that for $D$ sufficiently large the operator norm of $\widehat{G}_r$ obeys
\begin{align*}
	\|\widehat{G}_r\|\leq C {\left( 1+  m \left(\frac{\log m}{D}\right)^{r / 2}\right)}. 
\end{align*}
As a consequence, we get
\begin{align}
	&\max_{\ell \in [m]} \mu_{r}(g^{(1)}_{\htau_\ell})^2 \sum^m_{\ell = 1}    \left\langle \sum^m_{k=1}  \mu_{r}(g_{\tau_k}) (\hw_k - w_k) \otimes \sum^{r}_{i=1} \left(\hw_k^{\otimes (r-i)} \otimes w_k^{\otimes (i - 1)}\right), \hw_\ell^{\otimes r}\right\rangle^2\notag \\
	\leq& \,\max_{\ell \in [m]} \mu_{r}(g^{(1)}_{\htau_\ell})^2 \|\widehat{G}_r\| \left\| \sum^m_{k=1}  \mu_{r}(g_{\tau_k}) (\hw_k - w_k) \otimes \sum^{r}_{i=1} \left( \hw_k^{\otimes (r-i)} \otimes w_k^{\otimes (i - 1)} \right) \right\|_F^2\notag\\
	\leq& \, C \max_{\ell \in [m]} \mu_{r}(g^{(1)}_{\htau_\ell})^2 {\left( 1+  m \left(\frac{\log m}{D}\right)^{r / 2}\right)} \left\| \sum^m_{k=1}  \mu_{r}(g_{\tau_k}) (\hw_k - w_k) \otimes \sum^{r}_{i=1} \left( \hw_k^{\otimes (r-i)} \otimes w_k^{\otimes (i - 1)} \right) \right\|_F^2.\label{eq:before_bound_delta_Tkr}
\end{align}
Denote now $\Delta_{k,r} := \mu_{r}(g_{\tau_k}) (\hw_k - w_k)$ and $T_{k,r} := \sum^{r}_{i=1} \left( \hw_k^{\otimes (r-i)} \otimes w_k^{\otimes (i - 1)} \right)$, then 
\begin{align}
	&\left\| \sum^m_{k=1}  \mu_{r}(g_{\tau_k}) (\hw_k - w_k) \otimes \sum^{r}_{i=1} \left( \hw_k^{\otimes (r-i)} \otimes w_k^{\otimes (i - 1)} \right) \right\|_F^2\notag \\
	=& \sum^m_{k,k'=1}\inner{\Delta_{k,r}\otimes T_{k,r} , \Delta_{k',r} \otimes T_{k',r}} =  \sum^m_{k,k'=1}\inner{\Delta_{k,r}, \Delta_{k',r}} \inner{T_{k,r} ,  T_{k',r}} \notag\\
	= &\, \sum^m_{k=1}\| \Delta_{k,r}\|_2^2 \|T_{k,r}\|_F^2 + \sum^m_{k\neq k'}\inner{\Delta_{k,r}, \Delta_{k',r}} \inner{T_{k,r} ,  T_{k',r}}\label{eq:bound_delta_Tkr}.
\end{align}
Using $\|w_k\|_2 = \|\hw_k\|_2 = 1$ we get
\begin{align*}
	\|T_{k,r}\|_F \leq \sum^r_{i=1} \| \hw_k^{\otimes (r-i)} \otimes w_k^{\otimes (i - 1)} \|_F \leq \sum^r_{i=1} \|\hw_k \|_2^{r-i} \|w_k \|_2^{i-1} = r,
\end{align*}
such that the left part of \eqref{eq:bound_delta_Tkr} can be estimated by 
\begin{equation}
    \begin{split}
        \label{eq:bound_delta_Tkr_left}
	\sum^m_{k=1}\| \Delta_{k,r}\|_2^2 \|T_{k,r}\|_F^2 &\leq r^2 \max_{k \in [m]} \mu_r(g_{\tau_k})^2 \sum^m_{k=1} \| \hw_k - w_k \|_2^2 \\
	& = r^2 \max_{k \in [m]} \mu_r(g_{\tau_k})^2 \| \wW - W\|_F^2. 
    \end{split}
\end{equation}
To bound the right part of \eqref{eq:bound_delta_Tkr} first note that, for $k\neq k'$,
\begin{align*}
	\inner{T_{k,r} ,  T_{k',r}}  &= \sum^r_{i,i'=1} \left\langle \hw_k^{\otimes (r-i)} \otimes w_k^{\otimes(i-1)}, \hw_{k'}^{\otimes (r-i')} \otimes w_{k'}^{\otimes(i'-1)} \right\rangle\\
	&\leq C \sum^r_{i,i'=1} \left(\frac{\log m }{D}\right)^{(r-1)/2} = C r^2 \left(\frac{\log m }{D}\right)^{(r-1)/2} ,
\end{align*} 
for some absolute constant $C$, which follows from the pairwise incoherence \ref{enum:correlation} as well as point (i) of Lemma \ref{lemma:incoherence_apx_weigths}. Therefore, the right part of \eqref{eq:bound_delta_Tkr}
is bounded by 
\begin{equation}
    \begin{split}
    \label{eq:bound_delta_Tkr_right}
	\sum^m_{k\neq k'}\inner{\Delta_{k,r}, \Delta_{k',r}} \inner{T_{k,r} ,  T_{k',r}} &\leq C r^2 \left(\frac{\log m }{D}\right)^{(r-1)/2} \sum^m_{k\neq k'} \left|\inner{\Delta_{k,r}, \Delta_{k',r}}\right| \\ &\leq C r^2 \left(\frac{\log m }{D}\right)^{(r-1)/2} \max_{k \in [m]} \mu_r(g_{\tau_k})^2  \sum^m_{k\neq k'} \left|\inner{\hw_{k} - w_k, \hw_{k'} - w_k'}\right|.
    \end{split}
\end{equation}
Plugging in \eqref{eq:bound_delta_Tkr_left} and \eqref{eq:bound_delta_Tkr_right} into \eqref{eq:bound_delta_Tkr} yields 
\begin{align}
	&\left\| \sum^m_{k=1}  \mu_{r}(g_{\tau_k}) (\hw_k - w_k) \otimes \sum^{r}_{i=1} \left( \hw_k^{\otimes (r-i)} \otimes w_k^{\otimes (i - 1)} \right) \right\|_F^2 \\
	\leq &\, C r^2 \max_{k \in [m]} \mu_r(g_{\tau_k})^2 \left[  \| \wW - W\|_F^2 +  \left(\frac{\log m }{D}\right)^{(r-1)/2}  \sum^m_{k\neq k'} \left|\inner{\hw_{k} - w_k, \hw_{k'} - w_k'}\right| \right] \label{eq:delta_grads_aux_final_frobound}.
\end{align}
Combining this with \eqref{eq:before_bound_delta_Tkr} yields the desired first statement
\begin{align*} 
	\sum^m_{\ell = 1}S_{r,\ell}^2 &\leq C r^2 \max_{k, \ell \in [m]} \mu_{r}(g^{(1)}_{\htau_\ell})^2\mu_r(g_{\tau_k})^2 {\left( 1+ m \left(\frac{\log m}{D}\right)^{r / 2}\right)}\\
	&\cdot \left[ \Delta_{W, F}^2 +  \left(\frac{\log m }{D}\right)^{(r-1)/2} \Delta_{W, O} \right].
\end{align*}
For the second statement, note that  $\max_{k \in [m]} \mu_r(g_{\tau_k})^2$ is bounded due to \ref{enum:nonpolynomoial} and $\max_{\ell \in [m]} \mu_{r}(g^{(1)}_{\htau_\ell})^2$ is bounded according to Lemma \ref{lemma:bounded_in_gaussianL2}. Hence it follows that  
\begin{align*}
	\sum_{r=2}^R 2^r  \sum^m_{\ell=1} S_{r,\ell}^2 &\leq  \sum_{r=2}^R 2^r  C r^2   {\left( 1+  m \left(\frac{\log m}{D}\right)^{r / 2}\right)}\left[  \Delta_{W, F}^2 +  \left(\frac{\log m }{D}\right)^{(r-1)/2}  \Delta_{W, O }\right]\\ 
	&\leq    {\left( 1+  m \left(\frac{\log m}{D}\right)\right)}\left(  \Delta_{W, F}^2 +   \left(\frac{\log m}{D}\right)^{1/2}  \Delta_{W, O }\right) \sum_{r=1}^R 2^r  C r^2.
\end{align*}
The second statement follows from the upper bound above by adjusting the constant $C$ due to $\sum_{r=1}^R 2^r  C r^2 < \infty$ for fixed $R$ and using $(m \log m) / D  > 1$. 
\end{proof}
\begin{lemma}\label{lemma:sum_1_hermites}
	Consider weights and approximated weights $(w_k)_{k\in [m]},(\hw_k)_{k\in [m]}$ of unit norm as before that both fulfill \ref{enum:correlation} and (i) in Lemma \ref{lemma:incoherence_apx_weigths}, as well shifts $(\tau_k)_{k\in [m]},(\htau_k)_{k\in [m]}$ within $[-\tau_\infty, \tau_\infty]$ for some $\tau_\infty < \infty$. Let $S_{r, \ell}$ be defined as in \eqref{eq:def_hermite_series_coefficients} and assume that $g$ fulfills the assumption \ref{enum:activation} - \ref{enum:nonpolynomoial}. Then, there exists a constant $C>0$ such that for $m \geq D$
	\begin{align*}
		\sum^m_{\ell=1} S_{1,\ell}^2 \leq C m \left(\frac{\log m }{D}\right)^{1/2} \left\| \sum^m_{k=1} w_k - \hw_k\right\|_2^2.
	\end{align*}
\end{lemma}
\begin{proof}
	According to the proof of Lemma \ref{lemma:aux_for_delta_grads}, in particular \eqref{eq:before_bound_delta_Tkr}, we can bound 
	\begin{align*}
		\sum^m_{\ell=1} S_{1,\ell}^2 \leq C m \left(\frac{\log m }{D}\right)^{1/2} \left\| \sum^m_{k=1} \mu_1(g_{\tau_k})(w_k - \hw_k)\right\|_2^2,
	\end{align*}
	for some constant $C>0$. Since $\mu_1(g_{\tau_k})$ is bounded for all $k\in [m]$, what remains is to show that the Hermite coefficients do not change signs. Note that the first Hermite polynomial is given by $h_1(u) = u$. According to the definition of the Hermite coefficients we have 
	\begin{align*}
		\mu_1(g_{\tau_k}) &= \int_\R u g(u + \tau_k) e^{-u^2 / 2} du = \left[ - g(u + \tau_k) e^{-u^2 / 2} \right]_{-\infty}^{\infty} + \int_\R  g^{(1)}(u + \tau_k) e^{-u^2 / 2} du \\
		&= \int_\R  g^{(1)}(u + \tau_k) e^{-u^2 / 2} du.
	\end{align*}
	Now note that $g^{(1)}(u + \tau_k)$ will always have the same sign since $g^{(2)}$ is monotonic due to \ref{enum:activation}. Therefore $\mu_1(g_{\tau_1}), \dots, \mu_1(g_{\tau_m})$ must all be either positive or negative, from which the proof follows directly. 
\end{proof}
\begin{lemma}\label{lemma:bound_hermite_tail}
	Assume that $g$ fulfills the assumption \ref{enum:activation}-\ref{enum:nonpolynomoial} and that the shifts $(\tau_k)_{k\in [m]},(\htau_k)_{k\in [m]}$ are within $[-\tau_\infty, \tau_\infty]$. Then, for $R\geq 4$, we have  
\begin{align}\label{eq:hermite_tail_series}
	\sum_{r \geq R} r |\max_{k,\ell \in [m]} \mu_r(g_{\tau_k})\mu_r(g^{(1)}_{\htau_\ell})| < \infty.
\end{align}
\end{lemma} 
\begin{proof}
	By applying Lemma \ref{lemma:hermite_coefficients_derivatives} (whose condition is met due to \ref{enum:activation} - \ref{enum:nonpolynomoial}), we immediately get that, for all $r\geq R$, 
	\begin{align*}
		\mu_r(g_{\tau_k}) \mu_r(g^{(1)}_{\htau_\ell}) &= \left(3! \binom{r}{r-3}\right)^{-1/2}\mu_{r-3}(g^{(3)}_{\tau_k}) \cdot \left(2! \binom{r}{r-2}\right)^{-1/2}\mu_{r-2}(g^{(3)}_{\htau_\ell}) \\
		&= \left( (r-2)(r-1)^2 r^2  \right)^{-1/2} \mu_{r-3}(g^{(3)}_{\tau_k}) \mu_{r-2}(g^{(3)}_{\htau_\ell}).
	\end{align*}
	Plugging this into \eqref{eq:hermite_tail_series} yields 
	\begin{align*}
		\sum_{r \geq R} r |\max_{k,\ell \in [m]} \mu_r(g_{\tau_k})\mu_r(g^{(1)}_{\htau_\ell})| &\leq \sum_{r \geq R}\frac{1}{\sqrt{r-2}(r-1)}  \max_{k, \ell \in [m]} \mu_{r-3}(g^{(3)}_{\tau_k})\mu_{r-2}(g^{(3)}_{\htau_\ell}) \\
		&\leq  \sum_{r \geq R-3} \max_{k \in [m]} \frac{1}{r^{3/2}}\mu_{r}(g^{(3)}_{\tau_k})^2 +  \sum_{r \geq R-2} \max_{\ell \in [m]} \frac{1}{r^{3/2}}\mu_{r}(g^{(3)}_{\htau_\ell})^2,
	\end{align*}
	where the second line follows by applying Cauchy-Schwarz. Using assumption \ref{enum:activation}, according to Lemma \ref{lemma:bounded_in_gaussianL2}, then gives $\max_{\tau \in [-\tau_\infty, \tau_\infty]} \mu_r(g^{(3)}_{\tau})^2 \leq C$ for all $r\geq 0$ and some constant $C>0$. Therefore we can conclude with 
	\begin{align*}
		\sum_{r \geq R} r |\max_{k,\ell \in [m]} \mu_r(g_{\tau_k})\mu_r(g^{(1)}_{\htau_\ell})| \leq 2C \sum_{r \geq 1}  \frac{1}{r^{3/2}} \leq 6C < \infty.
	\end{align*}
\end{proof}

\begin{lemma}\label{lemma:delta_grads_tail_bound}
Consider weights and approximated weights $(w_k)_{k\in [m]},(\hw_k)_{k\in [m]}$ of unit norm as before that both fulfill \ref{enum:correlation} and (i) in Lemma \ref{lemma:incoherence_apx_weigths}, as well as shifts $(\tau_k)_{k\in [m]},(\htau_k)_{k\in [m]}$ within $[-\tau_\infty, \tau_\infty]$ for some $\tau_\infty < \infty$.
Let $S_{r, \ell}$ be defined as in \eqref{eq:def_hermite_series_coefficients}, and assume that $g$ fulfills the assumption \ref{enum:activation}. Then, there exists a constant $C>0$ such that for $R\geq 9$ we have 
\begin{align*}
	\sum^m_{\ell = 1} \left( \sum_{r \geq R}S_{r,\ell} \right)^2 \leq C \sqrt{m}\Delta_{W,F}^2.
\end{align*}
\end{lemma}
\begin{proof}
	We start by applying the Cauchy product to the squared series 
	\begin{align*}
		\sum^m_{\ell = 1} \left( \sum_{r \geq R}S_{r,\ell} \right)^2 &= \sum^m_{\ell=1}\left(\sum_{r \geq 0}\sum^r_{s=0} S_{r+R-s, \ell}S_{s+R, \ell}\right) \\ 
	&= \sum_{r \geq 0}\sum^r_{s=0} \sum^m_{\ell=1} S_{r+R-s, \ell}S_{s+R, \ell}\\
	&\leq \sqrt{m}\sum_{r \geq 0}\sum^r_{s=0} \left(\sum^m_{\ell=1} S_{r+R-s, \ell}^2 S_{s+R, \ell}^2\right)^{1/2}.
	\end{align*}
	The inner sum is now controlled by a sequence of inequalities similar to Lemma \ref{lemma:aux_for_delta_grads}. Again we denote $\Delta_{k,r} := \mu_{r}(g_{\tau_k}) (\hw_k - w_k)$ and $T_{k,r} := \sum^{r}_{i=1} \left( \hw_k^{\otimes (r-i)} \otimes w_k^{\otimes (i - 1)} \right)$, then by applying the same chain of inequality as in the beginning of the proof of Lemma \ref{lemma:aux_for_delta_grads} we receive 
	\begin{align*}
		\sum^m_{\ell=1} S_{r+R-s, \ell}^2 S_{s+R, \ell}^2 &= \sum^m_{\ell = 1} \mu_{r+R-s}(g^{(1)}_{\htau_\ell})^2 \mu_{s+R}(g^{(1)}_{\htau_\ell})^2 \\
		&\cdot \left\langle \sum^m_{k=1} \Delta_{k, r+R-s} \otimes T_{k, r+R-s}, \hw_\ell^{\otimes r+R-s} \right\rangle^2 \left\langle \sum^m_{k=1} \Delta_{k, s+R} \otimes T_{k, s+R}, \hw_\ell^{\otimes s+R}  \right\rangle^2 \\ 
		&= \sum^m_{\ell = 1} \mu_{r+R-s}(g^{(1)}_{\htau_\ell})^2 \mu_{s+R}(g^{(1)}_{\htau_\ell})^2 \\
		&\cdot \left\langle \left( \sum^m_{k=1} \Delta_{k, r+R-s} \otimes T_{k, r+R-s}\right) \otimes \left(\sum^m_{k=1}  \Delta_{k, s+R} \otimes T_{k, s+R}\right), \hw_\ell^{\otimes r+2R}  \right\rangle^2.
	\end{align*}
	As before we now invoke the frame like condition described in Lemma \ref{lemma:aux_grammian_frame_bound} to attain a bound depending on the upper spectrum of the Grammian $(\widehat{G}_{r+2R})_{ij} = \inner{\hw_i, \hw_j}^{r+2R}$. More precisely, by using the shorthand 
	\begin{align} 
	\mu'_{r,s} := \max_{\ell \in [m]} \mu_{r+R-s}(g^{(1)}_{\htau_\ell})^2 \mu_{s+R}(g^{(1)}_{\htau_\ell})^2
	\end{align} we then have
	\begin{align}
		\sum^m_{\ell=1} S_{r+R-s, \ell}^2 S_{s+R, \ell}^2 &\leq \mu'_{r,s} \| \widehat{G}_{r+2R} \|
		\left\| \left( \sum^m_{k=1} \Delta_{k, r+R-s} \otimes T_{k, r+R-s}\right) \otimes \left(\sum^m_{k=1}  \Delta_{k, s+R} \otimes T_{k, s+R}\right) \right\|_F^2 \\
		&\leq \mu'_{r,s} \| \widehat{G}_{r+2R} \|
		\left\| \sum^m_{k=1} \Delta_{k, r+R-s} \otimes T_{k, r+R-s}\right\|_F^2 \left\| \sum^m_{k=1}  \Delta_{k, s+R} \otimes T_{k, s+R} \right\|_F^2. \label{eq:cauchy_product_jumpoff}
	\end{align}
	The two Frobenius norms can now be estimated as in Lemma \ref{lemma:aux_for_delta_grads}, more precisely \eqref{eq:delta_grads_aux_final_frobound}, where we also use the shorthands $\Delta_{W,F}, \Delta_{W, O}$ defined in \eqref{eq:delta_fro_shorthand} - \eqref{eq:delta_offdiag_shorthand} as well as 
	\begin{align*}
		\mu_{r,s}:=  \max_{k \in [m]}\mu_{r+R-s}(g_{\tau_k})^2 \max_{k \in [m]}\mu_{s+R}(g_{\tau_k})^2.
	\end{align*}	
	This gives for some absolute constant $C>0$ 
	\begin{align*}
	\mu'_{r,s}  \| \widehat{G}_{r+2R} \|
		&\left\| \sum^m_{k=1} \Delta_{k, r+R-s} \otimes T_{k, r+R-s}\right\|_F^2 \left\| \sum^m_{k=1}  \Delta_{k, s+R} \otimes T_{k, s+R} \right\|_F^2 \\ 
		&\leq C \mu'_{r,s} \mu_{r,s} \| \widehat{G}_{r+2R} \| (r+R-s)^2 (s+R)^2 \\ &\cdot \left(\Delta_{W,F}^2 + \left(\frac{\log m}{D}\right)^{(r+R-s-1)/2} \Delta_{W,O} \right) \left(\Delta_{W,F}^2 + \left(\frac{\log m}{D}\right)^{(s+R-1)/2} \Delta_{W,O} \right) \\
		&\leq C \mu'_{r,s} \mu_{r,s} \| \widehat{G}_{r+2R} \| (r+R-s)^2 (s+R)^2 \left( \Delta_{W, F}^2 + \left(\frac{\log m}{D}\right)^{(R-1)/2} \Delta_{W, O} \right)^2. 
	\end{align*}
	Next we identify the dominant factors and simplify. Due to Lemma \ref{lemma:grammians_inc_greshgorin} we have for some constant $C>0$ that 
	\begin{align*}
		\| \widehat{G}_{r+2R} \| \leq C \left( 1 + m \left( \frac{\log m}{D}\right)^{(r+2R)/2} \right) \leq C \left( 1 + m \left( \frac{\log m}{D}\right)^{9} \right).
	\end{align*}
	where the last stop follows since $R\geq 9$ and due to $m (\log m)^2 \leq D^2$ this can be further simplified to $\| \widehat{G}_{r+2R} \| \leq C $. Similarly, we have 
	\begin{align*}
		\left(\frac{\log m}{D}\right)^{(R-1)/2} \Delta_{W, O} &=  \left(\frac{\log m}{D}\right)^{(R-1)/2} \sum^m_{k\neq k'} \left|\inner{\hw_{k} - w_k, \hw_{k'} - w_k'}\right| \\
		&\leq \left(\frac{\log m}{D}\right)^{4}  m^2 \delta_{\max}^2 \leq \delta_{\max}^2 \leq \Delta_{W,F}^2
	\end{align*}
	and therefore we get 
	\begin{align*}
		\| \widehat{G}_{r+2R} \| \left( \Delta_{W, F}^2 + \left(\frac{\log m}{D}\right)^{(R-1)/2} \Delta_{W, O} \right)^2 \leq C \Delta_{W, F}^4 
	\end{align*}
	for some absolute constant $C>0$. Plugging these into \eqref{eq:cauchy_product_jumpoff} results in 
	\begin{align*}
		\sum^m_{\ell=1} S_{r+R-s, \ell}^2 S_{s+R, \ell}^2 \leq C \mu'_{r,s} \mu_{r,s} (r+R-s)^2 (s+R)^2   \Delta_{W, F}^4.
	\end{align*}
	Hence, we have 
	\begin{align*}
		\sum^m_{\ell = 1} \left( \sum_{r \geq R}S_{r,\ell} \right)^2 &\leq \sqrt{m}\sum_{r \geq 0}\sum^r_{s=0} \left(\sum^m_{\ell=1} S_{r+R-s, \ell}^2 S_{s+R, \ell}^2\right)^{1/2}\\
		&\leq C \Delta_{W, F}^2 \sqrt{m}\sum_{r \geq 0}\sum^r_{s=0} \sqrt{|\mu'_{r,s} \mu_{r,s}|} (r+R-s) (s+R)\\
		&\leq C \Delta_{W, F}^2 \sqrt{m}\left(\sum_{r \geq R} r |\max_{k,\ell} \mu_r(g_{\tau_k})\mu_r(g^{(1)}_{\htau_\ell})|\right)^2.
	\end{align*}
	The result then follows by applying Lemma \ref{lemma:bound_hermite_tail} onto the series in the last line followed by a unification of the constants. 
\end{proof}

Now, we are finally able to formalize the key lemma of this section. Recall that 
\begin{align*}
	\Delta_{W, F} &:= \| \wW - W\|_F \\
	\Delta_{W, O} &:= \sum^m_{k\neq k'} \left|\inner{\hw_{k} - w_k, \hw_{k'} - w_{k'}}\right|,\\
	\Delta_{W, S} &:= \left\| \sum^m_{k=1} w_k - \hw_k \right\|_2.
\end{align*}

\begin{lemma}\label{lemma:delta_grads} 
Consider a shallow neural network $f$ with unit norm weights described by ${W}$, shifts ${\tau}_1,\dots, {\tau}_m \in [-\tau_{\infty}, \tau_{\infty}]$ stored in $\tau$ and an activation function $g$ that adheres to \ref{enum:activation} with $D\leq m$. Furthermore, consider $J, J_*$ given by \eqref{eq:gd_loss_proof}, \eqref{eq:def_Jbar_proof} constructed with ${N_{\text{train}}}\geq m$  network evaluations $y_1,\dots, y_{N_{\text{train}}}$ of ${f}$ where $y_i = {f}(X_i)$ and  $X_1,\dots, X_N \sim \CN(0,\Id_D)$. Denote by $\hat  f$ an approximation to ${f}$ constructed from parameters $\wW = [\hw_1| \dots | \hw_m], \htau$ as described above with $\|\hw_k\| = 1$ for all $k \in [m]$.
Then, there exists an absolute constant $C>0$ and $D_0$ such that, for dimension $D\geq D_0$, the difference between the gradients of $J$ and of the idealized objective $J_*$ obeys
	\begin{align}\label{eq:lemma_delta_grads_main}
		\norm{\nabla J(\htau)  - \nabla J_*(\htau) }_2  \leq 2\kappa^2 \sqrt{m}\norm{\htau - \tau}^2_2+C\Delta_{W,1} + \left(\frac{m^3 \delta_{\max}^2 t}{{N_{\text{train}}}}\right)^{1/2}
	\end{align}
	for $t>0$ with probability at least  $1 - 2 m^2 \exp\left(- \frac{ t }{C \kappa^4}\right)$ and where
	\begin{align*}
		\Delta_{W,1} \leq   \frac{ m^{1 / 2} \log(m)^{3/4}}{D^{1 / 4}}  \left[ 
			\| \wW - W\|_F  + \frac{\Delta_{W, O}^{1/2}}{D^{1 / 2}} + \left\| \sum^m_{k=1} w_k - \hw_k \right\|_2
		\right].
	\end{align*} 
\end{lemma}

\begin{proof}[Proof of Lemma \ref{lemma:delta_grads}]
 Recall that
	\begin{align*}
		J(\htau) = \frac{1}{2{N_{\text{train}}}}\sum_{i=1}^{N_{\text{train}}} \Big( \hat{f}(X_i, \htau) - {f}(X_i, \tau) \Big)^2.
	\end{align*}
	By chain rule we compute the gradient of $J$ w.r.t. $\htau$ as
	\begin{align*}
		\nabla J(\htau) &= \frac{1}{{N_{\text{train}}}}\sum_{i=1}^{N_{\text{train}}} \left(\hat{f}(X_i, \htau) - {f}(X_i, \tau)\right)\nabla \hat{f}(X_i, \htau).
	\end{align*}
	Adding $0 = (\hat f(X_i, {\tau}) -\hat f(X_i, {\tau}))\nabla \hat f(X_i, \htau)$ to $J(\htau)$ and applying the triangle inequality to $\nabla J- \nabla J_*$ allows us to separate the error caused by the weight approximation
	\begin{align}
		 \Big\Vert\nabla J(\htau) & - \nabla J_*(\htau) \Big\Vert_2 \notag \\\leq& \,
		\Big\Vert  \frac{1}{{N_{\text{train}}}}\Big(\sum_{i=1}^{N_{\text{train}}} (\hat f(X_i, \htau) - \hat f(X_i, {\tau})) \nabla \hat f(X_i, \htau) - \nabla \hat f(X_i, {\tau})\nabla \hat f(X_i, {\tau})^\top (\htau - {\tau}) \Big)\Big\Vert_2 \label{eq:proof_diff_grads}\\
		&+\Big\Vert \frac{1}{{N_{\text{train}}}}\sum^{N_{\text{train}}}_{i=1}\left(\hat f(X_i, {\tau}) - {f}(X_i, {\tau})\right)\nabla \hat f(X_i, \htau) \Big\Vert_2	.	\label{eq:proof_diff_grads_2nd_line}
	\end{align}
	To bound the first term in \eqref{eq:proof_diff_grads} denote $h(\lambda) = (1-\lambda)\htau + \lambda {\tau}$, then we have 
	\begin{align*}
		\hat f(X_i, \htau) - \hat f(X_i, {\tau}) &= \sum^m_{k = 1} g(\hw_k^\top X_i + \tau_k) - g(\hw_k^\top X_i + \htau_k)\\
		&= \sum^m_{k = 1} g(\hw_k^\top X_i + h(1)_k) - g(\hw_k^\top X_i + h(0)_k)\\
		&= \sum^m_{k = 1} \int^{h(1)_k}_{h(0)_k} g^{(1)}(\hw_k^\top X_i + u) du \\
		&= \sum^m_{k = 1} \int^{1}_{0} g^{(1)}(\hw_k^\top X_i + h(\lambda)_k) h'(\lambda)_k d\lambda \\
		&= \sum^m_{k = 1} \int^{1}_{0} g^{(1)}(\hw_k^\top X_i + h(\lambda)_k)  d\lambda (\tau_k - \htau_k). 
	\end{align*}
	Therefore, we can bound \eqref{eq:proof_diff_grads} as follows:  
	\begin{align*} 
		&\,\Big\Vert  \frac{1}{{N_{\text{train}}}}\Big(\sum_{i=1}^{N_{\text{train}}} (\hat f(X_i, \htau) - \hat f(X_i, {\tau})) \nabla  \hat f(X_i,\htau) - \nabla \hat f(X_i, {\tau})\nabla \hat f (X_i, {\tau})^\top (\htau - {\tau}) \Big)\Big\Vert_2\\
		\leq &\,\Big\Vert  \frac{1}{{N_{\text{train}}}}\Big(\sum_{i=1}^{N_{\text{train}}} \nabla \hat f(X_i, \htau )\Big(\int^1_0 \nabla \hat f(X_i, h(\lambda))d\lambda\Big)^\top - \nabla \hat f(X_i, {\tau})\nabla \hat f(X_i,{\tau})^\top   \Big)\Big\Vert \norm{(\htau - {\tau})}_2.
	\end{align*}
	Let us fix $\htau,{\tau}$ for now and write the last line in terms of matrices $\hat F, {F}, F^* \in \R^{{N_{\text{train}}}\times m}$, where the $i$-th row of these matrices is given by $\nabla \hat f(X_i, \htau ), \nabla \hat f(X_i, {\tau} )$
	and $\int^1_0 \nabla \hat f(X_i, h(\lambda))d\lambda$, respectively. We obtain
	\begin{align} 
		&\Big\Vert  \frac{1}{{N_{\text{train}}}}\Big(\sum_{i=1}^{N_{\text{train}}} \nabla \hat f(X_i, \htau )\Big(\int^1_0 \nabla \hat f(X_i, h(\lambda))d\lambda\Big)^\top - \nabla \hat f(X_i, {\tau})\nabla \hat f(X_i,{\tau})^\top   \Big)\Big\Vert \norm{(\htau - {\tau})}_2 \notag \\
		\leq & \frac{1}{{N_{\text{train}}}}\|\hat F^\top F^* - {F}^\top {F}\| \, \|\htau - {\tau}\|_2 \notag\\
		 \leq & \frac{1}{{N_{\text{train}}}}\Big( \|\hat F- {F}\| \|F^*\| +\|{F}\| \|{F} - F^*\|\Big)\norm{\htau - {\tau}}_2 \label{eq:Fmatrices_tbc}.
	\end{align}
	A simultaneous upper bound for $\|\hat F-{F}\|$ and $\|{F} - F^*\|$ can be established with elementary matrix arithmetic and the Lipschitz continuity of $g^{(1)}$:
	\begin{align*}
		\|\hat F- {F}\| &\leq \|\hat F-{F}\|_F= \Big[ \sum^{N_{\text{train}}}_{i=1}\sum^m_{k=1} (g^{(1)}( \langle \hw_k, X_i \rangle + \htau_k)-g^{(1)}( \langle \hw_k, X_i \rangle + \tau_k))^2
		\Big]^{\frac{1}{2}}\\
		&\leq\norm{g^{(2)}}_\infty \Big[ \sum^{N_{\text{train}}}_{i=1}\sum^m_{k=1} \Big(  \htau_k - {\tau}_k \Big)^2
		\Big]^{\frac{1}{2}} = \kappa \sqrt{{N_{\text{train}}}} \norm{\htau -{\tau}}_2, \\
	\end{align*}
	the same bound follows for $\norm{{F} - F^*}$. A crude bound for $\norm{{F}}$ is given by
	\begin{align*} 
	\norm{F} \leq \sqrt{{N_{\text{train}}}m} \max_{ik} \snorm{ F_{ik}} \leq \sqrt{{N_{\text{train}}}m} \norm{g^{(1)}}_\infty \leq \kappa \sqrt{{N_{\text{train}}}m}  ,
	\end{align*}
	the same bound follows for $\norm{F^\ast}$.
	Hence, we can continue from \eqref{eq:Fmatrices_tbc} with 
	\begin{align*}
	   \frac{1}{{N_{\text{train}}}}\Big( \|\hat F- {F}\| \|F^*\| +\|{F}\| \|{F} - F^*\|\Big)\norm{\htau - {\tau}}_2 \leq 2 \kappa^2  \sqrt{m} \| \htau - \tau \|_2^2.  
	\end{align*}
	
	The error \eqref{eq:proof_diff_grads_2nd_line} caused by the difference between $\wW$ and the original weights ${W}$ has the form 
	\begin{align*}
		&\,\Big\Vert\frac{1}{{N_{\text{train}}}}\sum^{N_{\text{train}}}_{i=1}\left( \hat f(X_i, {\tau}) - {f}(X_i, {\tau})\right)\nabla \hat f(X_i, \htau) \Big\Vert_2	\\=& \Big\Vert\frac{1}{{N_{\text{train}}}}\sum^{N_{\text{train}}}_{i=1}\Big(\sum^m_{k=1}g(X_i^{\top} \hw_k + {\tau}_k) - g(X_i^{\top} {w}_k + {\tau}_k) \Big)\nabla \hat f(X_i, \htau) \Big\Vert_2.
	\end{align*}
	Let us define
	\begin{align*}
	    \Delta_W^2 &:= \Big\Vert\frac{1}{{N_{\text{train}}}}\sum^{N_{\text{train}}}_{i=1}\Big(\sum^m_{k=1}g(X_i^{\top} \hw_k + {\tau}_k) - g(X_i^{\top} {w}_k + {\tau}_k) \Big)\nabla \hat f(X_i, \htau) \Big\Vert_2^2 \\
	    &= \sum^m_{\ell = 1}\left[\frac{1}{{N_{\text{train}}}}\sum^{N_{\text{train}}}_{i=1}\Big(\sum^m_{k=1}g(X_i^{\top} \hw_k + {\tau}_k) - g(X_i^{\top} {w}_k + {\tau}_k) \Big) g^{(1)}(\inner{X_i,\hw_\ell} + \htau_\ell) \right]^2\\
	\end{align*}
	To keep the expressions more compact, we define $Z_{ik\ell} := \varphi_{k,\ell}(X_i)$ and 
	\begin{align*}
		\varphi_{k,\ell}(x) := \Big(g(x^\top \hw_k + \tau_k) - g(x^\top w_k + \tau_k)\Big)g^{(1)}(x^\top \hw_\ell + \htau_\ell). 
	\end{align*}
Let us also define
	\begin{align*}
		\Delta_{W,1}^2 &:= \sum^m_{\ell = 1}\left[\frac{1}{{N_{\text{train}}}}\sum^{N_{\text{train}}}_{i=1}\sum^m_{k=1} \mathbb{E}[Z_{ik\ell}] \right]^2, \\
		\Delta_{W,2}^2 &:= \sum^m_{\ell = 1}\left[\frac{1}{{N_{\text{train}}}}\sum^{N_{\text{train}}}_{i=1}\sum^m_{k=1} (Z_{ik\ell}-\mathbb{E}[Z_{ik\ell}])\right]^2.
	\end{align*}
	Then, 
	\begin{align}\label{eq:delta_grads_first_split_deltaW}
		\Delta_W^2 &= \sum^m_{\ell = 1}\left[\frac{1}{{N_{\text{train}}}}\sum^{N_{\text{train}}}_{i=1}\sum^m_{k=1} Z_{ik\ell} \right]^2 \leq 2 \Delta_{W,1}^2 + 2\Delta_{W,2}^2.
	\end{align}
	In what follows we will control $\Delta_{W,1}^2, \Delta_{W,2}^2$ by using Hermite expansions and a concentration argument, respectively.

	We begin with $\Delta_{W,2}^2$: The first step is to establish that $Z_{ik\ell}-\mathbb{E}[Z_{ik\ell}]$ is subgaussian and to compute its subgaussian norm. 
	We remark that all expectations for the remainder of this proof are w.r.t. the inputs $X_1, \dots, X_{N_{\text{train}}} \sim \mathcal{N}(0, \Id_D)$.
	First note that by the mean value theorem there exists values $\xi_{i, k}$ such that 
	\begin{align*}
		Z_{ik\ell} = \langle \hw_k - w_k, X_i \rangle g^{(1)}(x^\top \hw_\ell + \htau_\ell) g^{(1)}(\xi_{i, k}),
	\end{align*} 
	where $g^{(1)}$ is a bounded function according to \ref{enum:activation}.
	We can combine this with the well known property of the sub-Gaussian norm which states that $\| Z_{ik\ell}-\mathbb{E}[Z_{ik\ell}] \|_{\psi_2} \leq C \| Z_{ik\ell}\|_{\psi_2}$ for some absolute constant $C>0$. This leads to 
	\begin{align*}
		\| Z_{ik\ell}-\mathbb{E}[Z_{ik\ell}] \|_{\psi_2} &\leq C \| Z_{ik\ell}\|_{\psi_2} \leq C \kappa^2 \| \langle \hw_k - w_k, X_i \rangle \|_{\psi_2} \leq C \kappa^2  \delta_{\max}
	\end{align*}
	for all $i\in [N_{\text{train}}]$, $k, \ell \in [m]$ and some absolute constant $C>0$. As a consequence we can apply the general Hoeffding inequality (cf. \cite[Theorem 2.6.2]{vershyninHighDimensionalProbabilityIntroduction2018}) which yields the estimate
	\begin{align*}
		\Delta_{W,2}^2 &= \frac{1}{{N_{\text{train}}^2}} \sum^m_{\ell = 1}\left(\sum^m_{k=1} \sum^{N_{\text{train}}}_{i=1} Z_{ik\ell}-\mathbb{E}[Z_{ik\ell}]\right)^2 \\ 
		&\leq \frac{1}{{N_{\text{train}}^2}} \sum^m_{\ell = 1}\left(\sum^m_{k=1} \left| \sum^{N_{\text{train}}}_{i=1} Z_{ik\ell}-\mathbb{E}[Z_{ik\ell}]\right| \right)^2 \leq \frac{1}{{N_{\text{train}}^2}} \sum^m_{\ell = 1} m^2 t^2 = \frac{m^3 t^2}{{N_{\text{train}}^2} },
	\end{align*}
	which holds using a union bound with probability at least
	\begin{align*}
		1 - \left( \sum^m_{k, \ell =1} 2 \exp\left(- \frac{c t^2 }{\sum^{{N_{\text{train}}}}_{i=1} \| Z_{ik\ell}-\mathbb{E}[Z_{ik\ell}] \|_{\psi_2}^2}\right)\right) \geq  1 - 2 m^2 \exp\left( - \frac{ t^2 }{C {N_{\text{train}}} \delta_{\max}^2 \kappa^4}\right),
	\end{align*}
	for all $t\geq 0$, where $c,C>0$ are absolute constants.
	This implies that there exists an absolute constant $C>0$ such that for all $t\geq 0$ 
	\begin{align}\label{eq:delta_grads_preliminary_deltaw2}
		\mathbb{P}\left( \Delta_{W,2}^2 \leq \frac{m^3 \delta_{\max}^2 t}{{N_{\text{train}}} } \right)  \geq  1 - 2m^2 \exp\left( - \frac{ t }{C \kappa^4}\right).
	\end{align}
	What remains is to control the means contained in $\Delta_{W,1}^2$. Using the shorthand $g_{\tau}(\cdot) = g(\cdot + \tau)$ and the Hermite expansion we get 
	\begin{align*}
		\mathbb{E}[Z_{ik\ell}] &= \mathbb{E}\left[ (g_{\tau_k}(\hw_k^\top X_i) - g_{\tau_k}(w_k^\top X_i))g^{(1)}_{\htau_\ell}(\hw_\ell^\top X_i) \right]\\
		&= \mathbb{E}\left[\left(\sum_{r\geq 0} \mu_{r}(g_{\tau_k})(h_r(\hw_k^\top X_i) - h_r(w_k^\top X_i)) \right)\sum_{t \geq 0} \mu_{t}(g^{(1)}_{\htau_\ell})h_t(\hw_\ell^\top X_i)\right] \\ 
		&= \sum_{r\geq 0} \mu_{r}(g_{\tau_k}) \mu_{r}(g^{(1)}_{\htau_\ell})\left( \inner{\hw_k, \hw_\ell}^r - \inner{w_k, \hw_\ell}^r \right),
	\end{align*}
	where the last two steps rely on the same properties of the Hermite expansion already used in the previous section. The summand corresponding to $r=0$ in the last line above vanishes, thus we have
	\begin{align*}
		\Delta_{W,1}^2 &= \sum^m_{\ell = 1}\left[\sum^m_{k=1} \sum_{r\geq 1} \mu_{r}(g_{\tau_k}) \mu_{r}(g^{(1)}_{\htau_\ell})\left( \inner{\hw_k, \hw_\ell}^r - \inner{w_k, \hw_\ell}^r \right) \right]^2.
	\end{align*}
	Denote now $$S_{r,\ell}:= \sum^m_{k=1} \mu_{r}(g_{\tau_k}) \mu_{r}(g^{(1)}_{\htau_\ell})\left( \inner{\hw_k, \hw_\ell}^r - \inner{w_k, \hw_\ell}^r \right),$$
	then, for any $R \geq 2$, we have 
	\begin{align}
		\Delta_{W,1}^2 &= \sum^m_{\ell = 1}\left( \sum_{r\geq 1} S_{r,\ell} \right)^2 \leq 2 \sum^m_{\ell=1} S_{1,\ell}^2 +  \sum_{r = 2}^{R-1} 2^r \sum^m_{\ell=1} S_{r,\ell}^2 +  2^R \sum^m_{\ell = 1}\left( \sum_{r\geq R} S_{r,\ell} \right)^2. \label{eq:delta_grads_split_series}
	\end{align} 
	Choose now $R=9$ and plug in the result from Lemma \ref{lemma:aux_for_delta_grads}, Lemma \ref{lemma:sum_1_hermites} and Lemma \ref{lemma:delta_grads_tail_bound} which yields for an appropriate constant $C>0$ the bound 
	\begin{align}\label{eq:sum_hermites_3part_intermediate}
		\Delta_{W,1}^2 &\leq C m \left(\frac{\log m }{D}\right)^{1/2} \left\| \sum^m_{k=1} w_k - \hw_k \right\|_2^2 +  \frac{C m \log m}{D}\left(  \Delta_{W, F}^2 + \left(\frac{\log m }{D}\right)^{1/2}  \Delta_{W, O }\right) \\ &+C \sqrt{m}\Delta_{W,F}^2.
	\end{align}
	Reordering the terms and taking the square root we receive 
	\begin{align*}
		\Delta_{W,1} \leq C\left(  m^{1/4} + m^{1/2} \left(\frac{\log m}{D}\right)^{1 / 2}\right)\Delta_{W, F} + C m^{1/2} \left(\frac{\log m}{D}\right)^{3 / 4}\Delta_{W, O}^{1/2}\\
		+ C m^{1/2}  \left(\frac{\log m }{D}\right)^{1/4} \left\| \sum^m_{k=1} w_k - \hw_k \right\|_2 \\ 
		\leq C \log(m)^{3/4} \left[ 
			\left(  m^{1/4} +  \frac{m^{1 / 2}}{D^{1 / 2}}\right)\Delta_{W, F} + \frac{m^{1/2}}{D^{3 / 4}}\Delta_{W, O}^{1/2} + \frac{m^{1 / 2}}{D^{1 / 4}} \left\| \sum^m_{k=1} w_k - \hw_k \right\|_2
		\right].
	\end{align*}
	Lastly, we can use 
	\begin{align*} 
		m^{1/4} +  \frac{m^{1 / 2}}{D^{1 / 2}} \leq  m^{1/4} +  \frac{m^{1 / 2}}{D^{1 / 4}} \leq  \frac{2 m^{1 / 2}}{D^{1 / 4}} 
	\end{align*}
	since $m\geq D$ followed by $\Delta_W \leq C(\Delta_{W,1} + \Delta_{W,2})$ to conclude the proof. Note that we can simply separate the constant that appears in the definition of $\Delta_{W,1}$ to appear outside of $\Delta_{W,1}$, such that we arrive at the formulation appearing in the original statement. 
\end{proof}
The previous result shows that the gradients associated with our two objective functions $J, J_*$ fulfill 
\begin{align*}
		  \norm{\nabla J(\htau)  - \nabla J_*(\htau) }_2  \leq \kappa^2 \sqrt{m}\norm{\htau - \tau}^2_2+\Delta_W,
\end{align*} 
according to Lemma \ref{lemma:delta_grads}, where $\Delta_W$ depends on the accuracy of the weight approximation. Next, we leverage this to establish sufficient conditions on the accuracy $\Delta_W$ and our initial shift estimate under which both gradient descent iterations will remain close to each other over any number of GD steps. The upcoming proof requires that one of the two gradient descent iterations does converge, which in combination with Lemma \ref{lemma:delta_grads} allows to control the other iteration locally. It was already established in Lemma \ref{lem:lower_bound_expectation_technical} that $A$ is positive definite in expectation. This suggests that $J_*(\htau) = (\htau - \tau)^\top A (\htau - \tau)$ is strictly convex, provided enough samples $N_{\text{train}}$ are used to concentrate $A$ around its expectation $E$. In particular, strict convexity directly implies that $\htau_*^{(n)}$ converges to the true biases $\tau$. We will show this as part of the proof of Theorem \ref{thm:local_result}, but for the sake of simplicity we will assume positive definiteness of $A$ in the next statement.
\begin{lemma}\label{lemma:delta_gradient_descent_iteration}
	Denote by $\htau^{(n)},\htau_*^{(n)}$ the gradient descent iterations given by \eqref{eq:gd_iteration_proof} and \eqref{eq:idealized_gd_iteration}, respectively.
    Assume that the objective functions $J, J_*$ defined above fulfill 
	\begin{align}\label{eq:delta_gradient_assumption}
	\| \nabla J(\htau) - \nabla {J_*}(\htau)\|_2  \leq L \norm{\htau - {\tau}}_2^2 + \Delta_W,
\end{align}
	for some $L, \Delta_W \geq 0$ and any $\htau\in \R^m$. Furthermore, assume that the matrix $A$ in \eqref{eq:def_A} fulfills $\lambda_{\min}:=\lambda_{\min}(A) > 0$.
	If $\Delta_W \leq \frac{\lambda_{\min}^2}{16L}$ and both gradient descent iterations are started with the same step size $\gamma \leq \|A\|^{-1}$ and from the same initialization $\htau^{(0)}={\htau_*}^{(0)}$, adhering to the bound
	\begin{align}\label{eq:thm4_assumptions}
		\| \htau^{(0)} - {\tau}\| _2\leq \frac{\lambda_{\min}}{4\sqrt{2} L}, 
	\end{align}
	then the distance between both iterations at gradient step $n \in \N$ satisfies
	\begin{align*}
		\|\htau^{(n)} - {\htau_*}^{(n)}\|_2 \leq \xi^n \|\htau^{(0)} - {\tau}\|_2+ \frac{2\Delta_W}{\lambda_{\min}}\left(1-\xi^n \right),
	\end{align*}
	for $\xi = 1 -\frac{\gamma \lambda_{\min}}{2} \in [0, 1)$. 
\end{lemma}
\begin{proof}[Proof of Lemma \ref{lemma:delta_gradient_descent_iteration}]
	Plugging in the gradient descent iteration with a simple expansion yields
	\begin{align*}
		&\norm{\htau^{(n+1)} - {\htau_*}^{(n+1)}}_2 \\=& \norm{\htau^{(n)} - {\htau_*}^{(n)} - \gamma\left(\nabla J(\htau^{(n)}) - \nabla {J_*}({\htau_*}^{(n)})\right)}_2\\
		=& \norm{\htau^{(n)} - {\htau_*}^{(n)} - \gamma\left(\nabla J(\htau^{(n)}) - \nabla {J_*}({\htau}^{(n)})\right)-\gamma\left(\nabla {J_*}(\htau^{(n)}) - \nabla {J_*}({\htau_*}^{(n)})\right)}_2\\
		=& \norm{\Big(\Id_m - {\gamma}A \Big) (\htau^{(n)} - {\htau_*}^{(n)}) - \gamma\left(\nabla J(\htau^{(n)}) - \nabla {J_*}({\htau}^{(n)})\right)}_2\\
		\leq& \norm{\Big(\Id_m - {\gamma}A \Big) (\htau^{(n)} - {\htau_*}^{(n)})}_2 + \gamma \norm{\nabla J(\htau^{(n)}) - \nabla {J_*}({\htau}^{(n)})}_2,
	\end{align*}
	where we used the definition of the iterations in the first line followed by a simple expansion and the triangle inequality in the last line. 
	The left term of the last line can be bounded with the spectral norm of $\Id_m - {\gamma}A$ and the right term according to our initial assumption \eqref{eq:delta_gradient_assumption}:
	\begin{align*}
		\|\htau^{(n+1)} - {\htau_*}^{(n+1)}\|_2 &\leq \|\Id_m - {\gamma}A \|\,\| \htau^{(n)} - {\htau_*}^{(n)}\|_2 + \gamma L\|\htau^{(n)} - {\tau}\|^2_2 + \gamma\Delta_W\\
		&\leq (1-\gamma \lambda_{\text{min}})\|\htau^{(n)} - {\htau_*}^{(n)}\|_2 + \gamma L\|\htau^{(n)} - {\tau}\|^2_2 + \gamma\Delta_W,
	\end{align*}
	where the second inequality follows from the bound on the minimal eigenvalue of $A$. Expanding the right term of the last line with ${\htau_*}^{(n)}$ yields
	\begin{align}
		\|\htau^{(n+1)} &- {\htau_*}^{(n+1)}\|_2 \notag \\
		\leq &(1-\gamma \lambda_{\text{min}})\| \htau^{(n)} - {\htau_*}^{(n)}\|_2 + \gamma L\|\htau^{(n)} - {\htau_*}^{(n)} + {\htau_*}^{(n)} -{\tau}\|^2_2 + \gamma \Delta_W\notag \\
		\leq &(1-\gamma \lambda_{\text{min}})\|\htau^{(n)} - {\htau_*}^{(n)}\|_2 + 2\gamma L \|\htau^{(n)} - {\htau_*}^{(n)}\|^2_2 + 2\gamma L \|{\htau_*}^{(n)} - {\tau}\|^2_2 + \gamma \Delta_W. \label{eq:intermediate_0}
	\end{align}
	We can now use the fact that the gradient descent iteration \eqref{eq:idealized_gd_iteration} in combination with the convexity of the idealized objective $J_*$ ($\lambda_{\min}(A) > 0)$ allows for the recursive bound 
	\begin{align*}
	    \|{\htau_*}^{(n)} - {\tau}\|_2 &= \|{\htau_*}^{(n-1)} -\gamma \nabla J_*({\htau_*}^{(n-1)}) - {\tau}\|_2 = \|{\htau_*}^{(n-1)} -\gamma A({\htau_*}^{(n-1)} - \tau) - {\tau}\|_2 \\ 
	    &= \|(\Id_m - \gamma A) ( {\htau_*}^{(n-1)}- \tau) \|_2 \leq \| \Id_m - \gamma A \| \| {\htau_*}^{(n-1)}- \tau \|_2 \\
	    &\leq \| \Id_m - \gamma A \|^n  \| {\htau_*}^{(0)}- \tau \|_2 \leq (1 - \gamma \lambda_{\min})^n \delta_0,
	\end{align*}
	where we have denoted by $\delta_0 = \|\htau^{(0)} - {\tau}\|$ the initial error.
	Plugging this into \eqref{eq:intermediate_0} results in 
	\begin{align}
		\|\htau^{(n+1)} - &{\htau_*}^{(n+1)}\|_2  \notag \\\leq &(1-\gamma \lambda_{\text{min}})\|\htau^{(n)} - {\htau_*}^{(n)}\|_2 + 2\gamma L \|\htau^{(n)} - {\htau_*}^{(n)}\|^2_2 + 2\gamma L (1-\gamma \lambda_{\text{min}})^{2n}\delta_0^2 + \gamma\Delta_W.\label{eq:intermediate}
	\end{align}
		Define $\Delta_n := \max_{k\leq n} \|\htau^{(k)} - {\htau_*}^{(k)}\|_2$. We first show by induction
	that $\Delta_n \leq \lambda_{\min}/4L$ provided that $\delta_0$ and $\Delta_W$ are sufficiently small.
	For step $n= 0$, we have $\|\htau^{(0)} - {\htau_*}^{(0)}\|_2 = 0$, so the statement is clearly true. Assume now it holds for $n$ and we have to show the induction step. In other words we have to show $\|\htau^{(n+1)} - {\htau_*}^{(n+1)}\|_2 \leq\lambda_{\min}/4L$, so the same bound would hold for $\Delta_{n+1}$. We continue from \eqref{eq:intermediate}, and get
	\begin{align*}
		\|\htau^{(n+1)} - {\htau_*}^{(n+1)}\|_2 
		&\leq (1-\gamma \lambda_{\text{min}} + 2\gamma L \Delta_n)\| \htau^{(n)} - {\htau_*}^{(n)}\|_2 + 2\gamma L (1-\gamma \lambda_{\text{min}})^{2n}\delta_0^2 + \gamma\Delta_W.
	\end{align*}
	Using the induction hypothesis $\Delta_{n}\leq \lambda_{\min}/4L$, this simplifies to
	\begin{align*}
		\|\htau^{(n+1)} - {\htau_*}^{(n+1)}\|_2 \leq \left(1-\gamma \lambda_{\text{min}} / 2\right)\| (\htau^{(n)} - {\htau_*}^{(n)})\|_2 + 2\gamma L(1-\gamma \lambda_{\text{min}})^{2n}\delta_0^2 + \gamma\Delta_W.
	\end{align*}
	To keep the computation more compact, we will denote 
    \begin{align*}
    \xi :=  1-\frac{\gamma \lambda_{\text{min}}}{2}.
    \end{align*}
	Now we can repeat the same computations for $\|\htau^{(k)} - {\htau_*}^{(k)}\|_2$, $k\leq n$ as well. This leads to
	\begin{align*}
		\|\htau^{(n+1)} - {\htau_*}^{(n+1)}\|_2 \leq 2\gamma L\delta_0^2\sum^n_{k=0} \xi^k  (1-\gamma \lambda_{\text{min}})^{2(n-k)} + \gamma \Delta_W \sum_{k=0}^{n}\xi^{k},
	\end{align*}
	where we used $\norm{\htau^{(0)} - {\htau_*}^{(0)}}_2 = 0$. Both sums are uniformly bounded in $n$, as can be seen by 
	\begin{align}
		\|\htau^{(n+1)} - {\htau_*}^{(n+1)}\|_2 \leq \, & 2\gamma L\delta_0^2\frac{\xi^{n+1}-(1-\gamma \lambda_{\text{min}})^{2(n+1)}}{\xi-(1-\gamma \lambda_{\text{min}})^{2}} + \gamma \Delta_W\frac{1 - \xi^{n+1}}{1- \xi} \label{eq:thm4_intermediate_after_ind}\\
		\leq \, &
		2\gamma L\delta_0^2\frac{\xi^{n+1}-(1-\gamma \lambda_{\text{min}})^{2(n+1)}}{\frac{3}{2}\gamma\lambda_{\min} - \gamma^2 \lambda_{\min}^2} + \frac{2\Delta_W}{\lambda_{\min}} \notag\\
		\leq \, &
		2 L\delta_0^2\frac{\xi^{n+1}}{\frac{3}{2}\lambda_{\min} - \gamma\lambda_{\min}^2} + \frac{2\Delta_W}{\lambda_{\min}} \leq 4 L\delta_0^2\frac{\xi^{n+1}}{\lambda_{\min}} + \frac{2\Delta_W}{\lambda_{\min}}.\notag
	\end{align}
	Now we have $4 L\delta_0^2 \xi^{n+1}\lambda_{\min}^{-1} \leq 4 L\delta_0^2 \lambda_{\min}^{-1}$. Furthermore, $4 L\delta_0^2 \lambda_{\min}^{-1}\le \frac{\lambda_{\min}}{8L} $ as long as
	\begin{align*}
		\delta_0^2 \leq \frac{\lambda_{\min}^2}{32 L^2},
	\end{align*}
	which holds according to our initial assumption \eqref{eq:thm4_assumptions}. Similarly, as $\Delta_W \leq \frac{\lambda_{\min}^2}{16L}$ by assumption, we get $\frac{2\Delta_W}{\lambda_{\min}} \leq \frac{\lambda_{\min}}{8L}$
	This means we now have
	\begin{align*}
		\Delta_{n+1}  \leq \frac{\lambda_{\min}}{8L} + \frac{\lambda_{\min}}{8L} \leq \frac{\lambda_{\min}}{4L},
	\end{align*}
	which concludes the proof of the induction establishing that the two iterations remain close to each other so that $\max_{k\leq n} \|\htau^{(k)} - {\htau_*}^{(k)}\|_2\leq \lambda_{\min}/4L$ for all $n \in \N$. To arrive at the final statement we can continue from \eqref{eq:thm4_intermediate_after_ind}
	\begin{align*}
		\|\htau^{(n)} - {\htau_*}^{(n)}\|_2 &\leq 2\gamma L\delta_0^2\frac{\xi^{n}-(1-\gamma \lambda_{\text{min}})^{2n}}{\xi-(1-\gamma \lambda_{\text{min}})^{2}} + \gamma \Delta_W\frac{1 - \xi^{n}}{1- \xi}\\
		&\leq
		\frac{4L\delta_0^2}{\lambda_{\min}} \xi^n + \frac{2\Delta_W}{\lambda_{\min}}\left(1-\xi^n \right).
	\end{align*}
\end{proof}
\subsection{Concluding the proof of Theorem \ref{thm:local_result}}
\label{subsec:concluding_theorem4}
Theorem \ref{thm:local_result} tells us how accurate the weight approximation and shift initialization has to be such that the initial shifts can be further improved w.h.p. by minimizing the empirical loss  $J(\htau) = \frac{1}{2{N_{\text{train}}}}\sum_{i=1}^{N_{\text{train}}} \Big( \hat{f}(X_i, \htau) - {f}(X_i, \tau) \Big)^2$ on a set of generic inputs via gradient descent. The proof of Theorem \ref{thm:local_result} follows directly by combining Lemma \ref{lem:lower_bound_expectation_technical}, Lemma \ref{lemma:delta_grads} and Lemma \ref{lemma:delta_gradient_descent_iteration}. Based on the first result we prove that the idealized gradient descent iteration $\tau^{(n)}_*$ will w.h.p. and linear rate converge to the ground-truth shifts $\tau$ by establishing
the strict convexity of $J_*$. The second set of auxiliary statements (i.e., Lemma \ref{lemma:delta_grads}-\ref{lemma:delta_gradient_descent_iteration}) then shows that the gradient descent iteration derived from the empirical risk $J(\htau)$ will stay close to $\tau^{(n)}_*$ if weight approximations $\wW$ and initial shifts $\tau^{(0)}$ are sufficiently accurate. 
\begin{proof}[Proof of Theorem \ref{thm:local_result}]
	Denote $E =\mathbb{E}_{X_1,\ldots,X_{N_{\text{train}}} \sim \CN(0,\Id_D)}[A]$ with $A$ as in \eqref{eq:def_A}, and constructed from inputs $X_1, \dots, X_{N_{\text{train}}} \sim \CN(0, \Id_D)$.
	According to Lemma \ref{lem:lower_bound_expectation_technical}, there exist constants $\omega, C_1>0$, which only depend on $g$ and $\tau_{\infty}$, with 
	\begin{align*}
	\lambda_{m}(E) \geq \omega - C_1 \frac{(m-1) \log^2 m }{D^2} \geq \frac{\omega}{2},
	\end{align*}
	provided $(2C_1 / \omega) m \log^2 m \leq D^2$, as assumed in Theorem \ref{thm:local_result}. 
	Note now that  $A$ is a sum of positive semi-definite rank-1 matrices. Thus we can apply the Matrix Chernoff bound in Lemma \ref{lem:chernoff_bound} to get the concentration bound
		\begin{align}
			\label{eq:aux_first_lower_bound}
			\mathbb{P}\left(\lambda_{m}(A) \geq \frac{\lambda_{m}(E)}{4}\right) \geq 1 - m \cdot 0.7^{ \frac{{N_{\text{train}}}\lambda_{m}(E)}{R}},
		\end{align}
		where $R = \sup_{x\in \R^D} \|\nabla \hat f(\tau,x)\|_2^2 \leq m\norm{g^{(1)}}_\infty^2 \leq m \kappa^2.$
    From $0.7 < \exp(-1/3)$ now follows that 
    \begin{align}\label{eq:potato_potato}
    \mathbb{P}\left( \lambda_{m}(A) \geq \frac{\omega}{8} \right) \geq 1- m \cdot \exp\left(- \frac{{N_{\text{train}}} \omega }{6 m \kappa^2}\right).
    \end{align}
    For the remainder of the proof we will condition on the event that the bound in \eqref{eq:potato_potato} holds. \\
	By the result of Lemma \ref{lemma:delta_grads}, the difference between the gradients $\nabla J, \nabla J_*$ satisfies \begin{align}\label{eq:from_lemma_delta_grads_main}
		\norm{\nabla J(\htau)  - \nabla J_*(\htau) }_2  &\leq 2\kappa^2 \sqrt{m}\norm{\htau - \tau}^2_2+\Delta_{W}  \\
		\Delta_{W} &= C \Delta_{W,1} + \left(\frac{m^3 \delta_{\max}^2 t}{{N_{\text{train}}}}\right)^{1/2},
	\end{align}
	for a constant $C>0$ and $t>0$ with probability at least  $1 - 2 m^2 \exp\left(- \frac{ t }{C \kappa^4}\right)$  where 
	\begin{align*}
		\Delta_{W,1} \leq \frac{m^{1 / 2} \log(m)^{3/4}}{D^{1 / 4}}  \left[ 
			\| \wW - W\|_F  + \frac{\Delta_{W, O}^{1/2}}{D^{1 / 2}} + \left\| \sum^m_{k=1} w_k - \hw_k \right\|_2
		\right].
	\end{align*}
	Assuming the event associated with \eqref{eq:from_lemma_delta_grads_main} occurs, we can invoke Lemma \ref{lemma:delta_gradient_descent_iteration} with $L=2\kappa^2\sqrt{m}$ meeting its condition by choosing an appropriate constant $C$ in \eqref{eq:assumption_thm_local}.
	Then, for a step-size $\gamma \leq 1 / \|A\|$, $\lambda_{\min} = \lambda_m(A)$ and $\xi = 1 - \gamma \lambda_{\min} / 2$, Lemma \ref{lemma:delta_gradient_descent_iteration} yields 
	\begin{align}\label{eq:delta_grad_traj_thm}
		\|\htau^{(n)} - {\htau_*}^{(n)}\|_2 \leq \xi^n \|\htau^{(0)} - {\tau}\|_2+ C \left(1-\xi^n \right) \Delta_W  .
	\end{align}
	The bound in \eqref{eq:delta_grad_traj_thm} controls the deviation of the gradient descent iteration \eqref{eq:gd_iteration} from the idealized gradient descent iteration \eqref{eq:idealized_gd_iteration}. What remains to be shown is that the idealized iteration converges to the correct parameter $\tau$ which follows directly by the lower bound on the minimal eigenvalue $\lambda_{\min}$. In fact, we have $J_*(\htau) = (\htau-{\tau})^\top A(\htau - {\tau})$ and 
	\begin{align*}
	\| {\htau_*}^{(n)} - \tau \|_2 &= \| {\htau_*}^{(n-1)} - \gamma \nabla J_*({\htau_*}^{(n-1)}) - \tau  \|_2 = \| (\Id_D - \gamma A)(\htau_*^{(n-1)} - {\tau})\|_2\\
	&\leq  \|\Id_D - \gamma A\|^n \| \htau_*^{(0)} - {\tau}\|_2 \leq (1-\gamma \lambda_{\min})^n\| \htau^{(0)} -{\tau}\|_2.
	\end{align*}
	Applying the triangle inequality to \eqref{eq:delta_grad_traj_thm} therefore yields 
\begin{align*}
	\|{\htau}^{(n)} - {\tau}\|_2 &\leq \|\htau_*^{(n)} - {\tau}\|_2+ \|\htau^{(n)} - {\htau_*}^{(n)}\|_2\\
	&\leq
\left(\left(1- {\gamma\lambda_{\min}}\right)^n +  \xi^n \right)\|\tau^{(0)} - {\tau}\|_2 + C \left(1-\xi^n \right) \Delta_W\\
&\leq 2 \xi^n \|\htau^{(0)} - {\tau}\|_2+C \left(1-\xi^n \right) \Delta_W.
\end{align*}
The main statement follows by a union bound over the events described above and by unifying the involved constants. \color{black}
\end{proof}

\section{Proof of Theorem \ref{thm:main_theorem}}\label{app:pfmain}
\begin{proof}[Proof of Theorem \ref{thm:main_theorem} ]
According to our assumptions, there exist $C, D_0$ such that the conditions of Theorem \ref{thm:weight_recovery} are fulfilled, and therefore we conclude that the ground truth weights obey  \ref{enum:RIP} - \ref{enum:GInverse} of Definition \ref{def:assumptions_overcomplete} and that the weight recovery (Algorithm \ref{alg:recover_weights}) returns vectors $\CU$ such that for all $\hw \in \CU$ we have
\begin{align}\label{eq:uniform_error_by_wr_in_main}
\max_{k\in[m]} \min_{s\in \{-1,+1\}} \norm{\hw - s w_k}_2 \leq C_1 (m / \alpha )^{1/4} \epsilon^{1/2},
\end{align}
with probability at least
\begin{align*}
1- \frac{1}{m} - D^2 \exp\left(-\min\{\alpha,1\} t / C_1 \right) - C_1\exp(- \sqrt{m}/C_1).
\end{align*}
Denote the weight approximations obtained in the last step by $\{ \hw_1, \dots, \hw_m \} \subset \bbS^{D-1}$.
There exists a permutation $\pi$ of these vectors such that ${w}_k \approx \pm \hw_{\pi(k)}$ for all $k\in [m]$. To invoke Proposition \ref{prop:initialization}, we now need to make sure that
\begin{align}\label{eq:meeting_cond_init_inmain}
\max_{k\in[m]} \min_{s\in \{-1,+1\}} \norm{\hw_{\pi(k)}- s {w}_k}_2 \leq \frac{1}{C_2}\frac{D^{1/2}}{ m \sqrt{  \log m}  }.
\end{align}
By applying the uniform error bound \eqref{eq:uniform_error_by_wr_in_main} above, we have
\begin{align*}
C_1 (m / \alpha )^{1/4} \epsilon^{1/2} \leq \frac{1}{C_2}\frac{D^{1/2}}{ m \sqrt{  \log m}  } \Leftrightarrow \epsilon \leq \frac{\sqrt{\alpha}}{C_1^2 C_2} \frac{D}{ m^{5/2} \log m   },
\end{align*}
which is guaranteed by our upper bound \eqref{eq:eps_bound_main} on $\epsilon$ for an appropriate constant. This in turn shows that \eqref{eq:meeting_cond_init_inmain} is met. Hence, by Proposition \ref{prop:initialization}, Algorithm \ref{alg:initialization} returns initial shifts $\htau$ such that there exists a $\alpha' \leq \alpha$ such that
\begin{align*}
\norm{\tau - \htau}_2 &\leq
C_2\sqrt{m}\epsilon + C_2 m^{3/2} \left( \frac{\log m}{D}\right)^{3/4} \max_{k\in[m]} \min_{s\in \{-1,+1\}} \norm{\hw_{\pi(k)}- s {w}_k}_2\\
&\leq C_2 \sqrt{m}\epsilon +  C_2 m^{3/2} \left( \frac{\log m}{D}\right)^{3/4}  C_1 (m / \alpha )^{1/4} \epsilon^{1/2}\leq \frac{1}{C m^{1/2}},
\end{align*} where the last line follows from \eqref{eq:eps_bound_main} chosen with an appropriate constant $C>0$. First, note that this implies that the signs learned by the parameter initialization will be correct. We denote this set of signs as $\bar{s}_1, \dots, \bar{s}_m$. Additionally, the last inequality implies that, for the given step-size, the condition of Theorem \ref{thm:local_result} (see \eqref{eq:assumption_thm_local}) w.r.t. the error in the initial shift is met. Another criteria that has to be met for Theorem \ref{thm:local_result} is that
\begin{align}\label{eq:main_gd_cond2_weights}
    \frac{C m^{1 / 2} \log(m)^{3/4}}{D^{1 / 4}}  \left( 
		\| \wW - W\|_F  + \frac{\Delta_{W, O}^{1/2}}{D^{1 / 2}} + \left\| \sum^m_{k=1} w_k - \hw_k \right\|_2 \right) &\leq \frac{1}{C \sqrt{m}},\\
        \left(\frac{m^3 \delta_{\max}^2 t}{{N_{\text{train}}}}\right)^{1/2}  &\leq \frac{1}{C \sqrt{m}}, \label{eq:main_gd_cond2_weights_probabilistic_term}
\end{align}
where $\Delta_{W, O} = \sum^m_{k\neq k'} \snorm{\inner{w_k - \hw_k, w_{k'} - \hw_{k'}}}$.
We begin with the upper term and rely on worst case bounds which express the different quantities in terms of the uniform error $$\delta_{\max} = \max_{k\in[m]} \min_{s\in \{-1,+1\}} \norm{\hw_{\pi(k)}- s {w}_k}_2,$$
such that 
\begin{align}
    \| W - \wW \|_F &\leq m^{1/2} \delta_{\max}, \\ 
    \frac{\Delta_{W, O}^{1/2}}{D^{1 / 2}} &\leq \frac{m \delta_{\max}}{D^{1/2}},\\
    \left\|\sum^m_{k = 1} w_k - \hw_k \right\|_2 &\leq m \delta_{\max}.  
\end{align}
Based on these bounds and after adjusting the constants we can simplify \eqref{eq:main_gd_cond2_weights} to 
\begin{align*}
\delta_{\max} \leq \frac{D^{1/4}}{C m^2 \log(m)^{3/4}} \Leftrightarrow \epsilon \leq \frac{D^{1/2} \alpha^{1/2}}{C m^{9/2} \log(m)^{3/2}},
\end{align*}
which is covered by our initial assumptions on the accuracy. Note that this implies for \eqref{eq:main_gd_cond2_weights_probabilistic_term} by plugging in the bound for $\delta_{\max}$ that
\begin{align*}
    \left(\frac{m^3 \delta_{\max}^2 t}{{N_{\text{train}}}}\right)^{1/2}  \leq  \left(\frac{t D^{1/2}}{N_{\text{train} }m  \log(m)^{3/2}}\right)^{1/2}.
\end{align*}
Using $N_{\text{train}} \geq m$ and $t = D^{1/2}$ this implies 
\begin{align*}
	\left(\frac{m^3 \delta_{\max}^2 t}{{N_{\text{train}}}}\right)^{1/2} \leq \frac{1}{N_{\text{train}}^{1/2} \log(m)^{3/4}} \leq \frac{1}{ C m^{1/2}},
\end{align*}
for $D,m$ sufficiently large. Therefore all conditions of Theorem \ref{thm:local_result} are satisfied. Hence, there exists a constant $C_4$ such that the gradient descent iteration \eqref{eq:gd_iteration_proof} started from initial shifts $\htau^{[0]} = \htau$ will produce iterates $\htau^{[0]},\dots, \htau^{[N_{\text{GD}}]}$ such that
  \begin{align}
  	\norm{{\tau}-\htau^{[n]}_{\pi}}_2 \leq \frac{C_4 m^{1 / 2} \log(m)^{3/4}}{D^{1 / 4}}   \left( 
		\| \wW - W\|_F  + \frac{\Delta_{W, O}^{1/2}}{D^{1 / 2}} + \left\| \sum^m_{k=1} w_k - \hw_k \right\|_2 \right) \\ + \left(\frac{m^3 \delta_{\max}^2 t}{{N_{\text{train}}}}\right)^{1/2}
		+ C_4 \frac{1}{\sqrt{m}}\xi^{n} ,
  	\end{align}
  	for all $n \in [N_{\text{GD}}]$, some permutation $\pi$ and some constant $\xi \in [0,1)$ with probability at least
  	\begin{align*}
  	1- m \exp(-{N_\text{train}} / C_4m)-2m^2\exp\left( - D^{1/2} / C_4\right).
  	\end{align*}
  	After unifying the constants and using the bound on $\delta_{\max}$, the statement of Theorem \ref{thm:main_theorem} follows. 
\end{proof}
\section{Auxiliary results}\label{app:aux}

\begin{lemma}
	\label{lemma:hermite_coefficients_derivatives}
	Let $g \in L_2(\R,w_H)$ be $K$-times continuously differentiable and assume
	\begin{align}
	\label{eq:boundary_integrals}
	\lim\limits_{t\rightarrow \infty} g^{(k)}(t)h_r(t) w_H(t) = \lim\limits_{t\rightarrow -\infty} g^{(k)}(t)h_r(t) w_H(t) = 0
	\end{align}
	for all $0 \leq k \leq K$. For any $r\in \N\cup\{0\}$ and $k \in [0,\ldots,K]$ we have
	\begin{align*}
	\mu_r(g^{(n)}) = \sqrt{{n+r \choose r}n!}\mu_{r+n}(g).
	\end{align*}
	\end{lemma}
	\begin{proof}
	The  Hermite polynomials, weighted by $ \exp(-t^2/2)$, satisfy the relation
	\begin{align*}
	&\frac{d}{dt} \left(h_r(t)\exp\left(-\frac{t^2}{2}\right)\right) = \frac{d}{dt}\left(\sqrt{\frac{1}{r!}}(-1)^r  \frac{d^r}{dt^r}\exp\left(-\frac{t^2}{2}\right)\right) = \sqrt{\frac{1}{r!}}(-1)^r  \frac{d^{r+1}}{dt^{r+1}}\exp\left(-\frac{t^2}{2}\right)\\
	&\quad= -\sqrt{r+1} \sqrt{\frac{1}{(r+1)!}}(-1)^{r+1} \frac{d^{r+1}}{dt^{r+1}}\exp\left(-\frac{t^2}{2}\right) = -\sqrt{r+1} h_{r+1}(t)\exp\left(-\frac{t^2}{2}\right).
	\end{align*}
	Therefore, by applying integration by parts, we obtain
	\begin{align*}
	\mu_r(g^{(n)}) &= \int g^{(n)}(t) h_r(t) w_H(t)dt = \left[g^{(n-1)}(t) h_r(t) w_H(t)\right]_{-\infty}^{\infty} - \int g^{(n-1)}(t) \frac{d}{dt}\left(h_r(t) w_H(t)\right)dt \\
	&= 0 + \sqrt{r+1}\int g^{(n-1)}(t)h_{r+1}(t)w_H(t)dt = \sqrt{r+1}\mu_{r+1}(g^{(n-1)}),
	\end{align*}
	where the boundary terms vanish due to \eqref{eq:boundary_integrals}. Applying the same computation
	$n$-times, we obtain
	\begin{align*}
	\mu_r(g^{(n)}) = \sqrt{\prod_{\ell = 1}^{n}(r + \ell)} \mu_{r+n}(f) = \sqrt{\frac{(r+n)!}{r!}}\mu_{r+n}(g).
	\end{align*}
	\end{proof}
\begin{lemma}\label{lemma:aux_grammian_frame_bound}
Let $w_k \in \R^D$ for $k=1,\dots,m$, and  denote by $G_n \in \R^{m \times m}$ the Grammian matrix associated with $(w_k^{\otimes n})_{k\in [m]}$, which is given by $(G_n)_{ij} = \inner{w_i, w_j}^n$. Then, for any n-mode tensor $T \in \R^{{D\times \dots \times D}}$, we have 
\begin{align}
	\sum^m_{k=1} \inner{T, w_k^{\otimes n}}^2 \leq \| G_n \| \| T\|^2_F.
\end{align}
\end{lemma}
\begin{proof}
First note that we can express the Frobenius inner product as an ordinary inner product over $\R^{D^n}$ with the help of the $\opvec(\cdot)$ operator, since $\inner{T, w_k^{\otimes n}} = \inner{\opvec(T), \opvec(w_k^{\otimes n})}$. Let us denote 
\begin{align*}
	W_n := \Big( \opvec(w_1^{\otimes n}) \Big| \dots \Big| \opvec(w_m^{\otimes n}) \Big) \in \R^{D^n \times m}.
\end{align*}
Then, the following chain of inequalities hold 
\begin{align*}
	\sum^m_{k=1} \inner{T, w_k^{\otimes n}}^2 &
	= \sum^m_{k=1} \inner{\opvec(T), \opvec(w_k^{\otimes n})}^2 \\
	&= \sum^m_{k=1} \opvec(T)^\top \opvec(w_k^{\otimes n}) \opvec(w_k^{\otimes n})^\top \opvec(T)\\
	&=\opvec(T)^\top W_n W_n^\top \opvec(T)  \\
	&\le \|W_n^\top W_n\| \cdot \|\opvec(T)\|_2^2 = \|W_n^\top W_n\| \|T\|_F^2.
\end{align*}
Since $\|W_n^\top W_n\| = \|G_n\|$, this finishes the proof. 
\end{proof}
\begin{lemma}\label{lemma:grammians_inc_greshgorin}
	Let $w_k \in \bbS^{D-1}$ for $k=1,\dots,m$ be unit vectors, and  denote by $G_n \in \R^{m \times m}$ the Grammian matrix associated with $(w_k^{\otimes n})_{k\in [m]}$, which is given by $(G_n)_{ij} = \inner{w_i, w_j}^n$. Assume that the vectors $w_1, \dots, w_m$ fulfill \ref{enum:correlation} of Definition \ref{def:assumptions_overcomplete}, then there exists an absolute constant $C>0$ only depending on $c_2$ in \ref{enum:correlation} such that 
\begin{align}
	\| G_n \|  \leq C \left( 1 + m\left(\frac{\log m}{D}\right)^{n / 2}\right).
\end{align}
	\end{lemma} 
\begin{proof}
	The result follows directly by Gershgorin circle theorem since the diagonal elements must be $1$ and the off-diagonal elements are bounded in absolute value by $c_2 \left(\frac{\log m}{D}\right)^{n / 2}$.
\end{proof}
\begin{lemma}
	\label{lem:chernoff_bound}
	Let $Z \in \R^{d}$ be a random vector and assume $\norm{Z}_2^2 \leq R$ almost surely. For $N$ independent copies $Z_1,\ldots,Z_N$ of $Z$,
	define the random matrix
	\begin{align*}
		G:= \sum_{i=1}^{N} Z_{i}Z_{i}^\top.
	\end{align*}
	Then, we have
	\begin{align*}
		\mathbb{P}\left(\lambda_{m}(G) \geq \frac{\lambda_m(\mathbb{E} G)}{4}\right) \geq 1 - m 0.7^{\frac{\lambda_m(\mathbb{E} G)}{R}}.
	\end{align*}
\end{lemma}
\begin{proof}
	The result follows directly from the standard matrix Chernoff bound.
\end{proof}


\end{document}